\documentclass{article}

\usepackage{microtype}
\usepackage{graphicx}
\usepackage{subfigure}
\usepackage{booktabs} 

\usepackage{hyperref}

\usepackage[accepted]{icml2021}

\usepackage{times} 
\usepackage{amsfonts}       
\usepackage{nicefrac}       
\usepackage{microtype}      
\usepackage{mathtools}
\usepackage{cases}
\usepackage{array}
\usepackage{color}
\usepackage{float}
\usepackage{amssymb}
\usepackage{wrapfig}
\usepackage{amsmath}
\usepackage{soul}
\usepackage{amsmath}
\usepackage{courier}
\usepackage{float}
\usepackage{tikz}
\usepackage{url}
\usepackage{bbm}
\usepackage{amssymb}
\usepackage{amsthm}
\newtheorem{theorem}{Theorem}[section]

\newtheorem{definition}{Definition}[section]

\usepackage{array}
\newcommand{\PreserveBackslash}[1]{\let\temp=\\#1\let\\=\temp}
\newcolumntype{C}[1]{>{\PreserveBackslash\centering}p{#1}}  
\newcolumntype{R}[1]{>{\PreserveBackslash\raggedleft}p{#1}}  
\newcolumntype{L}[1]{>{\PreserveBackslash\raggedright}p{#1}} 

\graphicspath{{./figs/}}
\newcommand{\wei}[1]{{\color{cyan}  \textbf{}}}
\newcommand{\hao}[1]{{\color{blue}  \textbf{}}}
\newcommand{\ding}[1]{{\color{red}  \textbf{}}}
\newcommand{\rebut}[1]{{\color{red}  \textbf{1}}}

\icmltitlerunning{Probabilistic Mixture-of-Experts for Efficient Deep Reinforcement Learning}

\begin{document}

\twocolumn[
\icmltitle{Probabilistic Mixture-of-Experts for Efficient Deep Reinforcement Learning}



\icmlsetsymbol{equal}{*}

\begin{icmlauthorlist}
\icmlauthor{Jie Ren}{equal,pku,xd}
\icmlauthor{Yewen Li}{equal,pku,xd}
\icmlauthor{Zihan Ding}{princeton}
\icmlauthor{Wei Pan}{delft}
\icmlauthor{Hao Dong}{pku}
\end{icmlauthorlist}


\icmlaffiliation{xd}{School of Electronic Engineering, Xidian University, Xi\'an, China}
\icmlaffiliation{princeton}{Princeton University, Princeton, United States}
\icmlaffiliation{delft}{Department of Cognitive Robotics, Delft University of Technology, Delft, Netherlands}
\icmlaffiliation{pku}{CFCS, Computer Science Department, Peking University, Beijing, China}

\icmlcorrespondingauthor{Hao Dong}{hao.dong@pku.edu.cn}

\icmlkeywords{Machine Learning, ICML}

\vskip 0.3in
]



\printAffiliationsAndNotice{\icmlEqualContribution} 

\begin{abstract}
	Deep reinforcement learning (DRL) has successfully solved various problems recently, 
	typically with a unimodal policy representation. However, grasping distinguishable 
	skills for some tasks with non-unique optima can be essential for further improving 
	its learning efficiency and performance, which may lead to a multimodal policy represented 
	as a mixture-of-experts (MOE). To our best knowledge, present DRL algorithms for general 
	utility do not deploy this method as policy function approximators due to the potential 
	challenge in its differentiability for policy learning. In this work, we propose 
	a probabilistic mixture-of-experts (PMOE) implemented with a Gaussian mixture model (GMM) 
	for multimodal policy, together with a novel gradient estimator for the indifferentiability problem, 
	which can be applied in generic off-policy and on-policy DRL algorithms using stochastic 
	policies, \emph{e.g.}, Soft Actor-Critic (SAC) and Proximal Policy Optimisation (PPO). Experimental 
	results testify the advantage of our method over unimodal polices and two 
	different MOE methods, as well as a method of option frameworks, based on the above two 
	types of DRL algorithms, on six MuJoCo tasks. Different gradient estimations for GMM like 
	the reparameterisation trick (Gumbel-Softmax) and the score-ratio trick are also compared 
	with our method. We further empirically demonstrate the distinguishable primitives learned 
	with PMOE and show the benefits of our method in terms of exploration.
\end{abstract}

\section{Introduction}
\label{submission}
Current deep reinforcement learning (DRL)~\cite{SuttonB98, SAC,PPO,a2c,DDPG, RL_book} methods are mostly built on parameterised models, specifically, neural networks. Function approximation lies at the core of DRL methods, including both value function approximation~\cite{DoubleDQN,DuelingDQN} and policy function approximation~\cite{SAC,PPO}. The methods with policy-centric function approximation in DRL are called policy-based algorithms, which typically apply the optimisation method: policy gradient. To improve the efficiency and performance of policy-based learning agents, various models are applied as representations of the policy. Scalability, generality and differentiability are desired characteristics of these models. 

More importantly, the models with broader capability of representation are essential for complex tasks. Among those policy representations, stochastic policies are those formulated as a` probabilistic distribution, that can generally be characterised as unimodal ones and multimodal ones.
In prior works, the choice of stochastic policies in DRL  is typically with the unimodal Gaussian distribution~\cite{gaussian_policy,SAC,PPO}.
However, most complex tasks usually allow multiple optimal solutions for even the same state, and a standard unimodal policy does not have the capability of capturing or leveraging that. 
Thus, some methods~\cite{MCP,MOERL_ACM09,MOERL_ACM16,MOERL_interpretable,MOERL_ICLR16,AdaptiveMOE} are proposed recently to model the policy as a multimodal distribution with the mixture-of-experts (MOE). 
The MOE, as originally proposed by \citep{gating91}, can be implemented with a Gaussian mixture model (GMM) in practice, which is also applicable for representing the  DRL policy.
Each unimodal Gaussian component of GMM can be seen as an expert or the so-called primitive~\cite{primitive}, and the mixing coefficients of GMM
are the probability for activating each primitive given an observation~\cite{MCP}. 
The distinguishable primitives learned by MOE can propose several solutions
for a task, which may potentially lead to better task performance and sample efficiency compared to its unimodal counterpart.

However, a straightforward usage of GMM in generic off-policy and on-policy DRL algorithms will face the indifferentiability problem in its end-to-end training scheme,
since optimising the GMM contains a process of optimising the
categorical distribution parameters~\cite{PRML}.
Prior MOE methods for RL usually works under some special settings, such as a two-stage training manner~\cite{MCP}, or with special task design~\cite{MOERL_ICLR16}. The indifferentiability 
problem can to solved to some extent by employing their policy representations, but our experiments show that it generally does not lead to better learning performances.
One possible explanation is the biased gradient estimation for the categorical distribution parameters of GMM.
Previous methods for solving the indifferentiability involves the reparameterisation
trick, \emph{e.g.}, Gumbel-Softmax~\cite{gumbel,gumbel2} and the score-ratio trick, \emph{e.g.}, REINFORCE~\cite{REINFORCE,likelihood-ratio,score-ratio1,score-ratio2,score-ratio3,score-ratio4}. However, the Gumbel-Softmax method will bring brittleness to the temperature hyperparameter~\cite{ABVRRG, REBAR} in the softmax function,
while the score-ratio tricks are hard to converge~\cite{variance_reinforce} due to the large variances. These problems will be fatal for the fragile DRL training~\cite{SAC} in general.

To tackle the above challenges, we propose a probabilistic MOE as a multimodal function approximator, which can be used as DRL policies for generic off-policy and on-policy DRL algorithms using stochastic policies. Moreover, a novel gradient estimator named Frequency Approximate Gradient is proposed for solving the indifferentiability problem when optimizing the MOE. 
Our method has shown advantageous performances on 6 MuJoCo continuous control tasks with an improvement up to 28.4\% in the SAC-based experiments and 39.3\% in the PPO-based experiments compared to vanilla SAC and PPO respectively, measured with the area under the curve (AUC).
Experiments also show the advantageous performance of our method over two other gradient estimation methods, Gumbel-Softmax and REINFORCE. We further analyse the learned MOE policies by displaying its distinguishable primitives and measuring its sensitivity to the choice of primitive numbers, as well as giving a potential reason for the performance improvement in terms of the exploration ranges. 

\section{Related Work}
\paragraph{Mixture-of-Experts}
To speed up the learning and improve the generalisation ability on different scenarios, \citet{AdaptiveMOE} proposed to use several different expert networks instead of a single one.
To partition the data space and assign different kernels for different spaces, \cite{ME_SVM_0,ME_SVM_1} combine MOE with SVM. To break the dependency among training outputs and speed 
up the convergence, the Gaussian process (GP) is generalised similarly to MOE \citep{ME_GP_0,ME_GP_1,ME_GP_2}. MOE can be also combined with RL \citep{ME_RL_1,MOERL_ACM09,MOERL_ACM16,MOERL_ICLR16,MCP}, 
in which the policies are modelled as probabilistic mixture models and each expert aim to learn distinguishable policies.
\cite{MOERL_ACM16} introduces a mixture of actor-critic experts approaches to learn terrain-adaptive dynamic locomotion skills. \cite{MCP} changes the mixture-of-experts distribution addition expression into the multiplication expression.

\paragraph{Hierarchical Policies}
The MOE policy structure can be seen as a hierarchical structure, as the agent chooses a
Gaussian component of GMM to act according to the mixing coefficient. 
There are two main related hierarchical policy structures. 
One is the feudal schema \citep{FEUDAL92}, which has “manager" agents to first make high-level decisions and the “worker" agents to make low-level actions.
\cite{FEUDAL_ICML17} generalises the feudal schema into continuous action space and uses an embedding operation to solve the indifferentiable problem.
The other is the options framework \citep{Option99,Option_old_0, Option_old_1, Option_old_2, Option_old_3}, which has an upper-level agent (policy-over-options) to decide whether the lower-level agent (sub-policy) should start or terminate. 
\cite{Option_NIPS16_2} uses internal and extrinsic rewards to learn sub-policies and policy-over-options. \cite{OptionAC} trains sub-policies and policy-over-options with a deep termination function.

\paragraph{Gradient Estimation Methods}
Stochastic neural networks rarely use discrete latent variables
due to the inability to backpropagate through samples~\cite{gumbel}.
Existing stochastic gradient estimation methods traditionally focus on the Path derivative gradient estimators and the Score-ratio gradient estimators.
Path derivative gradient estimators are formulated specifically for reparameterisable distributions, \emph{e.g.}, 
~\cite{VAE,vae2} employs a reparameterisation trick for the latent Gaussian distribution,
~\cite{st-estimator} designs a Straight-Through estimator for Bernoulli distribution, and ~\cite{gumbel} introduces
Gumbel-Softmax to approximate categorical samples whose parameter gradients can be easily computed via the
reparameterisation trick.
Score-ratio gradient estimators use the Log-derivative trick to derive an estimator, such as the score function
estimator(also referred to as REINFORCE~\cite{REINFORCE}), likelihood ratio estimator~\cite{likelihood-ratio},
and other estimators augmented with Monte Carlo variance reduction techniques~\cite{score-ratio1,score-ratio2,score-ratio3,score-ratio4}. 

\paragraph{Policy-based RL}
Policy-based RL aims to find the optimal policy to maximise the expected return through gradient updates. 
Among various algorithms, Actor-critic is often employed~\citep{Barto,SuttonB98}. 
Off-policy algorithms~\citep{PGQ, DDPG,Q-Prop, SAC} are more sample efficient than on-policy ones~\citep{PetersS08, PPO, a2c, Reactor}. 
However, the learned policies are still unimodal.

\section{Method}
	
	\subsection{Notation}
	The model-free RL problem can be formulated by Markov Decision Process (MDP), denoted as a tuple $(\mathcal{S}, \mathcal{A}, P, r)$, where $\mathcal{S}$ and $\mathcal{A}$ are continuous state and action space, respectively. 
	The agent observes state $s_t \in \mathcal{S}$ and takes an action $a_t \in \mathcal{A}$ at time step $t$. 
	The environment emits a reward $r: \mathcal{S} \times \mathcal{A} \rightarrow [r_{min}, r_{max}]$ and transitions to a new state $s_{t+1}$ according to the transition probabilities $P: \mathcal{S} \times \mathcal{S} \times \mathcal{A} \rightarrow [0, \infty)$.
	In deep reinforcement learning algorithms, we always use the Q-value function $Q(s_{t}, a_{t})$ to describe the expected return after taking an action $a_{t}$ in the state $s_{t}$. The Q-value can be iteratively computed by applying the Bellman backup given by:
		\begin{equation}
		Q(s_{t}, a_{t}) \triangleq \mathbb{E}_{s_{t+1}\sim P}\big[r(s_{t}, a_{t}) + \gamma \mathbb{E}_{a_{t+1}\sim \pi}[Q(s_{t+1}, a_{t+1})]\big].
		\end{equation}
	Our goal
	is to maximise the expected return:
	\begin{equation}
	\label{eq:general_goal}
	\pi_{\Theta*}(a_{t}|s_{t}) = \arg\max_{\pi_{\Theta}(a_{t}|s_{t})} \mathbb{E}_{a_t \sim\pi_{\Theta} (a_{t}|s_{t})}[Q(s_{t}, a_{t})],
	\end{equation}
	where $\Theta$ denotes the parameters of the policy network $\pi$. 
	With Q-value network (critic) $Q_{\phi}$ parameterised by $\phi$, Stochastic gradient descent (SGD) based approaches are usually used to update the policy network:
	\begin{equation}
	\label{eq:general_update}
	\Theta \leftarrow  \Theta + \nabla_{\Theta}\mathbb{E}_{a\sim\pi_{\Theta}(a_t|s_t)}[Q_{\phi}(s_t, a_t)].
	\end{equation}
	
	\subsection{Probabilistic Mixture-of-Experts (PMOE)}
	The proposed PMOE method follows the typical setting of MOE method, which decomposes a complex policy $\pi$ into a mixture of low-level stochastic policies with each of them as a probability distribution, represented as the following: 
	\begin{align}
	\label{eq:mixture_define}
    	\pi_{\{\theta, \psi\}}(a_{t}|s_{t}) &= \sum_{i=1}^{K} w_{\theta_{i}}(s_{t}) \pi_{\psi_{i}}(a_{t}|s_{t}), \\ &s.t.\, \sum_{i=1}^{K}w_{\theta_{i}} = 1, \, w_{\theta_{i}}>0,
	\end{align}
	where each $\pi_{\psi_{i}}$ denotes the action distribution within each low-level policy, \emph{i.e.} a \textit{primitive}, and $K$ denotes the number of primitives. $w_{\theta_{i}}$ is the weight that specifies the probability of the activating primitive $\pi_{\psi_i}$, which is called the \textit{routing} function. According to the GMM assumption~\cite{PRML}, $w_{\theta}$ is a Categorical distribution and $\pi_{\psi}$ is a unimodal Gaussian distribution. For $\forall  i \in \{1,2,...,K\}$, $\theta_i$ and $\psi_i$ are parameters of $w_{\theta_i}$ and $\pi_{\psi_i}$, respectively. 
	After the policy decomposition with PMOE method, we can rewrite the update rule in Eq.~\ref{eq:general_update} as:
	\begin{equation}
		\begin{aligned}
			\theta \leftarrow \theta + \nabla_{\theta}\mathbb{E}_{a_t\sim\pi_{\{\theta, \psi\}}(a_{t}|s_{t})}[Q_{\phi}(s_t, a_t)],\\
			\psi \leftarrow \psi + \nabla_{\psi}\mathbb{E}_{a_t\sim\pi_{\{\theta, \psi\}}(a_{t}|s_{t})}[Q_{\phi}(s_t, a_t)].
		\end{aligned}
	\end{equation}
	In practice, if we apply a Gaussian distribution for each of the low-level policies here in PMOE, the overall PMOE will end up to be a GMM. However, sampling from the mixture distributions of primitives $\pi_{\{\theta, \psi\}}(a_t|s_t)$ embeds a sampling process from the categorical distribution $w_{\theta}$, which makes the differential calculation of policy gradients commonly applied in DRL hard to achieve. We provide a theoretically guaranteed solution for approximating the gradients in the sampling process of PMOE and apply it for optimising the PMOE policy model within DRL, which will be described in details.

	\subsection{Learning the Routing}
	\label{sec:freq_loss}

	The routing function in MOE typically involves a sampling process from a categorical distribution (due to the discontinuity among multiple experts), which is indifferentiable~\cite{gumbel}. To handle this difficulty, we propose a new gradient estimator for this routing function. 
	
	Specifically, given a state $s_t$, we sample one action $a_t^i$ from each primitive $\pi_{\psi_{i}}$, to get a total of $K$ actions  $\{a_t^i; i=1, 2,\cdots, K\}$, and compute $K$ Q-value estimations $\{Q_{\phi}(s_t, a_t^i); i=1, 2,\cdots, K\}$ for each of the actions. We say the primitive $j$ is the “optimal" primitive under the Q-value estimation if $j = \arg\max_i Q_{\phi}(s_t, a_t^i)$.
	There exists a frequency of the primitive $j$ to be the “optimal" primitive given a set of state $\{s_t\}$, here we propose a new gradient estimator which optimise $\theta_k$ towards the frequency.
	\begin{definition}[Frequency Approximate Gradient] For a stochastic mixture-of-experts, the gradient value of a single-instance sampling process for the routing function $w_\theta$ can be estimated with the frequency approximate gradient, which is defined as:
		\begin{equation}
		\text{grad} = \delta_k\nabla_{\theta_k}w_{\theta_k}, \,\delta_k = -\mathbbm{1}_k^{\text{best}} + w_{\theta_k},
	\end{equation}
		where $\nabla_{\theta_k}w_{\theta_k}$ is the gradient of $w_{\theta_k}$ for parameters $\theta_k$ and $\mathbbm{1}_k^{\text{best}}$ is the indicator function that $\mathbbm{1}_k^{\text{best}} = 1$ if $k = \arg\max_jQ_{\phi}(s_{t}, a_{t}^j)$ and $\mathbbm{1}_k^{\text{best}} = 0$ otherwise.
	\end{definition}
	
	\begin{theorem}
    
    The accumulated frequency approximate gradient is an asymptotically unbiased estimation of the true gradient for the sampling process from a categorical distribution in the routing function, with a batch of $N \to \infty$ samples.


	\end{theorem}
	
	\begin{proof}
	\begin{equation}
		\begin{aligned}
			\delta_k
			&= -\mathbbm{1}_k^{\text{best}} + w_{\theta_k}\\
			&= \mathbbm{1}_k^{\text{best}}w_{\theta_k} - \mathbbm{1}_k^{\text{best}}w_{\theta_k} -\mathbbm{1}_k^{\text{best}} + w_{\theta_k}\\
			&= \frac{1}{2}\mathbbm{1}_k^{\text{best}}\nabla_{w_{\theta_k}}(1 - w_{\theta_k})^2 + \frac{1}{2}(1 - \mathbbm{1}_k^{\text{best}})\nabla_{w_{\theta_k}}w_{\theta_k}^2.
		\end{aligned}
		\label{eq:lead_to_freq_loss}
	\end{equation}
	Now we assume a batch of samples with a number of $N$ is applied, and the true probability of primitive $k$ to be the best primitive (\emph{i.e.} $k = \arg\max_jQ_{\phi}(s_{t}, a_{t}^k)$, is $f_k$.
	The batch accumulated gradient will be
	\begin{equation}
		\begin{aligned}
			\overline{\text{grad}}
			=& \frac{1}{N}\sum_{j=1}^N \text{grad} \\
			=& \frac{N_t}{2N}\nabla_{w_{\theta_k}}(1 - w_{\theta_k})^2\nabla_{\theta_k}w_{\theta_k} \\
			&+ \frac{(N-N_t)}{2N}\nabla_{w_{\theta_k}}w_{\theta_k}^2\nabla_{\theta_k}w_{\theta_k}\\
			\stackrel{N \to \infty}{=}&  (f_k - w_{\theta_k})\nabla_{\theta_k}w_{\theta_k},
		\end{aligned}
		\label{equ:acc_grad}
	\end{equation}
	where the last formula indicates that $N_t=N f_k$ when $N \to \infty$, since the true probability can be approximated by $\frac{N_t}{N}$ in the limit case.  Since $\nabla_{\theta_k}w_{\theta_k}$ is not always equal to $0$, the gradient equals to $0$ if and only if $w_{\theta_k} = f_k$ .  Optimising with the above equation is the same as minimising the distance between $w_{\theta_k}$ and $f_k$, with the optimal situation as $w_{\theta_k}=f_k$ when letting the last formula of Eq.~\ref{equ:acc_grad} be zero.
    \end{proof}
    
    Then $\theta_k$ is updated via a gradient descent based algorithm, \emph{e.g.}, Stochastic Gradient Descent (SGD):
	\begin{equation}
		\theta_k \leftarrow \theta_k - \delta_k\nabla_{\theta_k}w_{\theta_k}.
	\end{equation}
	
	According to the formation of Eq~\ref{eq:lead_to_freq_loss}, we can build an elegant loss function that achieves the same goal of gradient estimating:
	\begin{equation}
	\label{eq:freq_loss}
	\mathcal{L}_{freq} = (v-w)(v-w)^T, w=[w_{\theta_{1}}, w_{\theta_{2}}, \cdots, w_{\theta_{K}}],
	\end{equation}
	$v$ is a one-hot code vector $v = [v_1, v_2, \cdots, v_K]$ with:
	\begin{equation}
		\label{eq:compute_v}
		v_j = \left\{
		\begin{aligned}
		&1, \,\text{if}\, j=\arg\max_{i}Q_{\phi}(s_t, a_t^i);\\
		&0, \,\text{otherwise}.
		\end{aligned}
		\right.
	\end{equation}
	\subsection{Learning the Primitive}
	\label{sec:primitive_loss}
	To update the $\psi_i$ within each primitive, we provide two approaches of optimising the primitives: \textit{back-propagation-all} and \textit{back-propagation-max} manners. 

	For the \textit{back-propagation-all} approach, we update all the primitive:
	\begin{equation}
	\mathcal{L}_{pri}^{bpa} = -\sum_i^KQ_{\phi}(s_{t}, a_{t}^i), \, a_{t}^i \sim \pi_{\psi_{i}}(a_t|s_t).
	\end{equation}
	
	For the \textit{back-propagation-max} approach, 
	we use the highest Q-value estimation as the primitive loss:
	\begin{equation}
	\label{eq:pri_loss}
	\mathcal{L}_{pri}^{bpm} = -\max_{i=1, 2, \cdots, K}\{Q_{\phi}(s_{t}, a_{t}^i)\}, \, a_{t}^i \sim \pi_{\psi_{i}}(a_t|s_t).
	\end{equation}
	\label{sec:backprog_all}
	
	With either approach, we have the same stochastic policy gradients as following:
	\begin{equation}
	\begin{aligned}
		\nabla_{\psi_{i}}\mathcal{L}_{pri} 
		=& -\nabla_{\psi_j} \mathbb{E}_{\pi_{\psi_i}}[Q_{\phi}(s_{t}, a_{t})]\\
		=& \mathbb{E}_{\pi_{\psi_i}}[-Q_{\phi}(s_{t}, a_{t})\nabla_{\psi_j}\log\pi_{\psi_j}(a_t|s_t)]
	\end{aligned}
	\end{equation}

	Ideally, both approaches are feasible for learning a PMOE model. However, in practice, we find that the \textit{back-propagation-all} approach will tend to learn primitives that are close to each other, while the \textit{back-propagation-max} approach is capable of keeping primitives distinguishable. The phenomenon is demonstrated in our experimental analysis. Therefore, we adopt the \textit{back-propagation-max} approach as the default setting of PMOE model without additional clarification.
	\subsection{Learning the Critic}
	\label{sec:critic_loss}
	Similar to the standard off-policy RL algorithms, our Q-value network is also trained to minimise the Bellman residual:
	\begin{equation}
    	\begin{aligned}
    	\mathcal{L}_{critic} = &\mathbb{E}_{(s_t, a_t)\sim\mathcal{D}} [\|Q_{\phi}(s_t, a_t) - [r_t + \\ &\gamma \max_{a_{t+1}}Q_{\bar{\phi}}(s_{t+1}, a_{t+1})]\|_2], a_{t+1} \sim \pi(a_{t+1}|s_{t+1})
    	\end{aligned}
		\label{eq:critic_loss}
	\end{equation}
	where $\bar{\phi}$ is the parameters of the target network.

	The learning component can be easily embedded into the popular actor-critic algorithms, such as soft actor-critic (SAC)~\citep{SAC}, one of the state-of-the-art off-policy RL algorithms. In SAC, $Q_{\psi}(s_t, a_t^j)$ is substituted with $Q_{\psi}(s_t, a_t^j) + \alpha\mathcal{H}_j$, where $\alpha$ is temperature and $\mathcal{H}_j = -\log\pi_{\psi_{j}}(a_t|s_t)$ is the entropy which are the same as in SAC. The algorithm is summarised in Algorithm~\ref{alg:training}. When $K=1$, our algorithm simply reduces to the standard SAC. 
    \begin{algorithm}[]
        \caption{Probabilistic MOE}
        \label{alg:training}
     \begin{algorithmic}
        \STATE {\bfseries Input:} $\theta, \psi, \phi$
        \STATE Initialise the target networks $\bar{\theta}\leftarrow\theta, \bar{\psi}\leftarrow\psi, \bar{\phi}\leftarrow\phi $
        \STATE {Initialise an empty replay buffer: $\mathcal{D}\leftarrow\Phi$}
        \REPEAT
        \FOR{each environment step}
        \STATE {Sample action from policy: $a_{t}\sim\pi_{\{\theta, \psi\}} (a_{t}|s_t)$}
        \STATE {Interact with the environment: $s_{t+1} = p(s_{t+1|a_t, s_t})$}
        \STATE {Store in buffer: $\mathcal{D} = \mathcal{D}\cup\{s_t, a_t, s_{t+1}, r(s_t, a_t)\}$}
        \ENDFOR

        \FOR{each update step}
            \STATE {Sample from buffer: $\{s_t, a_t, s_{t+1}, r_t\} \sim \mathcal{D}$}
                \STATE {Compute $\mathcal{L}_{freq}$ with Eq. \ref{eq:freq_loss}, $\mathcal{L}_{pri}$ with Eq.~\ref{eq:pri_loss}, and $\mathcal{L}_{critic}$ with Eq.~\ref{eq:critic_loss}}
                \STATE {Update the policy with: \\
                $\theta \leftarrow \theta - \lambda_\theta\nabla_{\theta}\mathcal{L}_{freq}$, $\psi \leftarrow \psi - \lambda_\psi\nabla_{\psi}\mathcal{L}_{pri}$}
				\STATE {Update the critic with: \\
				$\phi \leftarrow \phi - \lambda_\phi\nabla_{\phi}\mathcal{L}_{critic}$}
        \ENDFOR
        \UNTIL{converge}
    \STATE {\bfseries Output:} $\theta, \psi, \phi.$
     \end{algorithmic}
     \end{algorithm}

\section{Experiments}

To testify the performance of our method, we conduct a thorough comparison on a set of challenging continuous control tasks in 
OpenAI Gym MuJoCo environments~\citep{GYM} with other baselines, including a MOE method with gating operation~\citep{gating91}, 
Double Option Actor-Critic (DAC)~\citep{DoubleOptionAC} option framework, the Multiplicative Compositional Policies (MCP)~\citep{MCP},
and other two implementations of PMOE with Gumbel-Softmax~\cite{gumbel, gumbel2} and REINFORCE~\cite{REINFORCE}. Well-known sample efficient algorithms involving Soft Actor-Critic (SAC)~\citep{SAC} and Proximal Policy Optimisation (PPO)~\citep{PPO} are adopted as basic DRL algorithms in our experiments, where different multimodal policy approximation methods are built on top. This verifies the generality of PMOE for different DRL algorithms.

To experimentally show the properties learned with PMOE, a deeper investigation of PMOE is conducted to find out the additional effects caused by deploying mixture models rather than a single model in policy learning.  We start with a simple self-defined \emph{target-reaching}  task with sparse rewards to show our method can indeed learn various optimal solutions with distinguishable primitives, which are further demonstrated on complicated tasks in MuJoCo. Additionally, the exploration behaviours are also compared between PMOE and other baselines to explain the advantageous learning efficiency of PMOE.
Finally, to know how the number of primitives affects the performance, we test PMOE with different values of $K$ on the  \textit{HumanoidStandup-v2} environment
to analyse the impact of the different number of primitives.

\subsection{Performance Evaluation}
\begin{figure}[ht!]
	\centering
	\subfigure[Ant-v2]{
		\includegraphics[width=0.170\textheight]{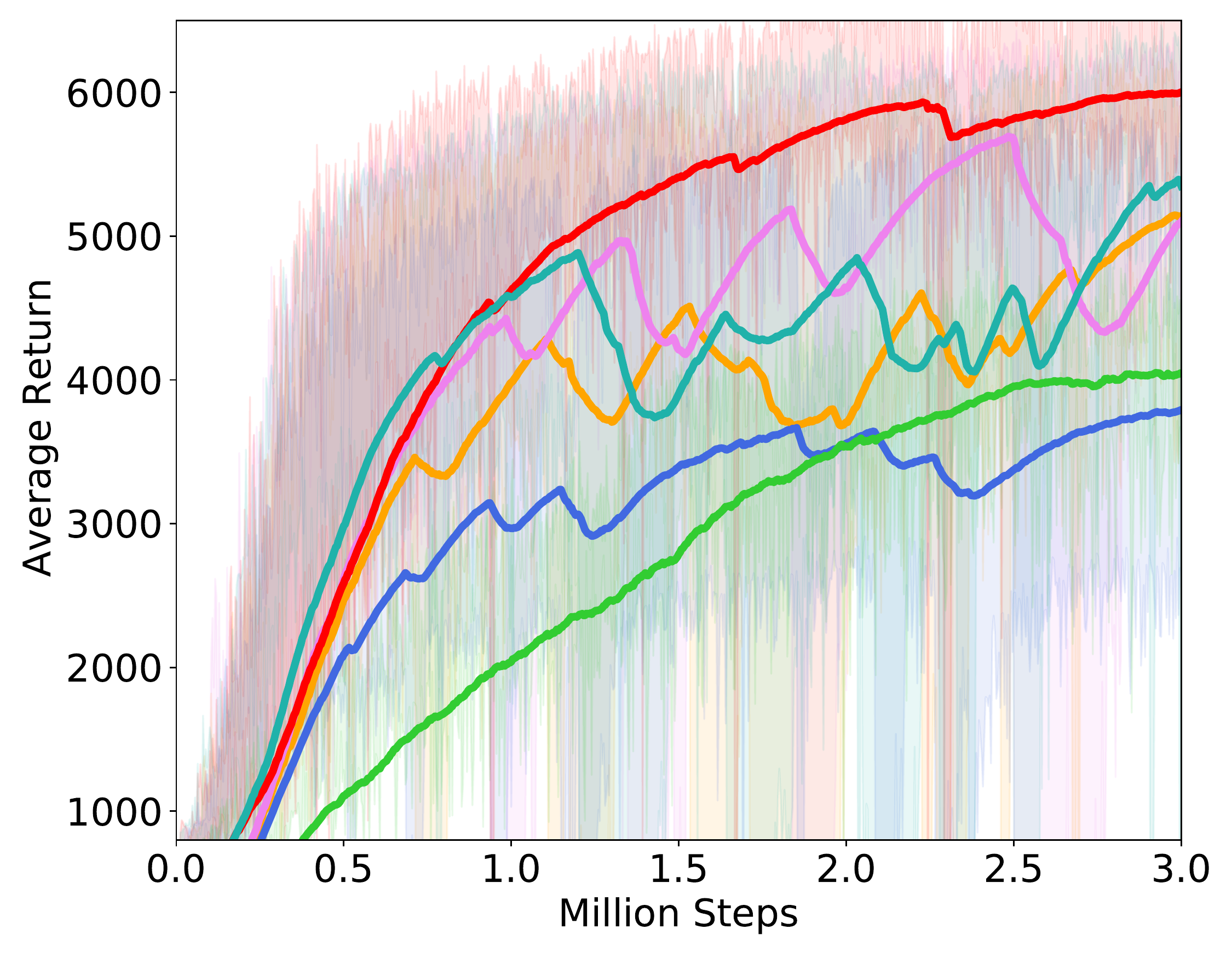}
	}
	\subfigure[Hopper-v2]{
		\includegraphics[width=0.170\textheight]{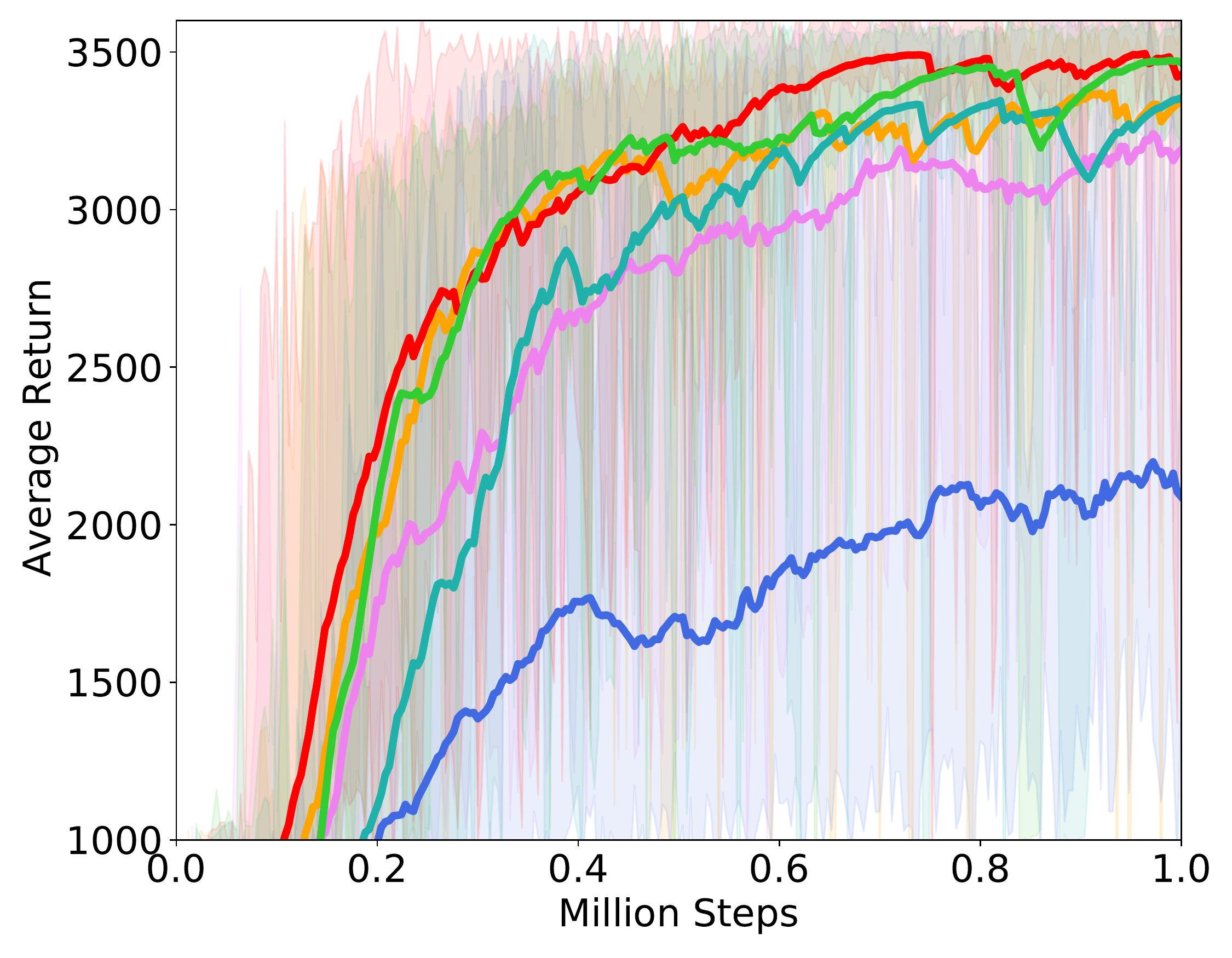}
	}
	\subfigure[Walker2D-v2]{
		\includegraphics[width=0.170\textheight]{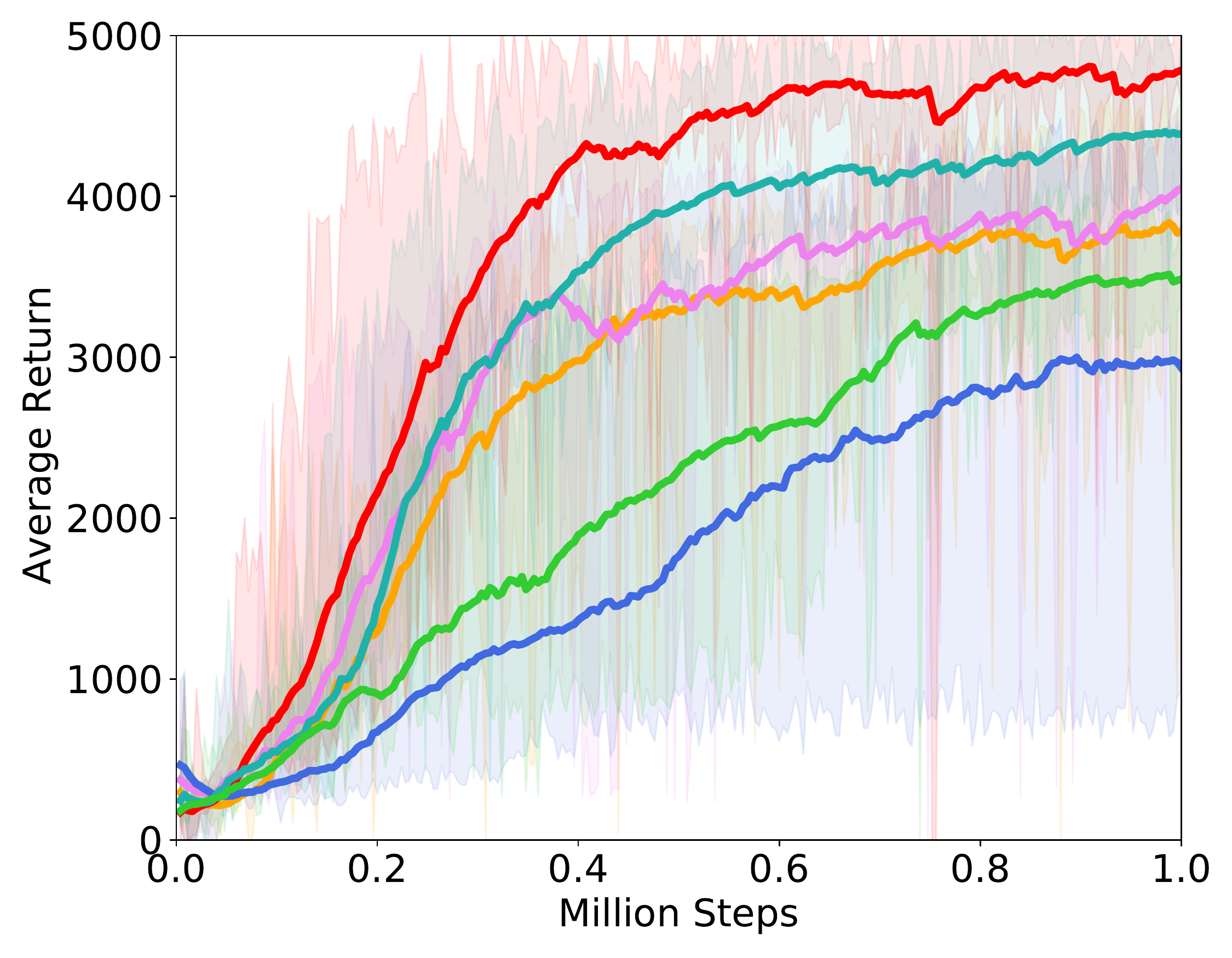}
	}
	\subfigure[Humanoid-v2]{
		\includegraphics[width=0.170\textheight]{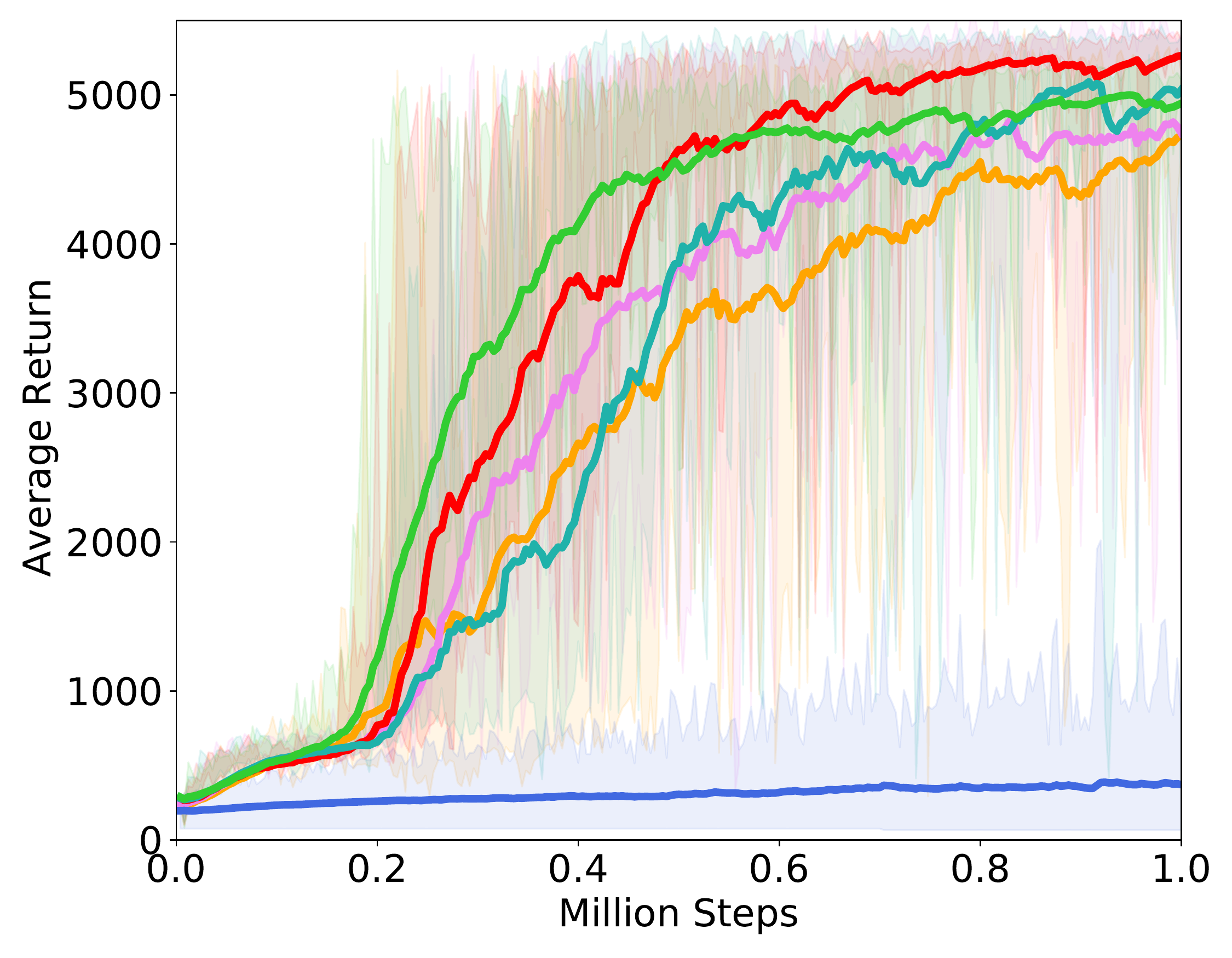}
	}
	\subfigure[HumanoidStandup-v2]{
		\includegraphics[width=0.170\textheight]{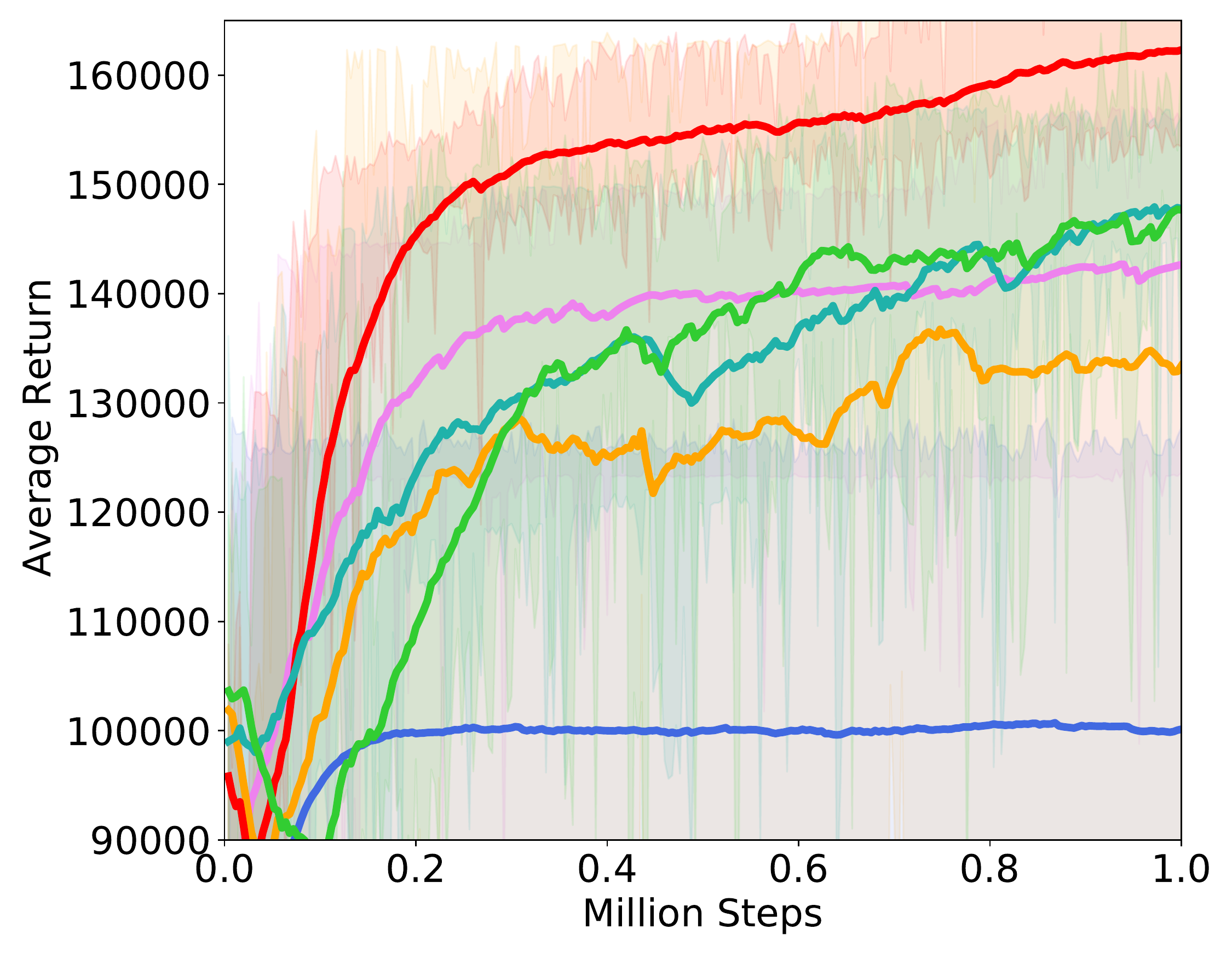}
	}
	\subfigure[HalfCheetah-v2]{
		\includegraphics[width=0.170\textheight]{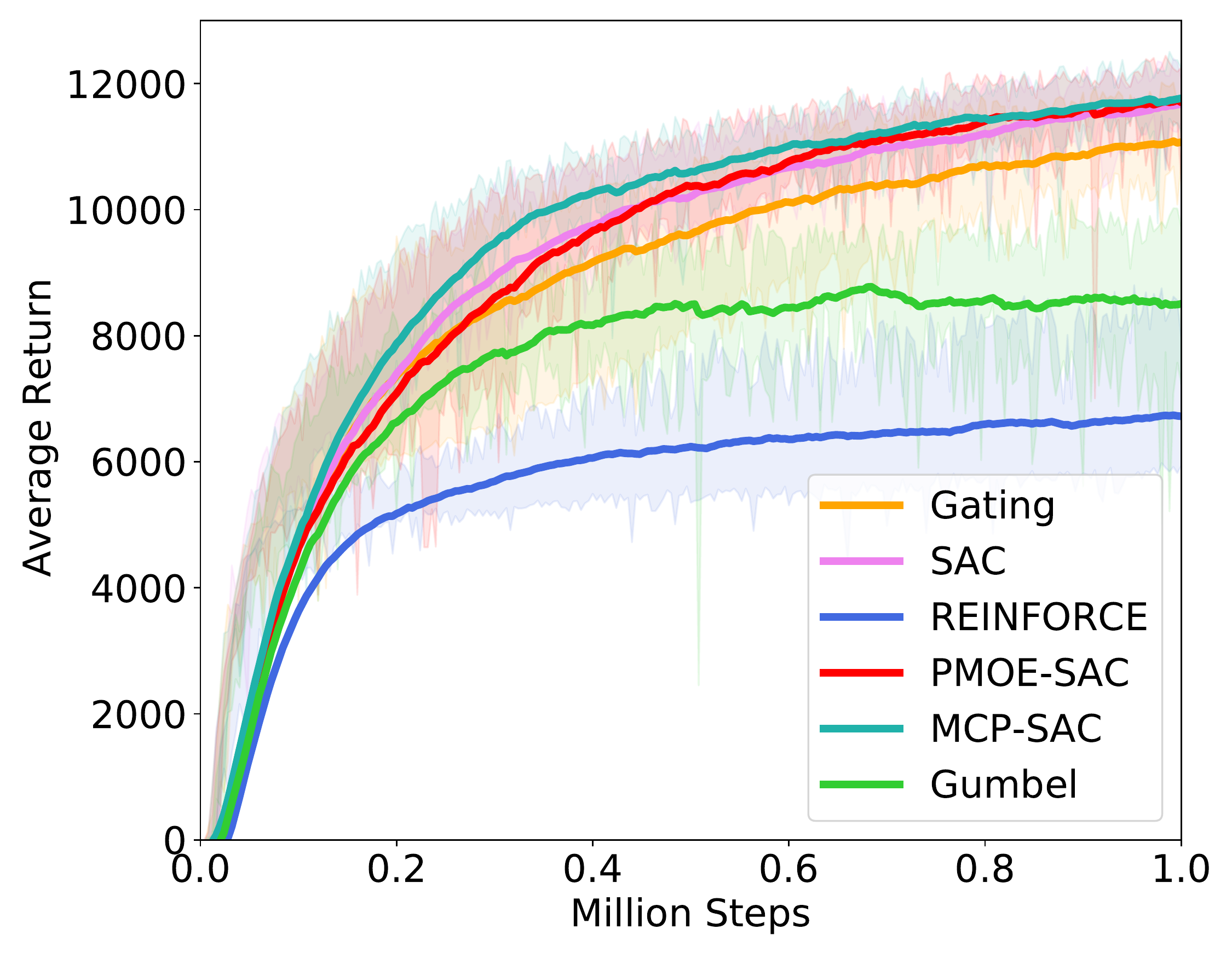}
	}
	\caption{Training curves on MuJoCo benchmarks with SAC-based algorithms. We set PMOE with $K=4$ in all the experiments except \textit{HalfCheetah-v2} with $K=2$ and \textit{HumanoidStandup-v2} and \textit{Humanoid-v2} with $K=10$.}
	\label{fig:sac_MuJoCo}
\end{figure}

The evaluation on the average returns with SAC-based and PPO-based algorithms are shown in Fig.~\ref{fig:sac_MuJoCo} and Fig.~\ref{fig:ppo_MuJoCo}, respectively. 
Specifically, in Fig.~\ref{fig:sac_MuJoCo}, SAC-based algorithms with either unimodal policy or our PMOE for policy approximation are compared against the MCP~\citep{MCP} and gating operation methods~\citep{gating91} in terms of the average returns across six typical MuJoCo tasks.\wei{I have no idea what you compared and can't understand the English} 
In Fig.~\ref{fig:ppo_MuJoCo}, a comparison with similar settings while basing on a different DRL algorithm is conducted to show the generality of PMOE method. 
Specifically, we compare the proposed PMOE for policy approximation on PPO with DAC and PPO-based MCP methods\wei{this is not right?}.
Besides, to demonstrate that our novel gradient estimator is proper for PMOE, we also compare our method with other two implementations based on the 
Gumbel-Softmax and REINFORCE in Fig.~\ref{fig:sac_MuJoCo} and in Fig.~\ref{fig:ppo_MuJoCo}.
PMOE is testified to provide improvement for general DRL algorithms with stochastic policies on a variety of tasks. 
Training details are provided in Appendix C.

\begin{figure}[ht!]
\centering
\subfigure[Ant-v2]{
	\includegraphics[width=0.17\textheight]{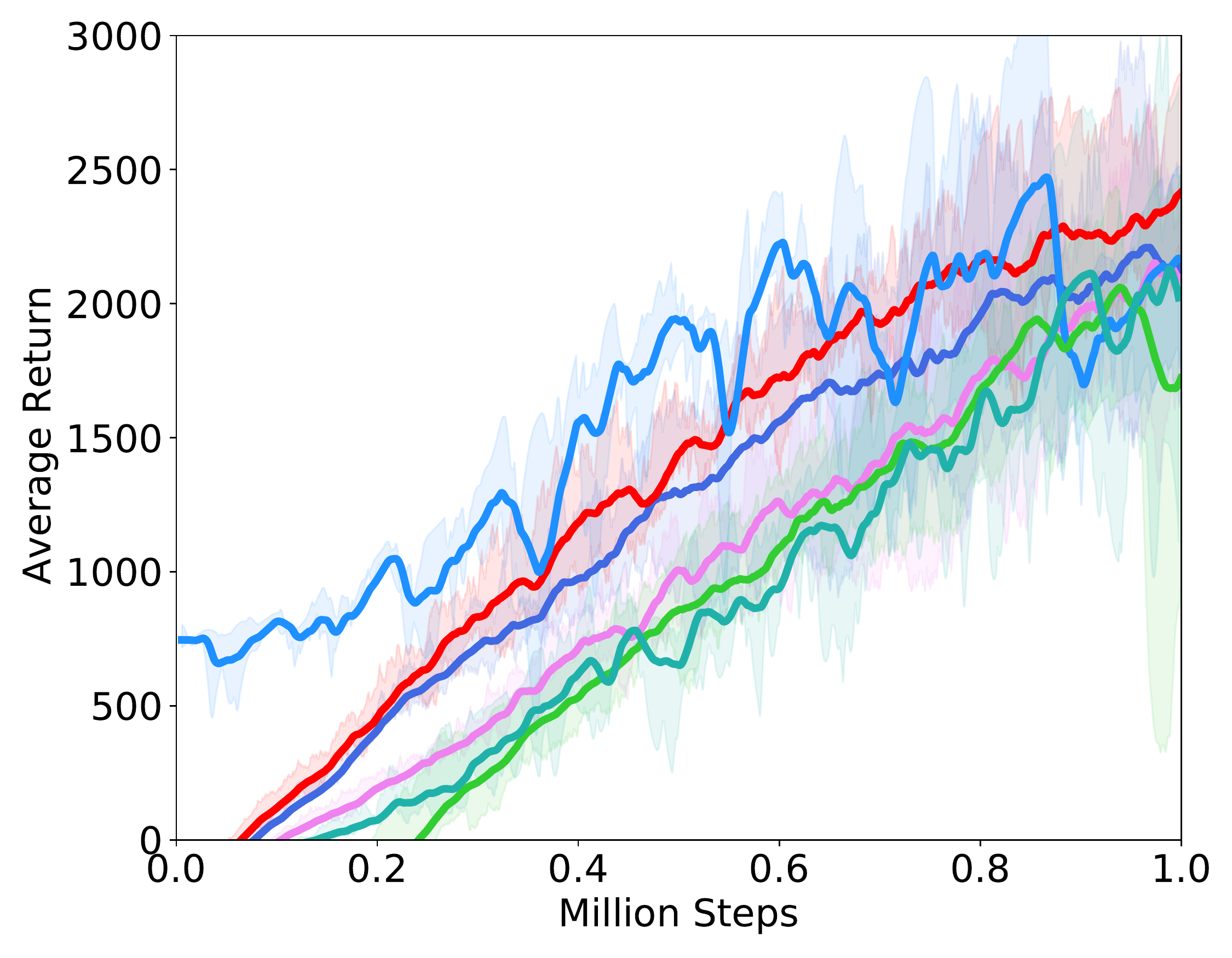}
}
\subfigure[Hopper-v2]{
	\includegraphics[width=0.17\textheight]{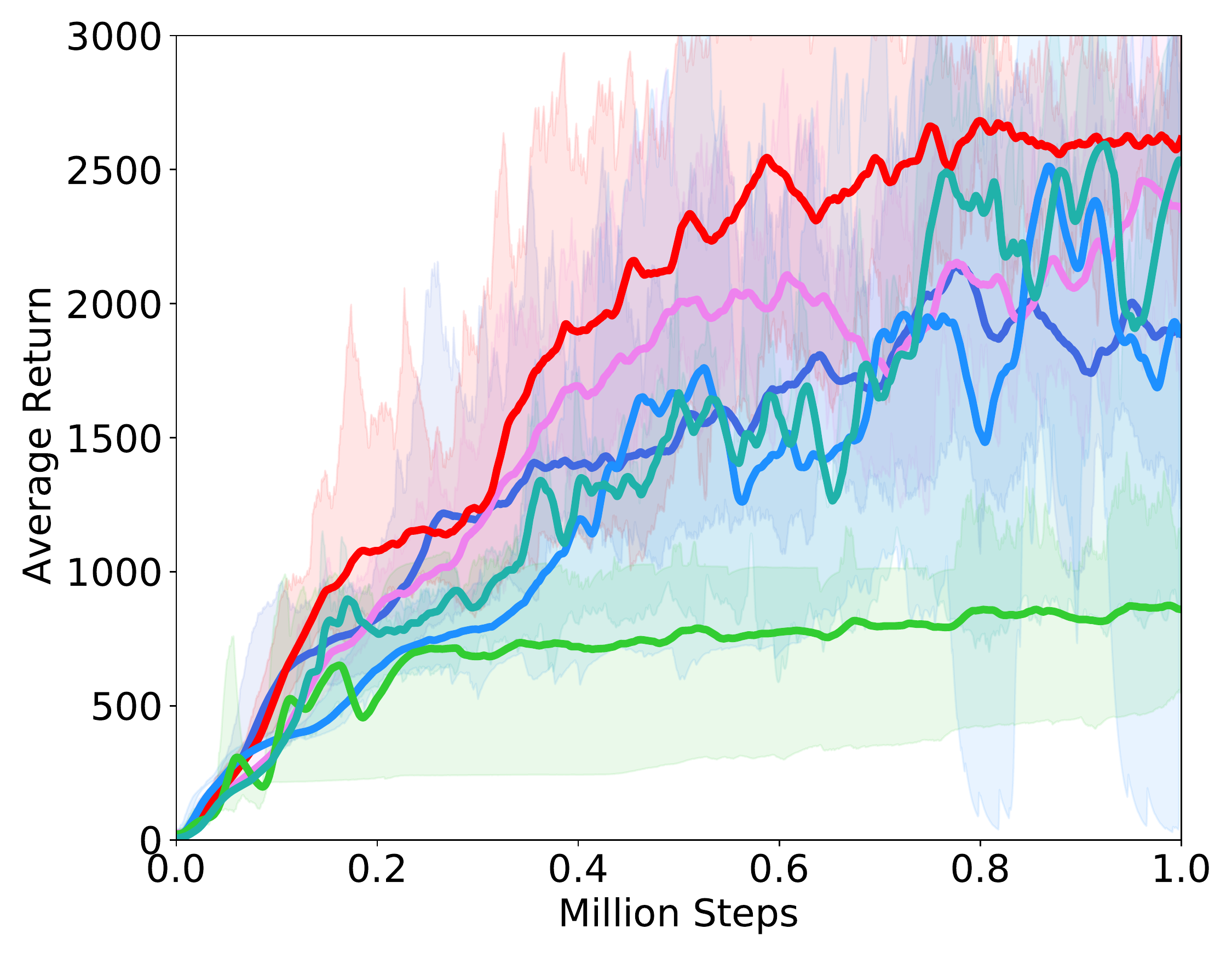}
}
\subfigure[Walker2D-v2]{
	\includegraphics[width=0.17\textheight]{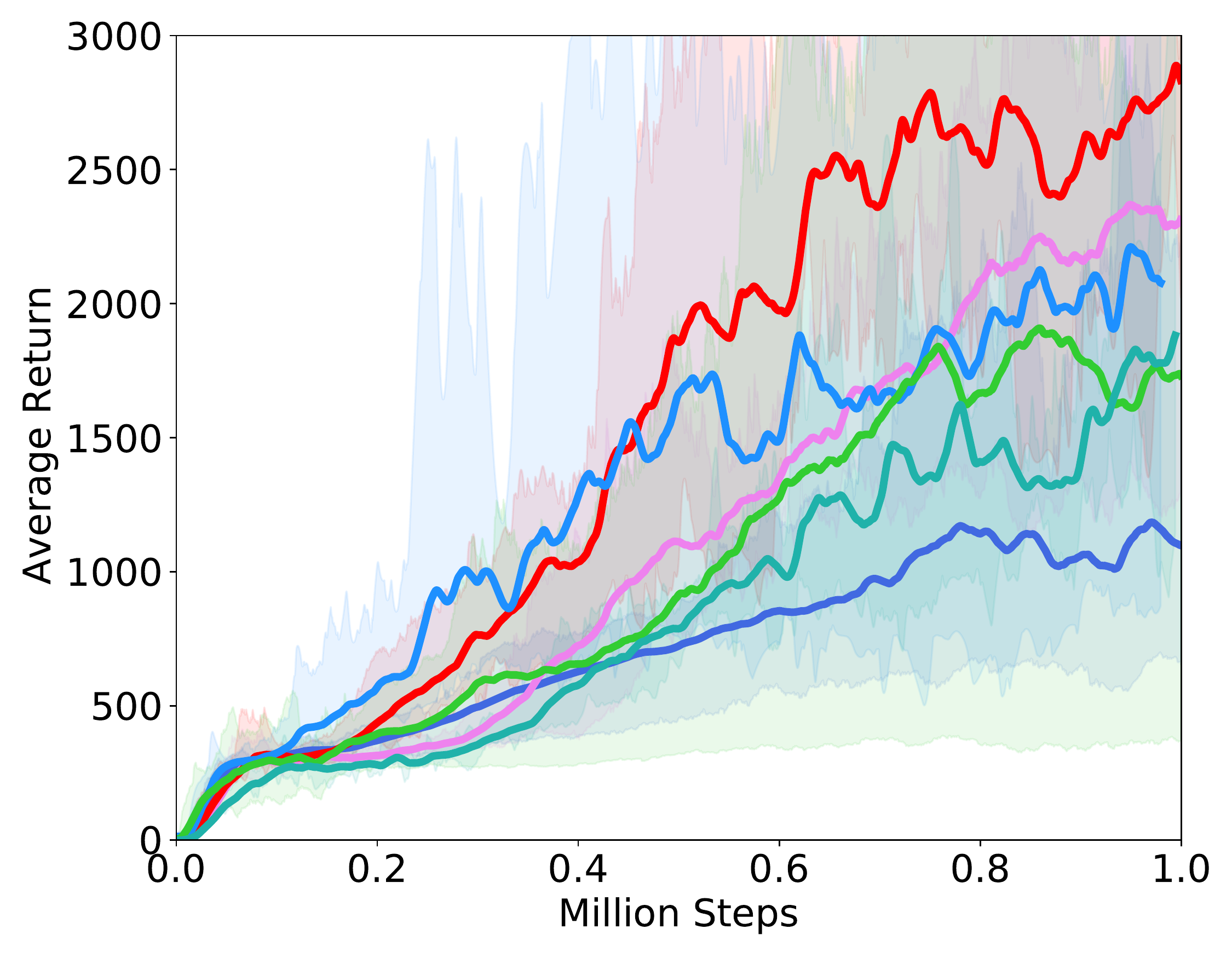}
}
\subfigure[Humanoid-v2]{
	\includegraphics[width=0.17\textheight]{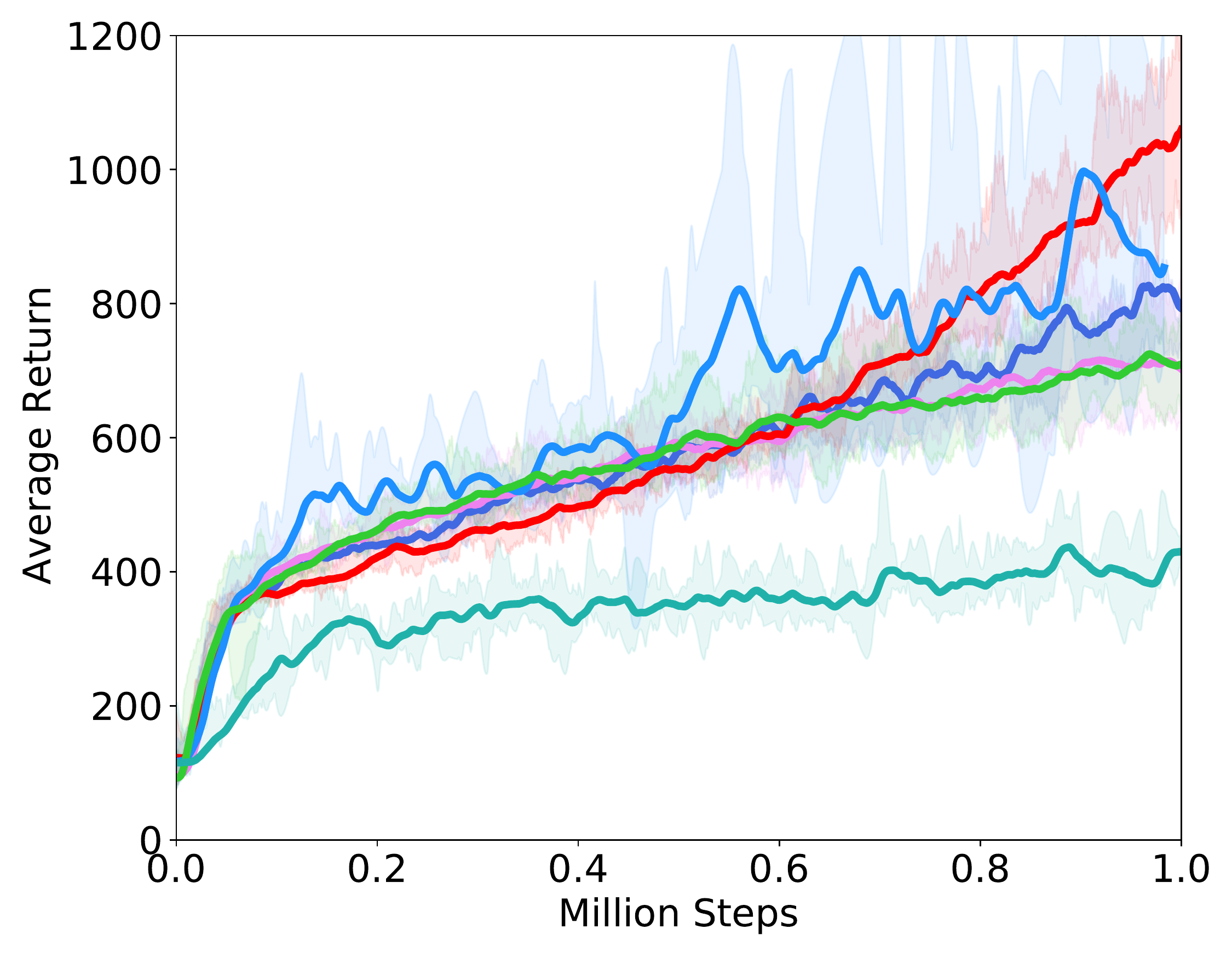}
}
\subfigure[HumanoidStandup-v2]{
	\includegraphics[width=0.17\textheight]{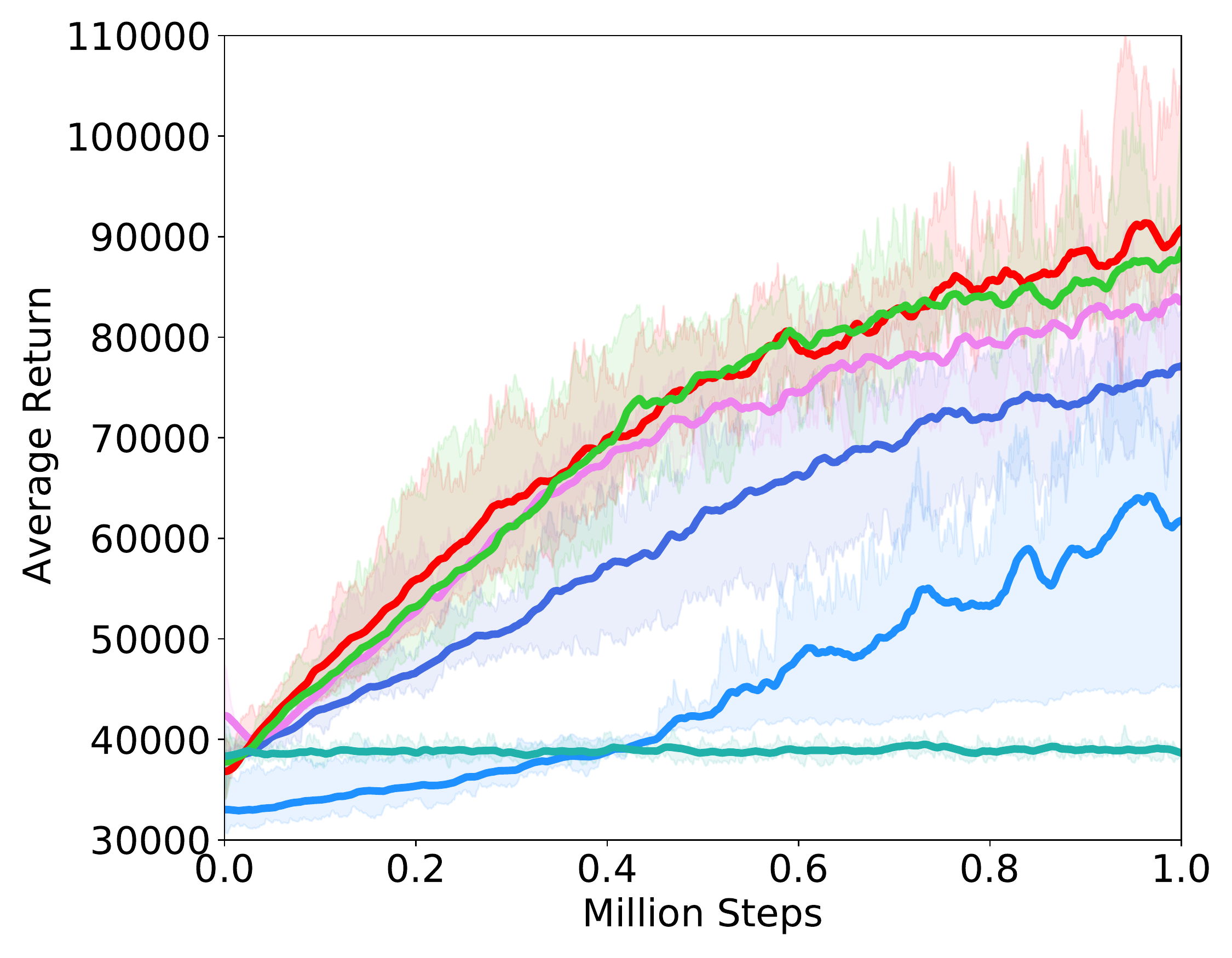}
}
\subfigure[HalfCheetah-v2]{
	\includegraphics[width=0.17\textheight]{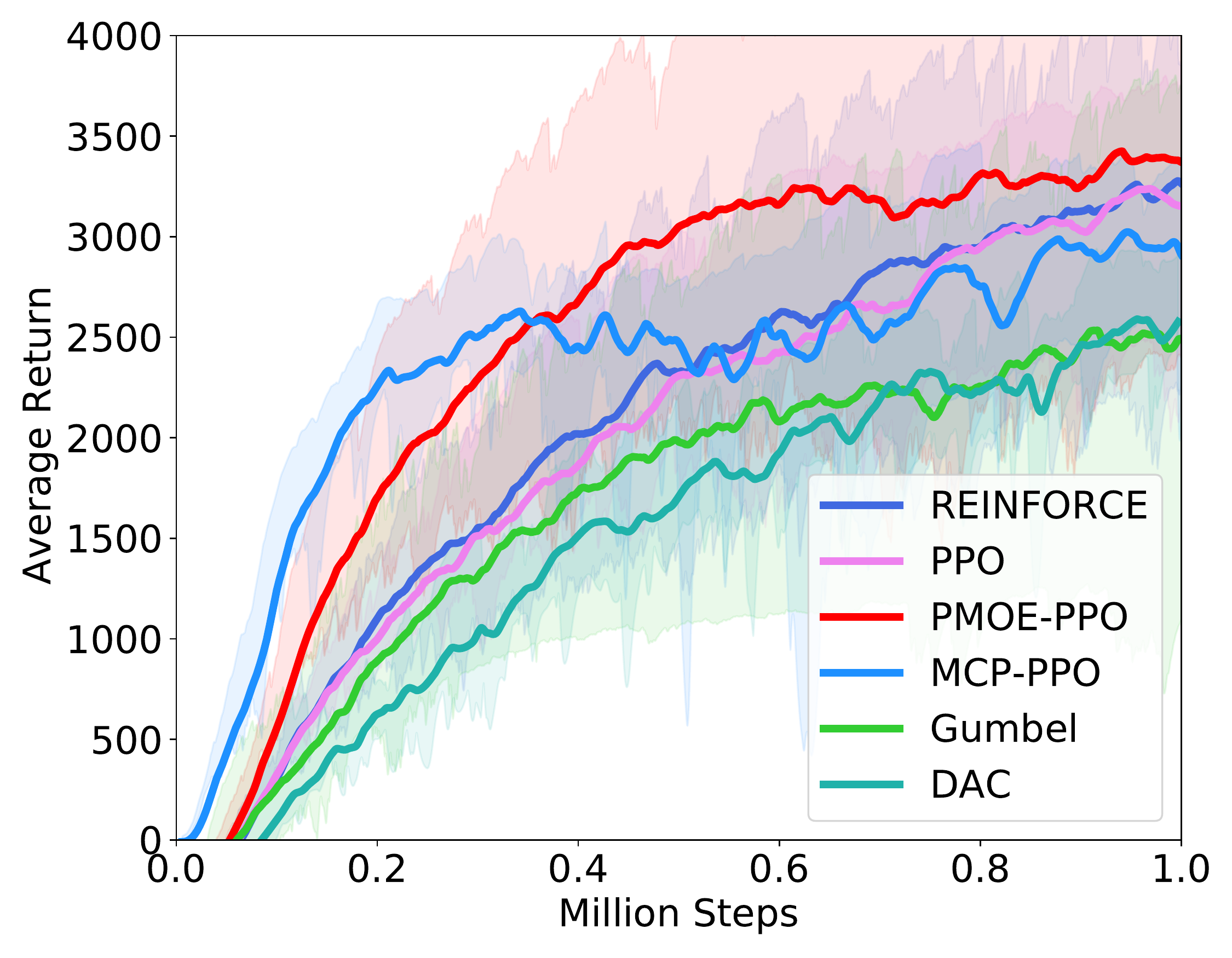}
}
\caption{Training curves on MuJoCo benchmark with PPO-based algorithms. We set a larger number, $K=16$ for \textit{Ant-v2}, $K=12$ for \textit{Hopper-v2}, $K=4$ for \textit{Walker2D-v2} and \textit{HumainoidStandup-v2}, $K=8$ for \textit{Humanoid-v2} and $K=8$ for \textit{HalfCheetah-v2}.
}
\label{fig:ppo_MuJoCo}
\end{figure}

\paragraph{Comparison of AUC}
\label{app:auc}

We compute the AUC (the area under the learning curve) to make the Fig.~\ref{fig:sac_MuJoCo}, Fig.~\ref{fig:ppo_MuJoCo},  Fig.~\ref{fig:compare_k} and Fig.~\ref{fig:k_entropy} more readable.
For SAC-based experiments, we assume the AUC of SAC is 1, and AUC values for all methods are shown in Table~\ref{tab:AUC-sac}.
For PPO-based experiments, we assume the AUC of PPO is 1, and AUC values for all methods are shown in Table~\ref{tab:AUC-ppo}.

\begin{table}[]
\scriptsize
\begin{tabular}{L{1.25cm}|C{0.68cm}C{0.68cm}C{0.73cm}C{0.75cm}C{0.6cm}C{0.6cm}}
          & Walker2D      & Half-Cheetah   & Humanoid      & Humanoid-Standup & Ant           & Hopper        \\ \hline
SAC       & 100.0\%          & 100.0\% & 100.0\%          & 100.0\%            & 100.0\%          & 100.0\%          \\ 
Gating    & 93.4\%           & 95.1\%           & 91.1\%           & 92.5\%             & 88.4\%           & 110.0\%          \\ 
MCP-SAC   & 109.0\%          & \textbf{103.2\%}          & 97.6\%           & 98.3\%             & 96.5\%           & 100.5\%          \\ 
REINFORCE & 60.0\%           & 60.7\%           & 9.8\%            & 73.0\%             & 71.5\%           & 63.7\%           \\ 
Gumbel    & 73.3\%           & 79.7\%           & \textbf{115.7\%} & 96.7\%             & 63.9\%           & 111.9\%          \\ \hline
PMOE-SAC  & \textbf{124.9\%} & 99.6\%           & 113.8\%          & \textbf{110.5\%}   & \textbf{128.4\%} & \textbf{115.5\%}
\end{tabular}
\caption{Comparison of the AUC between PMOE-SAC and other methods on six MuJoCo tasks.}
\label{tab:AUC-sac}
\end{table}

\begin{table}[htbp]
\scriptsize
\begin{tabular}{L{1.25cm}|C{0.68cm}C{0.68cm}C{0.73cm}C{0.75cm}C{0.6cm}C{0.6cm}}
          & Walker2D      & Half-Cheetah   & Humanoid      & Humanoid-Standup & Ant           & Hopper        \\ \hline
PPO       & 100.0\%          & 100.0\%          & 100.0\%          & 100.0\%            & 100.0\%          & 100.0\%          \\ 
DAC       & 73.9\%           & 75.2\%           & 61.0\%           & 57.6\%             & 87.9\%           & 90.9\%           \\ 
MCP-PPO   & 111.1\%          & 118.1\%          & \textbf{113.6\%} & 66.5\%             & \textbf{164.3\%} & 82.7\%           \\ 
REINFORCE & 63.0\%           & 103.6\%          & 102.3\%          & 89.1\%             & 124.6\%          & 90.5\%           \\
Gumbel    & 88.9\%           & 83.4\%           & 100.2\%          & 103.8\%            & 81.2\%    		   & 44.9\%			  \\ \hline
PMOE-PPO  & \textbf{138.4\%} & \textbf{126.6\%} & 107.5\% & \textbf{105.3\%}   & 139.3\%          & \textbf{120.7\%}
\end{tabular}
\caption{Comparison of the AUC between PMOE-PPO and other methods on six MuJoCo tasks.}
\label{tab:AUC-ppo}
\end{table}

\subsection{Experiment Analysis}
We provide in-depth analysis of the proposed PMOE to analysis the differences between PMOE method and other methods in RL process, and estimate the effects of the number of primitives. 
\paragraph{Distinguishable Primitives.}
Firstly, for a straight illustration of the distinguishable primitives,  we use a self-defined \emph{target-reaching} sparse reward environment as a toy example to analyse our method.
In this environment, the agent starts from a fixed initial position, then acts to reach a random target position within a certain range and avoids the obstacles on the path.
Only when the agent reaches the target position successfully, the agent can get the reward. 
As the reward setting is sparse, the exploration can be important, so we also analyse the exploration behaviours in this environment.
Details about this environment is provided in Appendix B. 

Fig.~\ref{fig:trajs_myenv} demonstrates the distinguishable primitives learned with PMOE on the target-reaching environment, 
for providing a simple and intuitive understanding. After the training stage, we sample 10 trajectories for each method and visualise them 
in Fig.~\ref{fig:trajs_myenv}. As we can see in Fig.~\ref{fig:traj_specified_primitive}, PMOE trained in back-propagation-max manner generates 
two distinguishable trajectories for different primitives.

In Fig.~\ref{fig:tsne_hopper}, 
we further demonstrate that PMOE can learn distinguishable primitives on more complex MuJoCo environments with the t-SNE~\citep{TSNE} method for visualisation. 
We sample 10K states $\{s_t;t=1, 2, \cdots, 10K\}$ from 10 trajectories and use t-SNE to perform dimensionality reduction on states and 
visualise the results in Fig.~\ref{fig:tsne_states}. We randomly choose one state cluster and sample actions corresponding to the states 
in that cluster. Then we use t-SNE to perform dimensionality reduction on those actions. The reason for taking a state clustering process 
before the action clustering is to reduce the variances of data in state-action spaces so that we can better visualise the action primitives 
for a specific cluster of states. The visualisation of action clustering with our approach and the gating operation are displayed 
in Fig.~\ref{fig:tsne_gating} and Fig.~\ref{fig:tsne_k4}. More t-SNE visualisations for other MuJoCo environments can be found in Appendix E. 
Our proposed PMOE method is testified to have stronger capability in distinguishing different primitives during policy learning process.

\begin{figure*}[ht!]
\centering
	\begin{tikzpicture}
    \node[anchor=south west,inner sep=0] at (0,0) {
    \begin{minipage}[]{0.5\textheight}
	\centering
	\subfigure[ ]{
		\includegraphics[width=0.1\textheight]{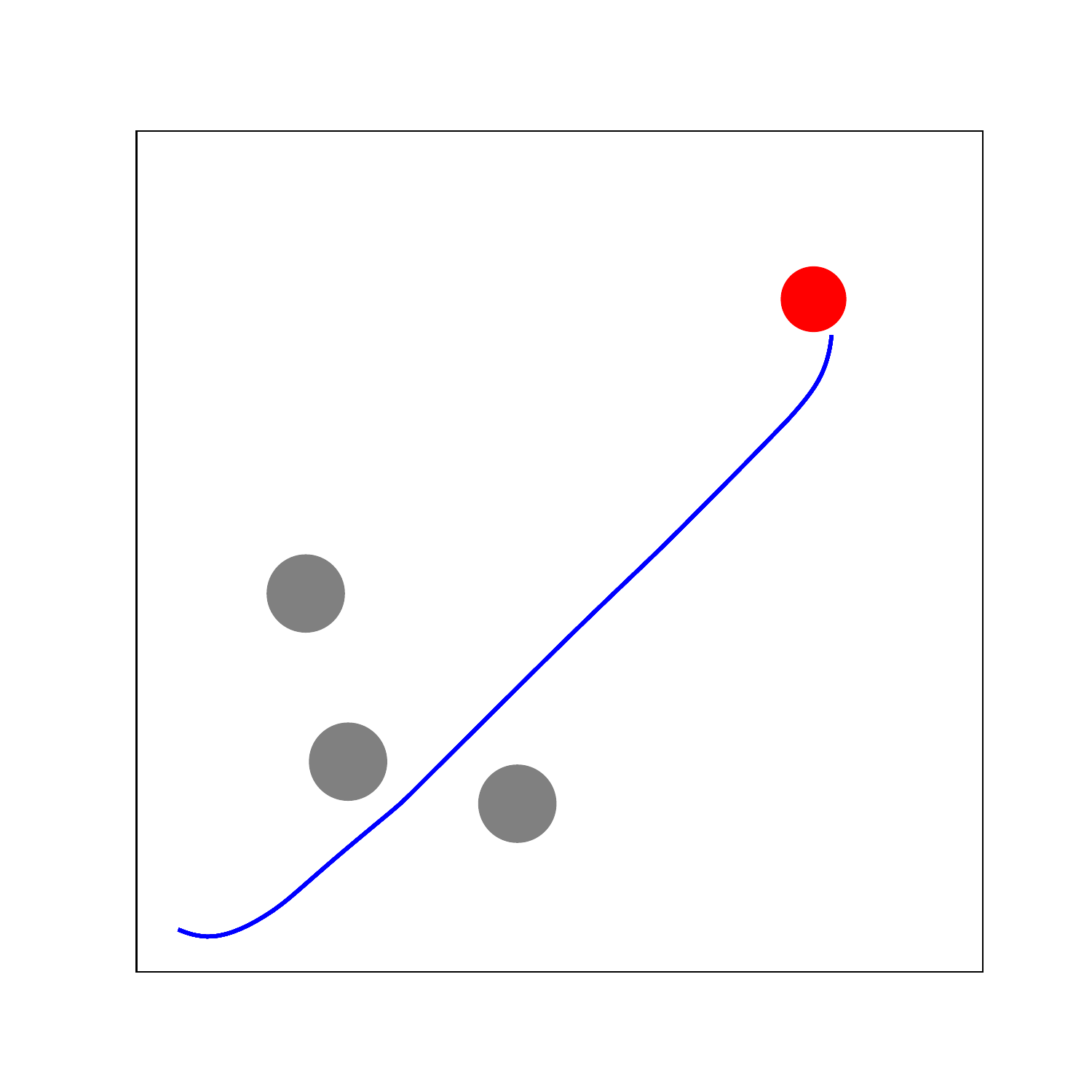}
		\label{fig:traj_k1}
	}
	\subfigure[ ]{
		\includegraphics[width=0.1\textheight]{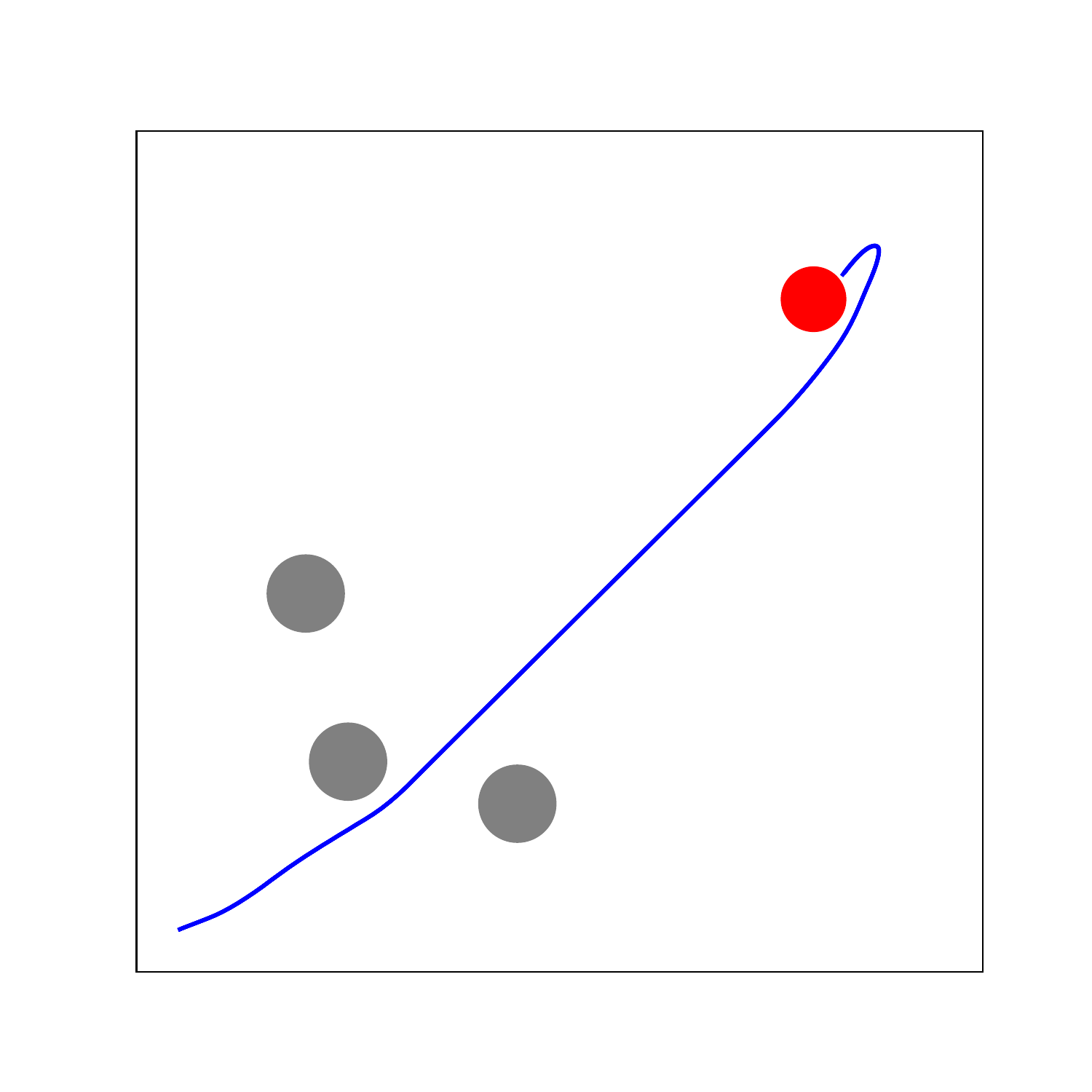}
		\label{fig:traj_k4_gating}
	}
	\subfigure[ ]{
		\includegraphics[width=0.1\textheight]{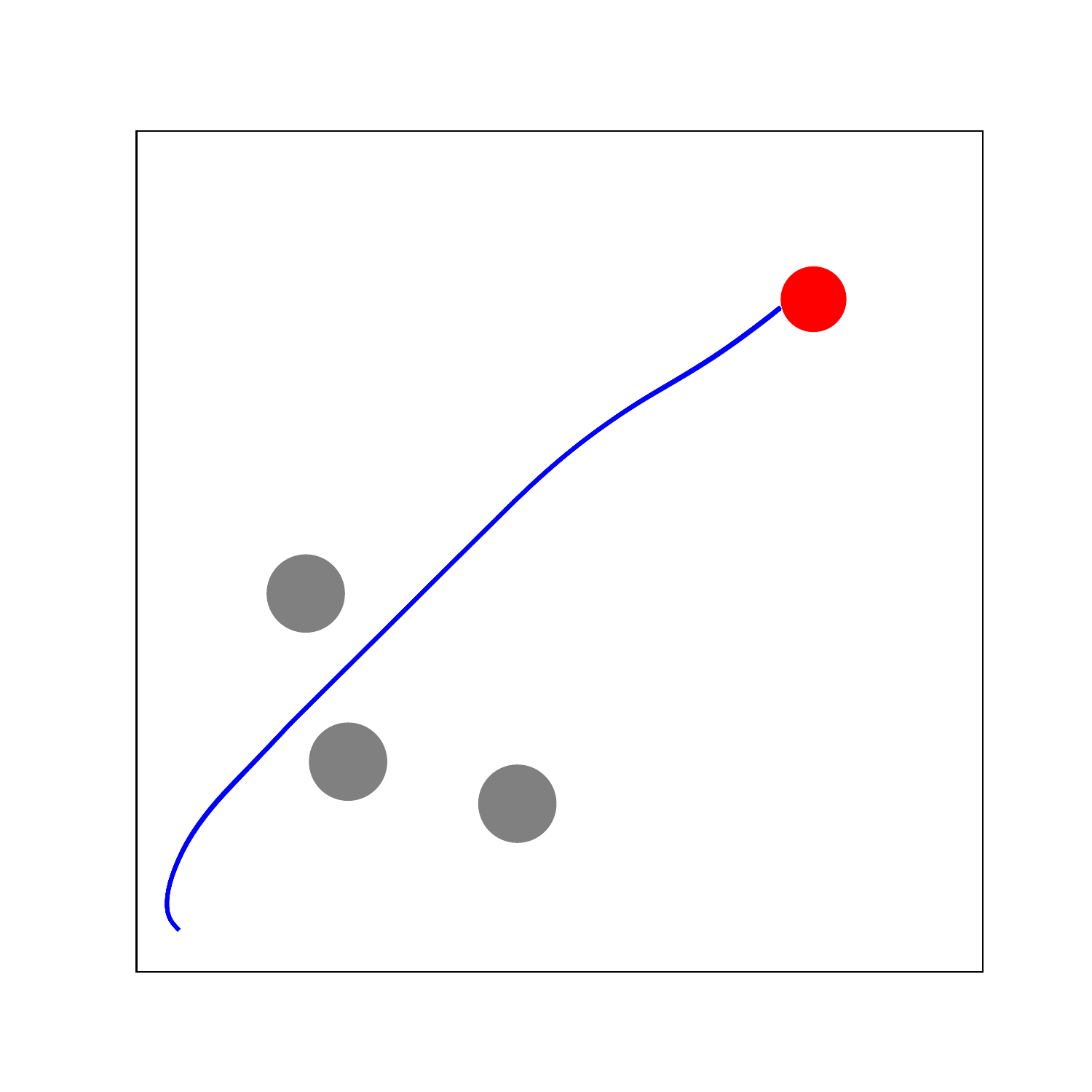}
		\label{fig:traj_k4_together}
	}
	\subfigure[ ]{
		\includegraphics[width=0.1\textheight]{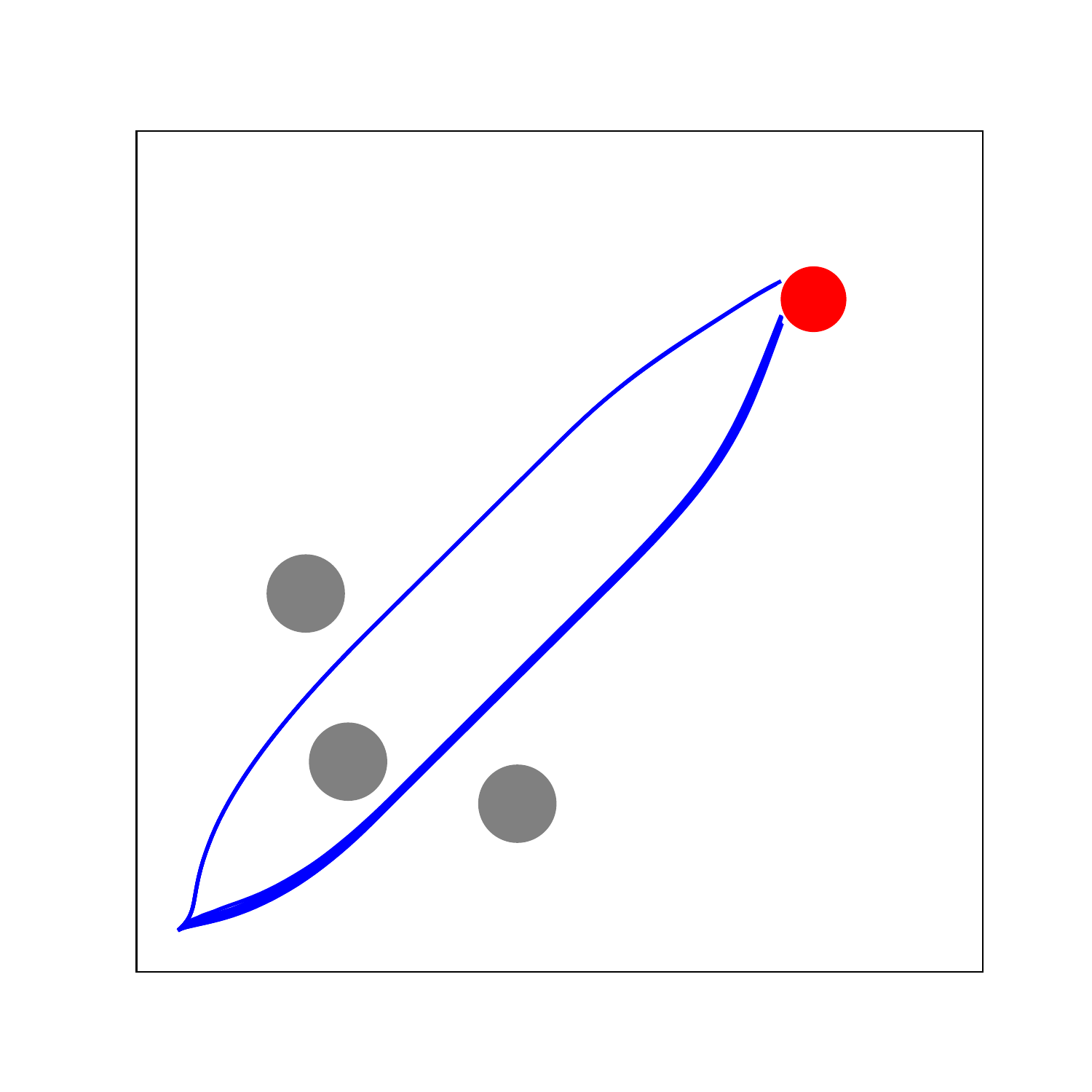}
		\label{fig:traj_k4_single}
	}
	\end{minipage}
	\subfigure[ ]{
    \begin{minipage}[]{0.08\textheight}
    \centering
	\includegraphics[width=0.07\textheight]{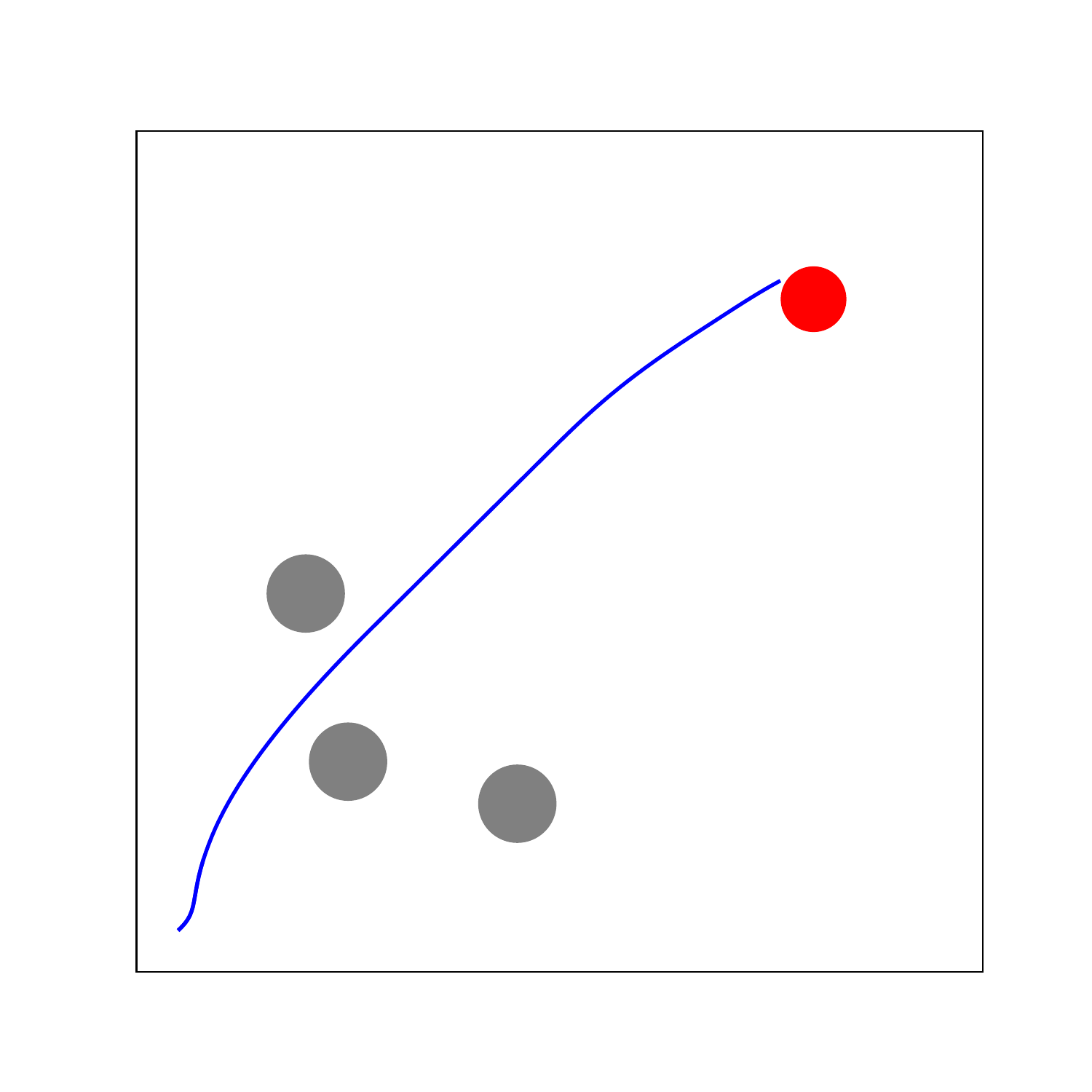}
	\includegraphics[width=0.07\textheight]{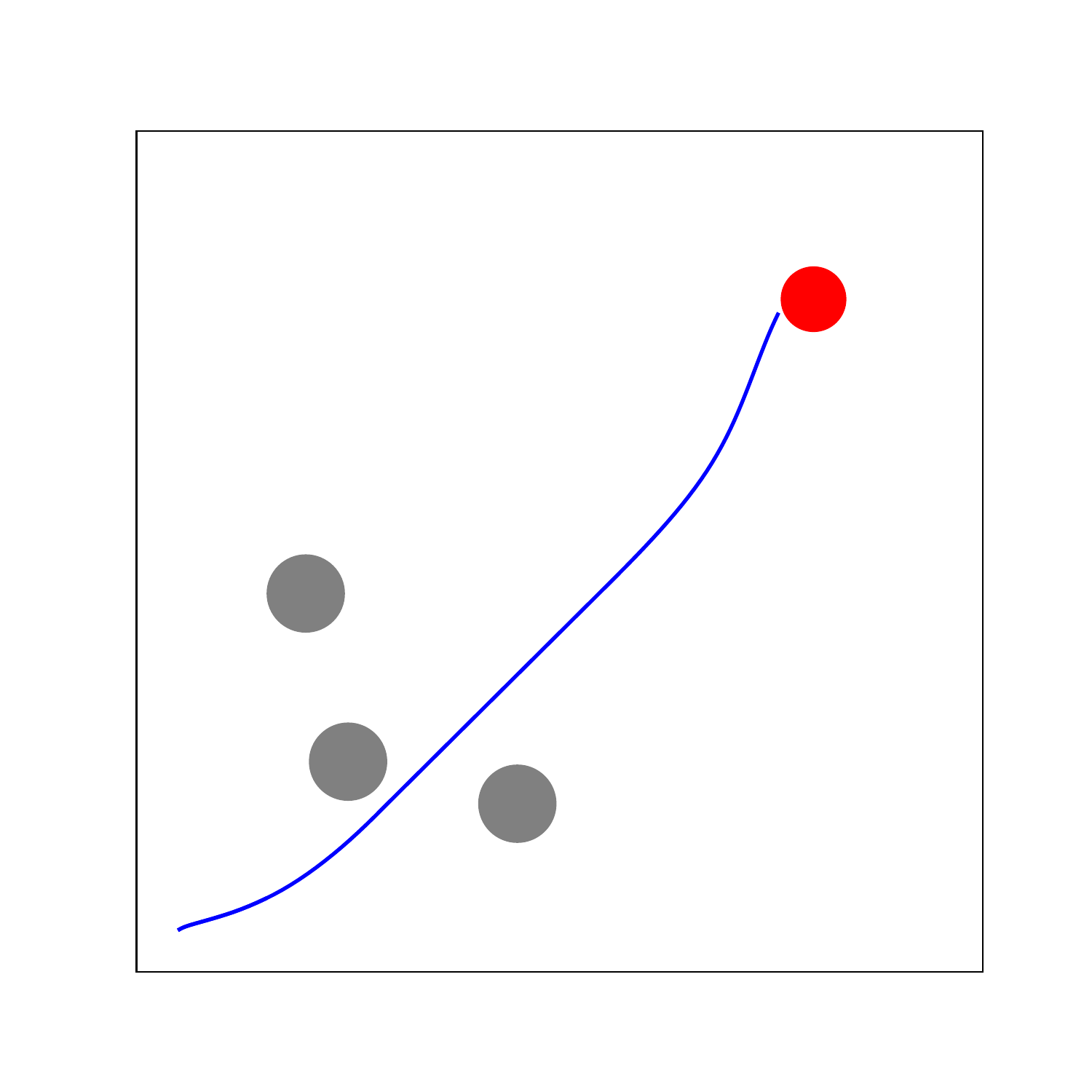}
	\end{minipage}
	\label{fig:traj_specified_primitive}
	}
	};
    \coordinate (start) at (11.3, 3.5);
    \coordinate (end) at (11.3, 0.5);
    
    
    \coordinate (upstart) at (9.5, 2.8);
    \coordinate (upend) at (11.7, 3.0);
    \coordinate (bottomstart) at (9.8, 2.5);
    \coordinate (bottomend) at (11.7, 1.0);
    
    \draw[very thick, dashed, gray] (start) -- (end);
    \draw[very thick, ->] (upstart) to[out=85,in=150] node {} (upend);
    \draw[very thick, ->] (bottomstart) to[out=-45,in=150] node {} (bottomend);
    
	\end{tikzpicture}
	\caption{Trajectories of the agents with our method and the baselines in the target-reaching environment. We fix the reset locations of target, obstacles and agent. ($a$), ($b$), ($c$) and ($d$) visualise the 10 trajectories collected with methods involving: original SAC, gating operation with SAC, back-propagation-all PMOE (discussed in Sec.~\ref{sec:backprog_all}) and back-propagation-max PMOE, respectively.
	($e$) shows the trajectories collected with two individual primitives with our approach. }
	\label{fig:trajs_myenv}
\end{figure*}
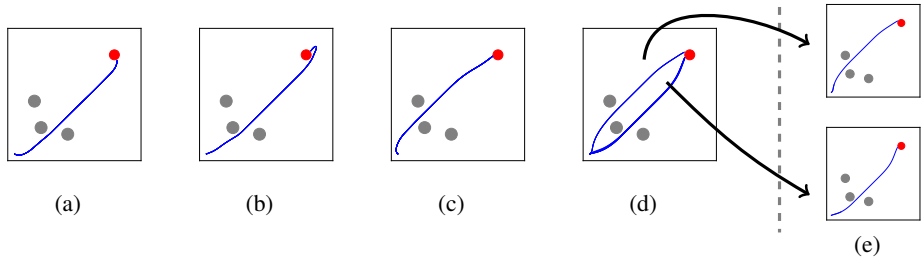

\label{sec:tsne}
\begin{figure*}[ht!]
\centering
\begin{tikzpicture}
\node[anchor=south west,inner sep=0] at (0,0) {
	\subfigure[Gating operation.]{
		\label{fig:tsne_gating}
		\includegraphics[width=0.2\textheight]{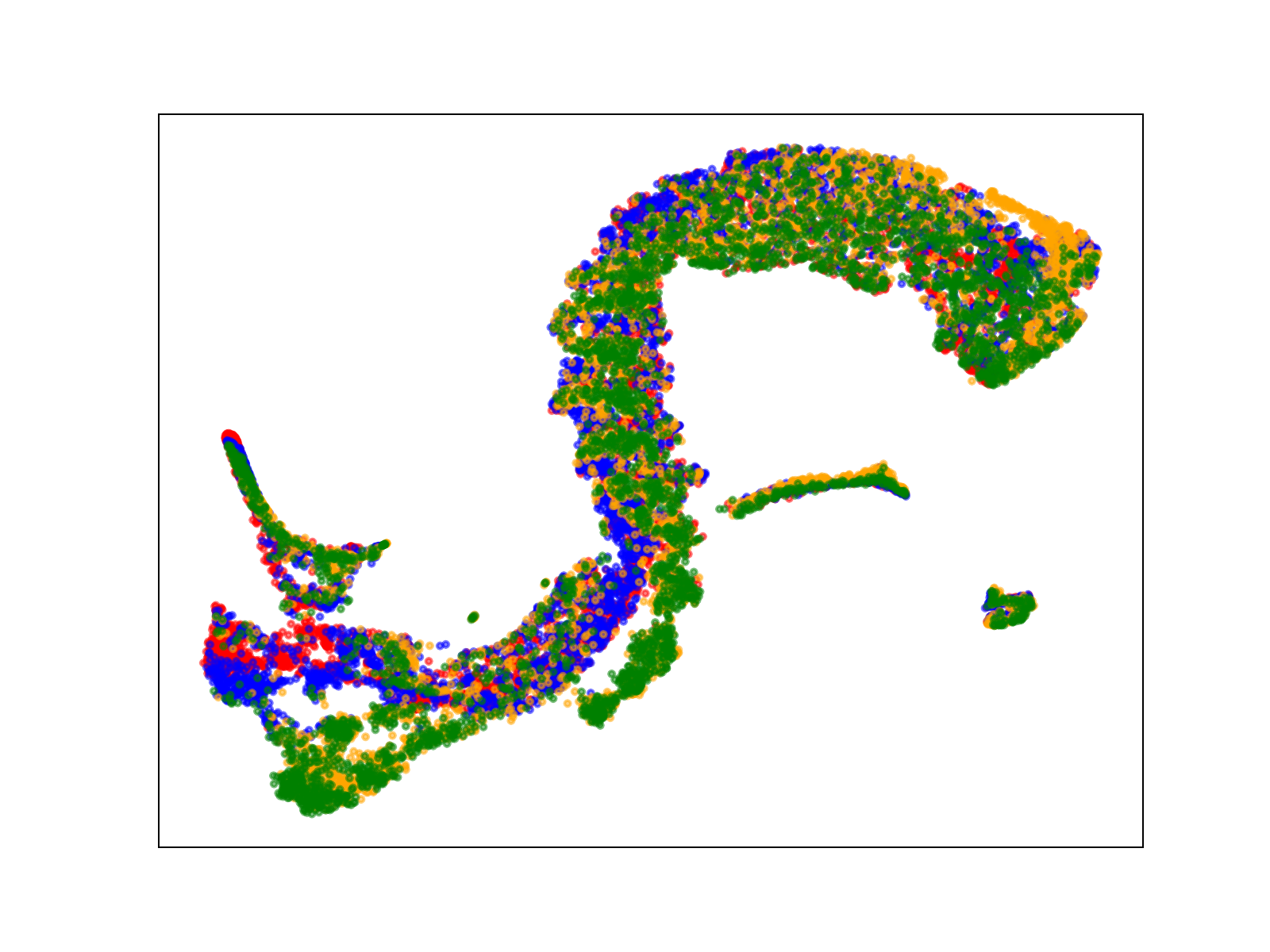}}
	\subfigure[States clusters.]{
		\label{fig:tsne_states}
		\includegraphics[width=0.2\textheight]{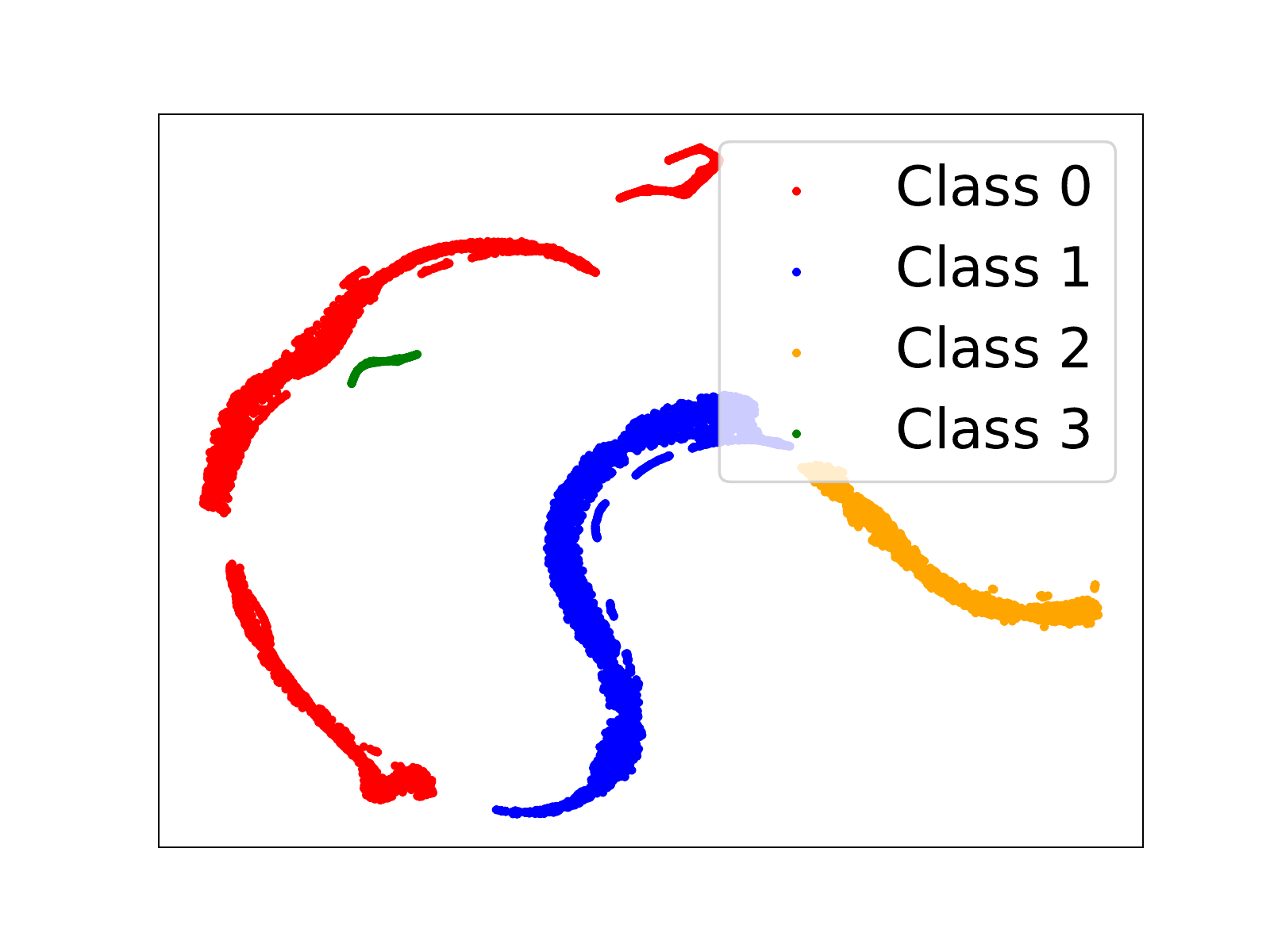}}
	\subfigure[Our approach.]{
		\label{fig:tsne_k4}
		\includegraphics[width=0.2\textheight]{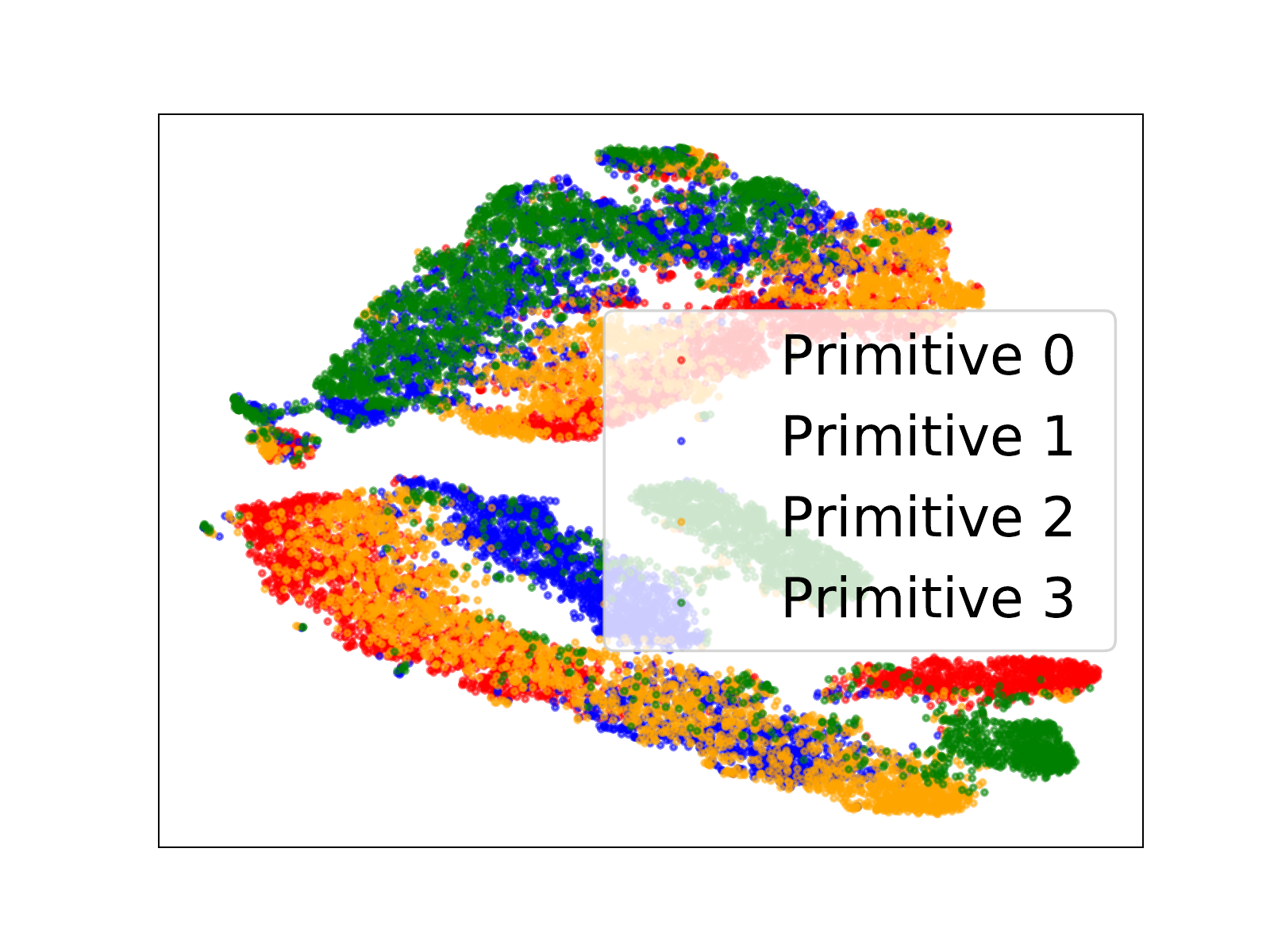}}
};
\coordinate (startleft) at (6.5, 2.5);
\coordinate (startright) at (8, 1.5);
\coordinate (gatingend) at (4.5, 2);
\coordinate (k4end) at (10, 2);
\draw[very thick, ->] (startleft) to[out=-135,in=45] node {} (gatingend);
\draw[very thick, ->] (startright) to[out=-45,in=-135] node {} (k4end);
\draw[dashed, gray, thick, rounded corners] (6.6, 1.0) rectangle (7.8, 2.9);
\end{tikzpicture}
\caption{Visualisation of distinguishable primitives learned with PMOE using t-SNE plot on \textit{Hopper-v2} environment. The states are first clustered as in ($b$). Then actions within the same state cluster are plotted with t-SNE as in ($a$) and ($c$) for the gating method and our approach, respectively. Our method clearly demonstrates more distinguishable primitives for the policy.}
\label{fig:tsne_hopper}
\end{figure*}
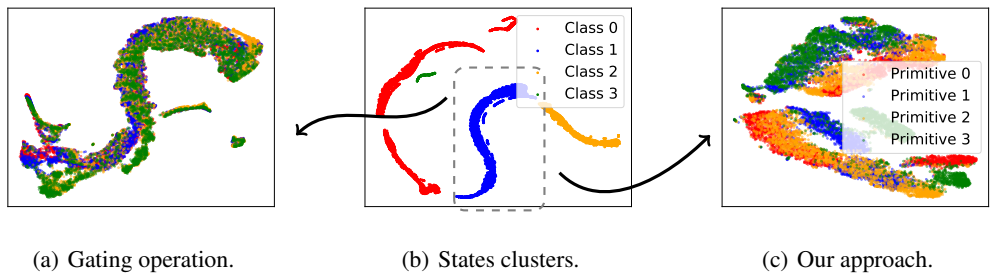

\paragraph{Exploration Behaviours.} Fig.~\ref{fig:compare_traj} demonstrates the exploration trajectories in the target-reaching environment, 
although all trajectories start from the same initial positions, our methods demonstrate larger exploration ranges compared against other 
baseline methods, which also yields a higher visiting frequency to the target region (in green) and therefore accelerates the learning process. 
To some extent, this comparison can be one reason for the improvement of using PMOE as the policy representation over a general unimodal policy. 
We find that by leveraging the mixture models, the agents gain effective information more quickly via different exploration behaviours, 
which cover a larger range of exploration space at the initial stages of learning and therefore ensure a higher ratio of target reaching.

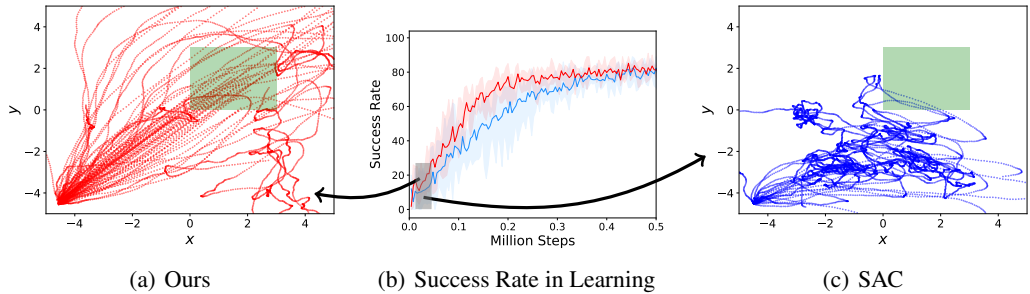
\begin{figure*}[ht!]
	\centering
	\begin{tikzpicture}
    \node[anchor=south west,inner sep=0] at (0,0) {
	\subfigure[Ours]{
		\includegraphics[width=0.2\textheight]{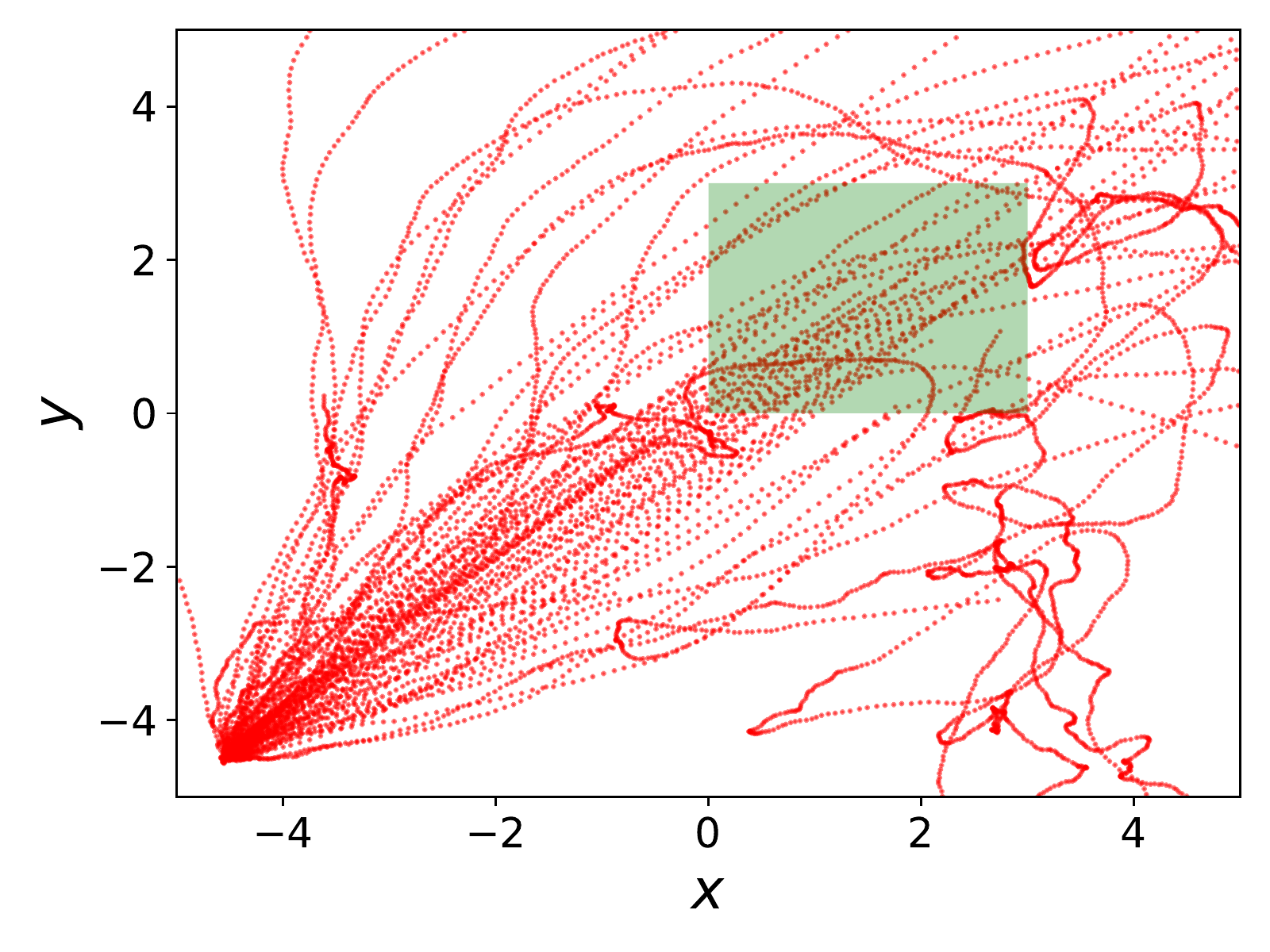}
	}
		\subfigure[Success Rate in Learning]{
		\includegraphics[width=0.18\textheight]{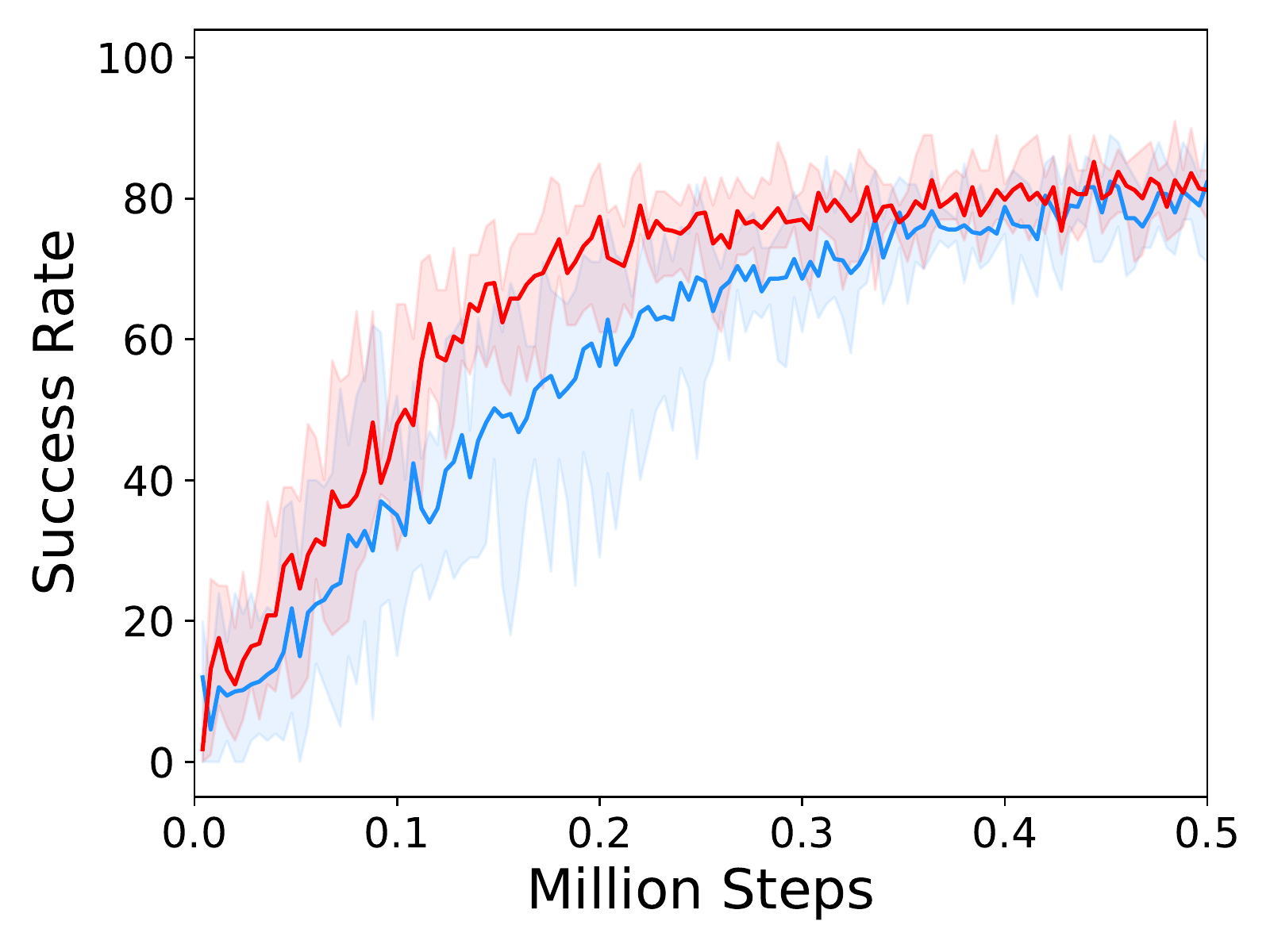}
	}
	\subfigure[SAC]{
		\includegraphics[width=0.2\textheight]{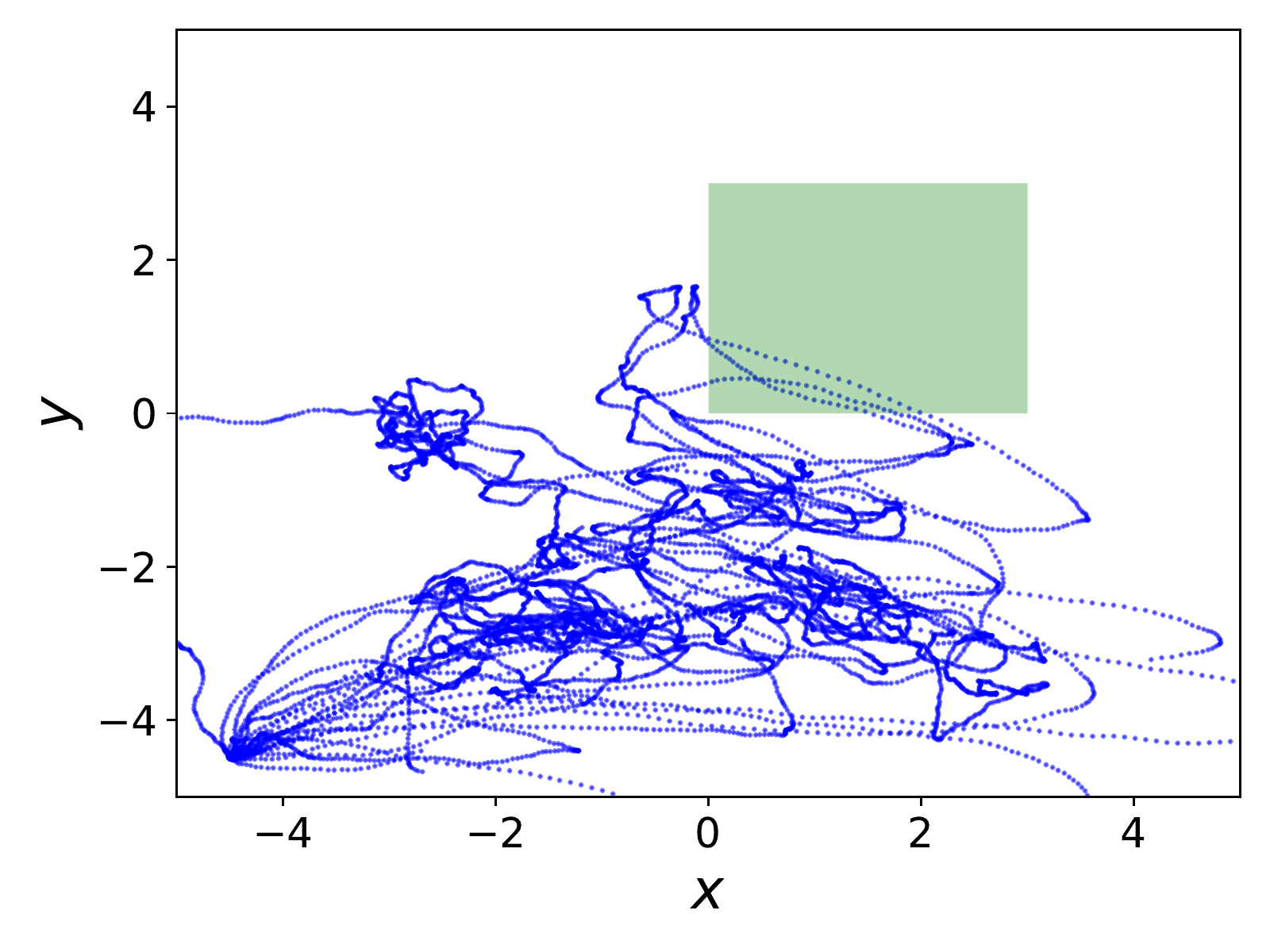}
	}
	};
    \coordinate (startleft) at (5.7, 1.7);
    \coordinate (startright) at (5.75,1.45);
    \coordinate (oursend) at (4.3, 1.5);
    \coordinate (sacend) at (9.5, 2);
    \draw[very thick, ->] (startleft) to[out=-150,in=-10] node {} (oursend);
    \draw[very thick, ->] (startright) to[out=-10,in=-150] node {} (sacend);
    \filldraw[gray, opacity=0.5] (5.65, 1.3) rectangle (5.85, 1.9);
	\end{tikzpicture}
	\caption{Visualisation of exploration trajectories in the initial training stage for the target-reaching environment. The initial $10K$ steps (the grey region on the learning curves in ($b$)) of exploration trajectories are plotted in ($a$) and ($c$) for our PMOE method (red) and SAC (blue), respectively. The green rectangle is the target region.}
	\label{fig:compare_traj}
\end{figure*}

\begin{figure}[ht!]
    \begin{center}
        \includegraphics[scale=0.21]{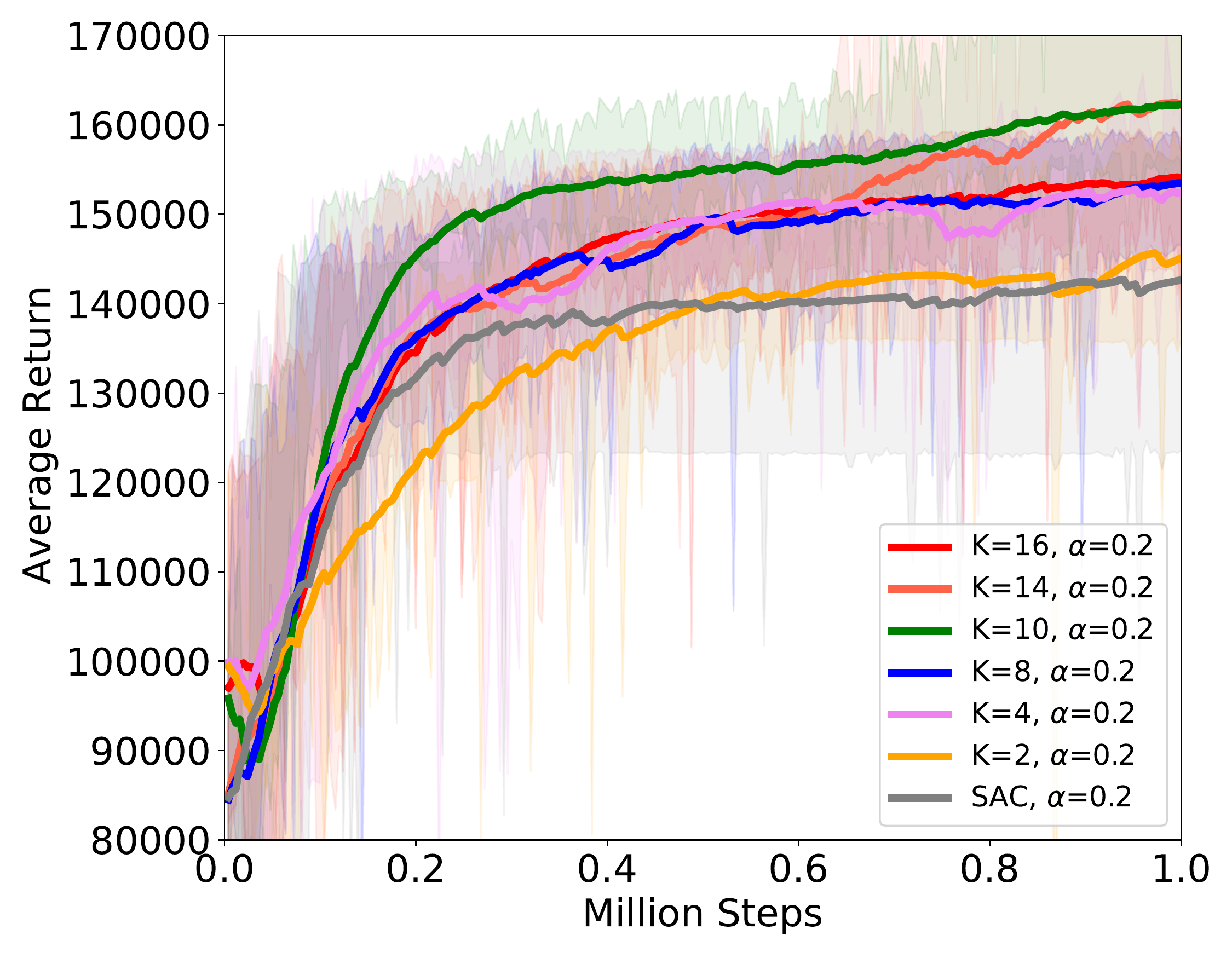}
    \end{center}
    \caption{
        Comparison of different numbers of primitives $K$ in terms of average returns on \textit{HumanoidStandup-v2} environment. For each case, we conduct 5 runs and take the means. The performance increases when $K$ increase from $2$ to $10$, but decreases if $K$ keep increasing from $10$ to $16$.
    }
    \label{fig:compare_k}
\end{figure}

\paragraph{Number of Primitives.} We investigate the effects caused by different numbers of primitives in GMM, 
as shown in Fig.~\ref{fig:compare_k}. This experiment is conducted on a relatively complex environment --- \textit{HumanoidStandup-v2} 
that has observations with a dimension of 376 and actions with a dimension of 17, therefore various skills could possibly lead to the goal of the task. The number of primitives $K$ is selected from $\{2, 4, 8, 10, 14, 16\}$ and other hyperparameters are the same. The results show that $K=10$ seems to perform the best, and $K=2$ performs the worst among all settings, showing that increasing the number of primitives can improve the learning efficiency in some situations. For Table~\ref{tab:AUC-k}, we assume the AUC of PMOE-SAC with $K=2$ is 1, and relative AUC values for all the methods are displayed.



\begin{table}[ht!]
\small
    \centering
    \begin{tabular}{c|c}
    \textbf{Number of K}&   AUC  \\
    \hline
    \textbf{2}  & 100\% \\    
    \textbf{4}  & 107.1\% \\    
    \textbf{8}  & 106.5\% \\    
    \textbf{10}  & \textbf{111.7}\% \\    
    \textbf{14}  & 108.0\% \\    
    \textbf{16}  & 107.0\% \\ \hline
    \end{tabular}
    \label{tab:AUC-k}
    \caption{Comparison of the AUC as a function of $K$ for PMOE-SAC algorithm on \textit{HumanoidStandup-v2} environment.}
\end{table}

To analyse the relationship of different $K$ and different entropy regularisation($\alpha$), 
we also compared 7 settings in the MuJoCo task \textit{HumanoidStandup-v2}, where $K$ is the number of primitives and $\alpha$ is the entropy regularisation. For each setting, we randomly choose 5 seeds to plot the learning curves in  Fig.~\ref{fig:k_entropy}. 
For  Table~\ref{tab:AUC-entropy},  weassume the AUC of SAC with $\alpha$ =  10 is 1, and relative AUC values for all the methods are displayed.

\begin{table}[ht!]
\small
\centering
    \begin{tabular}{l|c}
    Settings           & AUC              \\ 
                      \hline
    SAC, $\alpha$=10   & 100.0\%          \\
    SAC, $\alpha$=1    & 121.7\%          \\
    SAC, $\alpha$=0.2  & 108.3\%          \\
    SAC, $\alpha$=0.05 & 108.5\%          \\
    K=4, $\alpha$=10   & 57.3\%           \\
    K=4, $\alpha$=1    & \textbf{125.7\%} \\
    K=4, $\alpha$=0.2  & 114.7\%          \\
    K=4, $\alpha$=0.05 & 116.0\%          \\ \hline
\end{tabular}
\label{tab:AUC-entropy}
\caption{Comparison of the AUC as a function of the relation between $K$ and $\alpha$ for PMOE-SAC algorithm on \textit{HumanoidStandup-v2} environment.}
\end{table}

\begin{figure}[ht!]
    \centering
    \includegraphics[width = 0.3\textwidth]{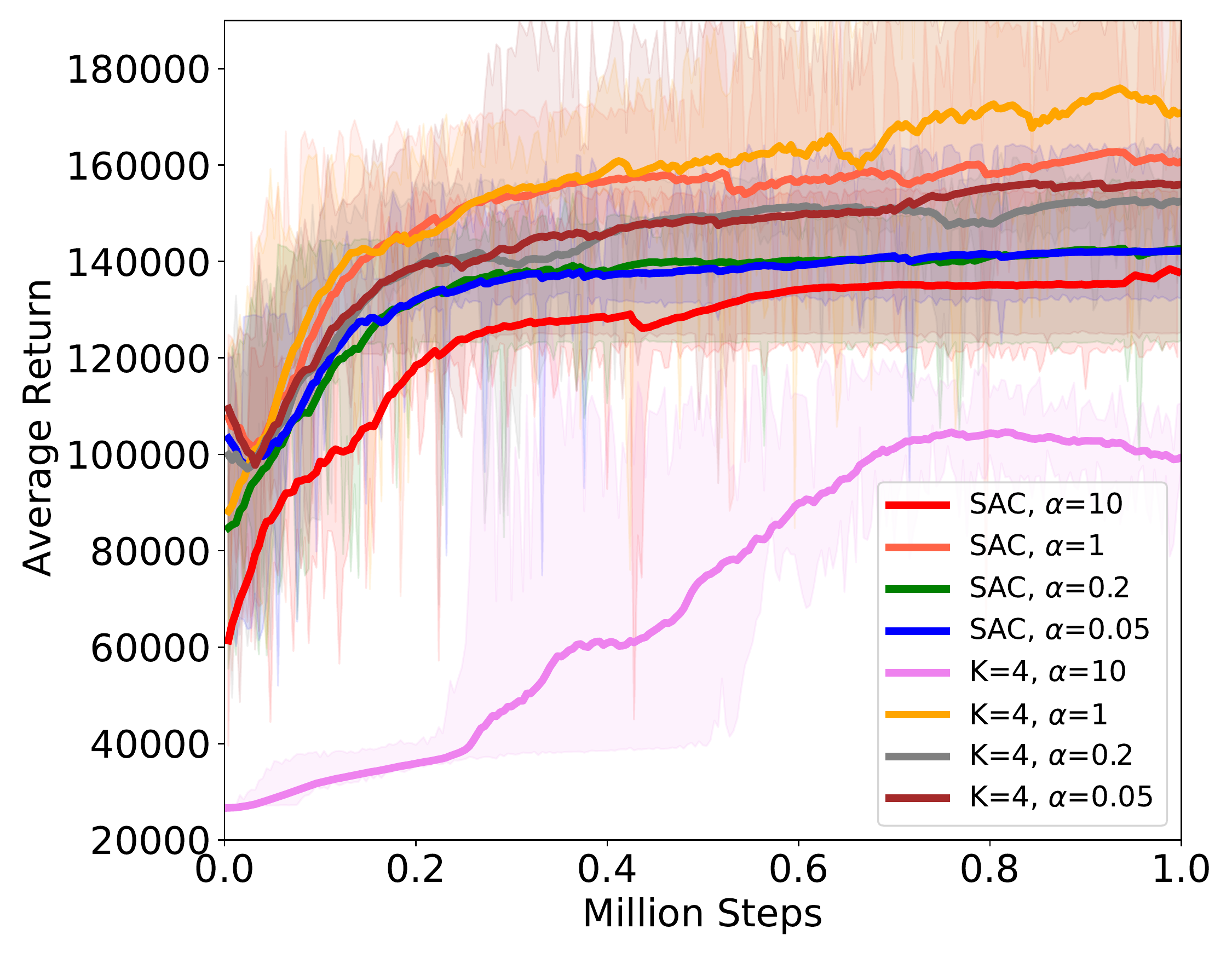}
    \caption{Comparison with different $K$ and different amounts of entropy regularisation. Our approach can be considered as a kind of entropy regularisation method and the number of primitives is positively correlated with the entropy of the policy. The larger number of primitives with smaller entropy has a similar performance to the smaller number of primitives with larger entropy.}
    \label{fig:k_entropy}
\end{figure}

\paragraph{Robustness Evaluation}
To evaluate the robustness of our approach, we develop an experiment on the \textit{Hopper-v2} environment. We add a random noise $\epsilon \sim \mathcal{N}(0, \sigma^2)$ to the state observation $s$, and use the noised state observation $\hat{s} = s + \epsilon $ as the input of the policy. Our approach has a more stable performance in the noised input observation situation, which is shown in Table~\ref{tab:robustness}.

\begin{table}[ht!]
\scriptsize
\centering
\begin{tabular}{l|lll}
Method & $\sigma=0$               & $\sigma=0.05$            & $\sigma=0.1$             \\ \hline
SAC    & 3387.4 $\pm$ 2.0          & 1994.9 $\pm$ 718.6          & 1006.2 $\pm$ 389.6          \\
gating & 3444.8 $\pm$ 3.1          & 2606.9 $\pm$ 864.4          & 1626.0 $\pm$ 771.9          \\
MCP    & 3524.8 $\pm$ 114.6          & 1610.2 $\pm$ 357.0          & 1008.4 $\pm$ 333.4          \\
Gumbel-Softmax &3248.3$\pm$453.6 &3042.9 $\pm$937.5 &1347.4$\pm$855.0 \\
REINFORECE  &1741.5$\pm$618.1 &1111.6$\pm$398.1 &558.7$\pm$276.5 \\
 \hline
Ours   & 3632.2 $\pm$ 4.0 & 3460.4 $\pm$ 456.7 & 1730.0 $\pm$ 703.1
\end{tabular}
\caption{We test our approach in the \textit{Hopper-v2} environment, each column stands for the average return with different variances of the noise distribution, the average return of each methods is averaged over 100 rounds.}
\label{tab:robustness}
\end{table}

\section{Conclusion}
To cope with the problems of low learning efficiency and multimodal solutions in continuous control tasks when applying DRL, this paper proposes the differentiable PMOE method that enables an end-to-end training scheme for generic RL algorithms with stochastic policies. Our proposed method is compatible with policy-gradient-based algorithms, like SAC and PPO. Experiments show performance improvement across various tasks is achieved by applying our PMOE method for policy approximation, as well as displaying distinguishable primitives for multiple solutions.


\bibliography{example_paper}
\bibliographystyle{icml2021}

\appendix

	\section{Probabilistic formulation of PMOE and Gating Operation}
	\label{app:gating}
	\label{app:diff_pdf}
	In this section, we show a detailed comparison of probabilistic formulation for GMM (as Eq.~(\ref{eq:gmm_a}) and (\ref{eq:PDF_GMM}), used in our PMOE method) and the gating operation method (Eq.~(\ref{eq:gating_origin}) to (\ref{eq:PDF_gating})), in term of their PDFs. The gating operation degenerates the multimodal action to a unimodal distribution, which is different from our PMOE method.
	
	For \textbf{GMM}, suppose a primitive $\pi_i(s)$ is a Gaussian distribution $\mathcal{N}(a|\mu(s), \sigma^{2}(s))$, drawing a sample from the mixture model can be seen as the following operation:
	\begin{equation}
	\label{eq:gmm_a}
	\begin{aligned}
			a \sim \pi(a|s) 
			&= \sum_{i=1}^{K}w_i(s)\pi_i(a|s)\\
			&= \sum_{i=1}^{K}w_i(s)\mathcal{N}(a|\mu_i(s), \sigma_i^{2}(s)),
	\end{aligned}	
    \end{equation}
	where the PDF is:
	\begin{equation}
		p(a) = \sum_{i=1}^{K}\frac{w_i(s)}{\sqrt{2\pi}\sigma_i(s)}\exp\{-\frac{(a-\mu_i(s))^2}{2\sigma_i^2(s)}\}.
	\label{eq:PDF_GMM}
	\end{equation}
	For \textbf{gating operation}, the outputs of the weight operation are the weights of each action from different primitives. 
	With those weights, the gating operation uses the weighted action as the final output action according to \cite{MCP}:
	\begin{equation}
	\label{eq:gating_origin}
	a = \sum_{i=1}^{K}w_i(s)a_i, \, s.t.\, a_i \sim \pi_i(a|s)=\mathcal{N}(a|\mu_i(s), \sigma_i^{2}(s)).
	\end{equation}
	As a primitive is a Gaussian distribution, Eq. \ref{eq:gating_origin} becomes:
	\begin{equation}
	a \sim \mathcal{N}(a|\sum_{i=1}^{K}w_i(s)\mu_i(s), \sum_{i=1}^{K}w_i(s)\sigma_i^{2}(s)),
	\end{equation}
	where the PDF is:
	\begin{equation}
	\label{eq:PDF_gating}
	\begin{split}
	p(a) = & \frac{1}{\sqrt{2\pi\sum_{i=1}^{K}w_i(s)\sigma_i^2(s)}} \\
	&\exp\{-\frac{(a-\sum_{i=1}^{K}w_i(s)\mu_i(s))^2}{2\sum_{i=1}^{K}w_i(s)\sigma_i^2(s)}\},
	\end{split}
	\end{equation}
	The above PDF shows that the gating operation could degenerate the Gaussian mixture model into the univariate Gaussian distribution.
	Other methods~\citep{AdaptiveMOE, MOERL_ACM16, FEUDAL_ICML17} also have the similar formulation.
	
	\section{Details of Target-Reaching Environment}
	\label{app:details_target_reaching}
	The visualisation of the target-reaching environment is shown in Fig.\ref{fig:myenv_vis}, the blue circle is the agent, the gray circles are obstacles and the circle in red is the target. The agent state is represented by an action vector $a=[a_x, a_y]$ and a velocity vector $v=[v_x, v_y]$. The playground of the environment is continuous and limited in $[-5, 5]$ in both x-axis and y-axis. The agent speed is limited into $[-2, 2]$, the blue coloured agent is placed at position $[x_{ag}, y_{ag}] = [-4.5, -4.5]$ and the red coloured target is randomly placed at position $[x_{tg}, y_{tg}]$, where $x_{tg}, y_{tg} \sim \mathcal{U}(0, 3)$ and $\mathcal{U}$ denotes uniform distribution. There are $M$ gray coloured obstacles with each position $[x_{obs}^i, y_{obs}^i]$, where $x_{obs}^i, y_{obs}^i\sim\mathcal{N}(0, 3^2)$ and $\mathcal{N}$ denotes Gaussian distribution. The observation is composed of $[[x_{tg} - x_{ag}, y_{tg} - y_{ag}], \{[x_{obs}^i - x_{ag}, y_{obs}^i - y_{ag}^i]; i=1, 2, \cdots, M\}, a, v]$. The input action is the continuous acceleration $a$ which is in the range of $[-2, 2]$.
	\begin{figure}[ht!]
		\centering
		\includegraphics[width=0.195\textheight]{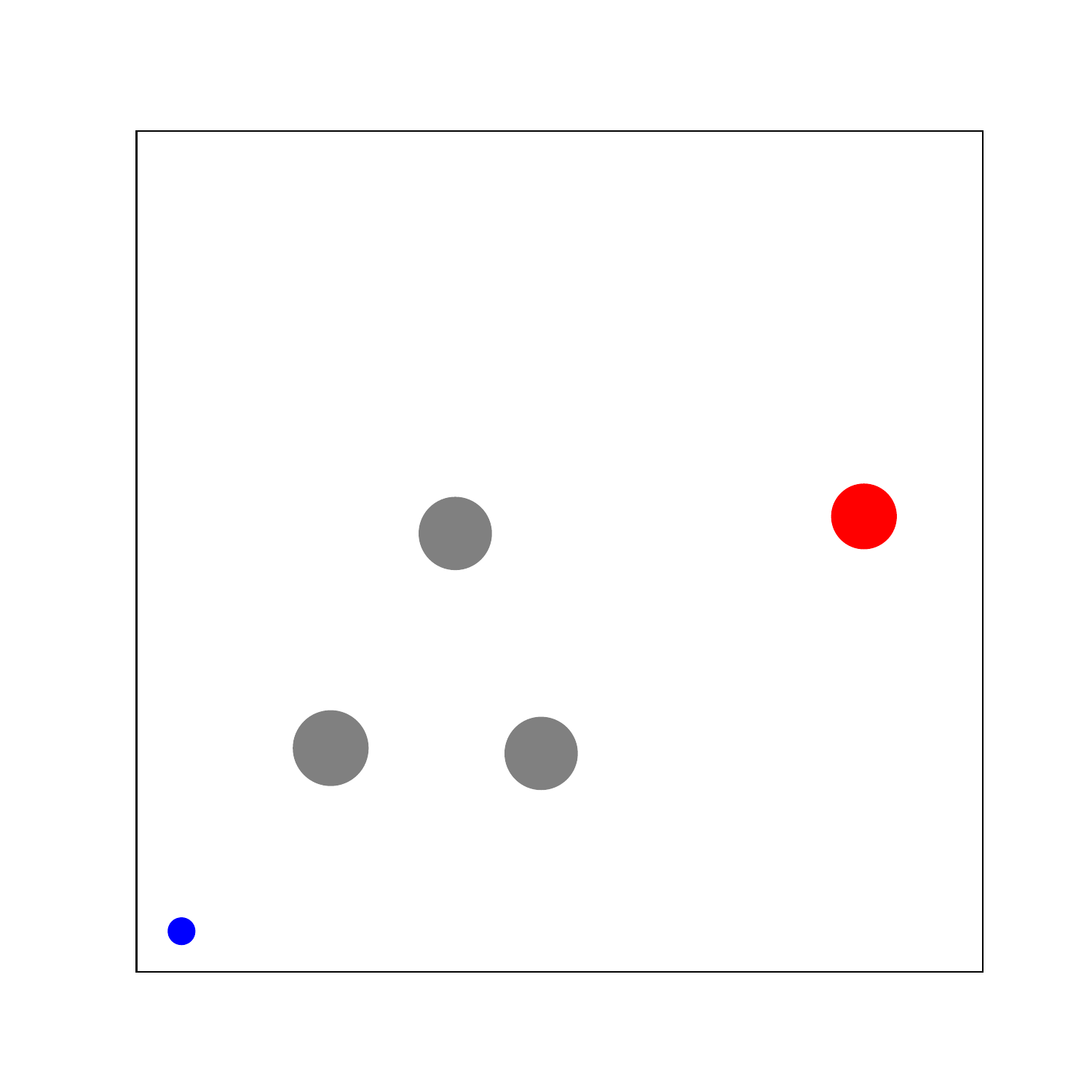}
		\caption{Visualisation of the target-reaching environment}
		\label{fig:myenv_vis}
	\end{figure}
	
	The immediate reward function for each time step is defined as:
	\begin{equation}
	r = \left\{
	\begin{aligned}
	&100, \,\text{if the agent reaches the target};\\
	&-10, \, \text{if the agent collides with edges or obstacles};\\
	&||v||_2, \text{otherwise}.
	\label{eq:myenv_reward}
	\end{aligned}
	\right.
	\end{equation}
	
	\section{Training Details}
	\label{app:training_details}
	For PMOE-SAC policy network, we use a two-layer fully-connected (FC) network with 256 hidden units and ReLU activation in each layer. For primitive network $\pi_{\psi}$, we use a two single-layer FC network, which outputs $\mu$ and $\sigma$ for the Gaussian distribution. Both the output layers for $\mu$ and $\sigma$ have the same number of units, which is $K * dim(\mathcal{A})$, with $K$ as the number of primitives and $dim(\mathcal{A})$ as the dimension of action space, \emph{e.g.}, 17 for \textit{Humanoid-v2}. For the routing function network $w_{\theta}$, we use a single FC layer with $K$ hidden units and the softmax activation function.. In critic network we use a three-layer FC network with 256, 256 and 1 hidden units in each layer and ReLU activation for the first two layers. Other hyperparameters for training are showed in Tab. 1(a). For PMOE-PPO, we use a two-layer FC network to extract the features of state observations. The FC layers have 64 and 64 hidden units with ReLU activation. The policy network has a single layer with the Tanh activation function. The routing function network has a single FC layer with $K$ units and the softmax activation function. The critic contains one layer only. Other training hyperparameters are showed in Tab. 1(b). We use the same hyperparameters in all the experiments without any fine-tuning. For MCP-SAC, we use the same network structure as MCP-PPO, other training hyperparemeters are the same as shown in Tab. 1(a). For other baselines, we use original hyperparameters mentioned in their paper.
	The full algorithm is summarised in Algorithm 1.
	\begin{table}[ht!]
		\centering
		\subtable[Hyperparameters for PMOE-SAC]{
		\begin{tabular}{l|l}
			\hline
			Parameter & Value \\ \hline
			optimiser & Adam \citep{Adam} \\
			learning rate & $10^{-3}$ \\
			discount ($\gamma$) & 0.99 \\
			replay buffer size & $10^6$ \\
			alpha & 0.2 \\
			batch size & 100 \\
			polyak ($\tau$) & 0.995 \\
			episode length & $10^3$ \\
			target update interval & 1 \\\hline
		\end{tabular}
		\label{tab:params_SAC}
		}
		\subtable[Hyperparameters for PMOE-PPO]{
		\begin{tabular}{l|l}
			\hline
			Parameter & Value \\ \hline
			optimiser & Adam \citep{Adam} \\
			learning rate & $3*10^{-4}$ \\
			discount ($\gamma$) & 0.99 \\
			alpha & 0.2 \\
			batch size & 64 \\
			polyak ($\tau$) & 0.95 \\
			episode length & $2*10^3$ \\
			gradient clip & 0.2 \\
			optimisation epochs & 20 \\\hline
		\end{tabular}
		\label{tab:params_PPO}
		}
		\caption{Hyperparameters}
	\end{table}

	\section{Probability Visualisation}
	
	We visualise the probabilities of each primitive over the time steps in the MuJoCo \textit{HalfCheetah-v2} environment. As shown in Fig.~\ref{fig:periodical}, we found that the probabilities are changed periodically. We also visualise the actions at the selected 5 time steps in one period. As shown in Fig.~\ref{fig:one_period}, the primitives are distinguishable enough to develop distinct specialisations.
	\begin{figure}[ht!]
		\centering
		\includegraphics[width=0.35\textheight]{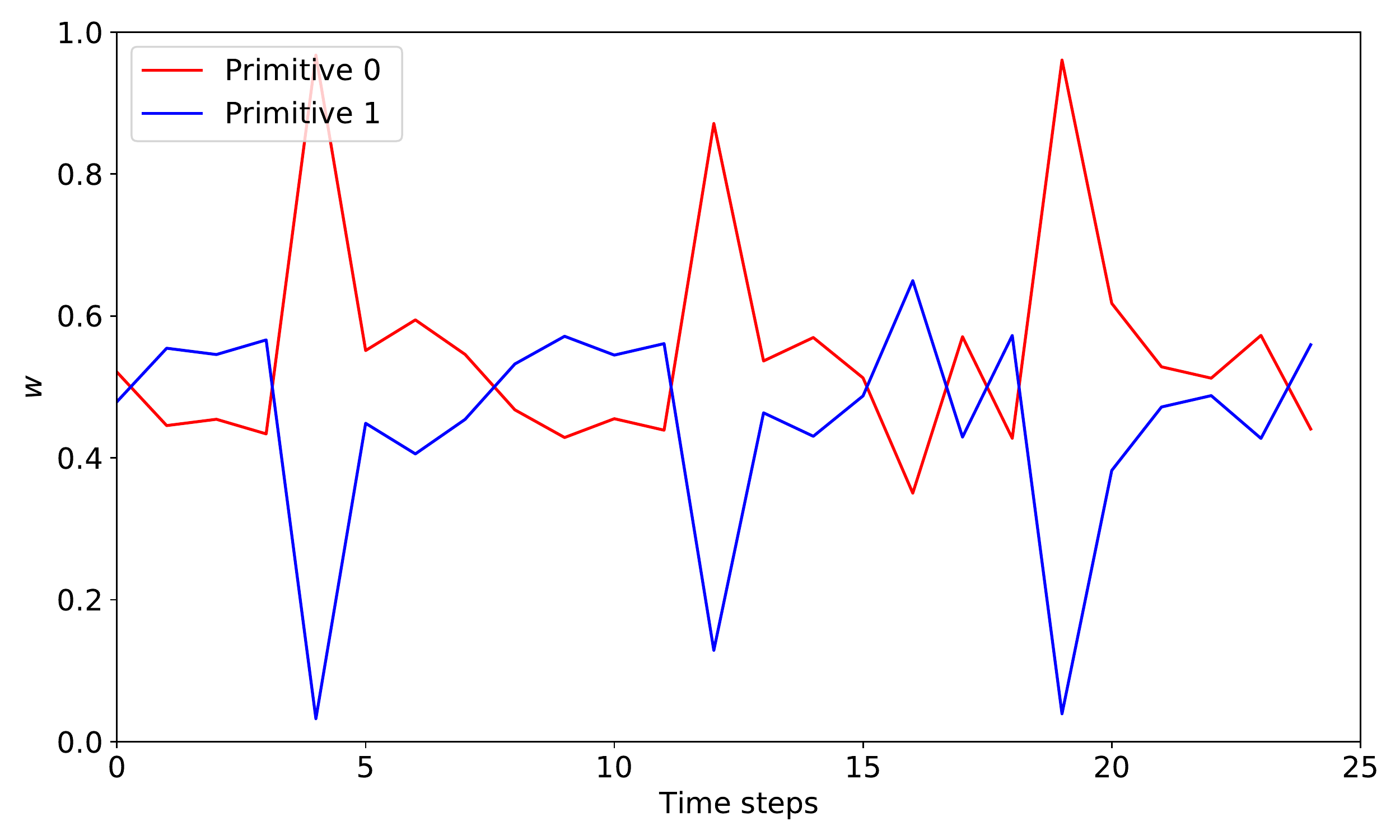}
		\caption{Visualisation of the probabilities of each primitive over the time steps in the MuJoCo \textit{HalfCheetah-v2} environment. The y-axis shows the probabilities of different primitives.}
		\label{fig:periodical}
	\end{figure}
	
	\begin{figure*}[ht!]
		\centering
		
		\subfigure[time step 1]{
			\label{time step 1}
			\includegraphics[width=0.10\textheight]{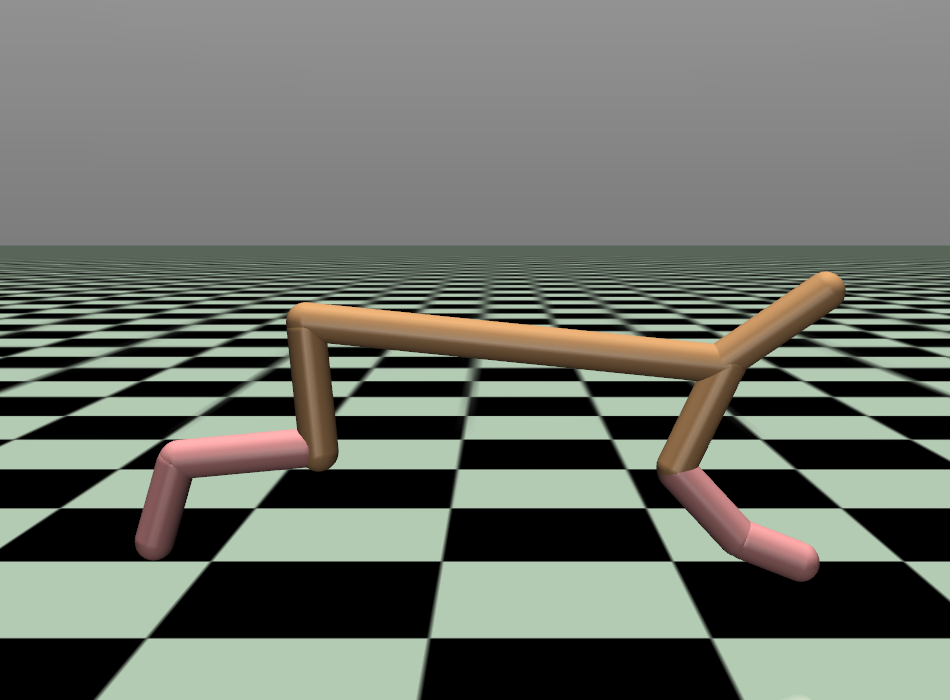}
		}
		\subfigure[time step 3]{
			\label{time step 3}
			\includegraphics[width=0.10\textheight]{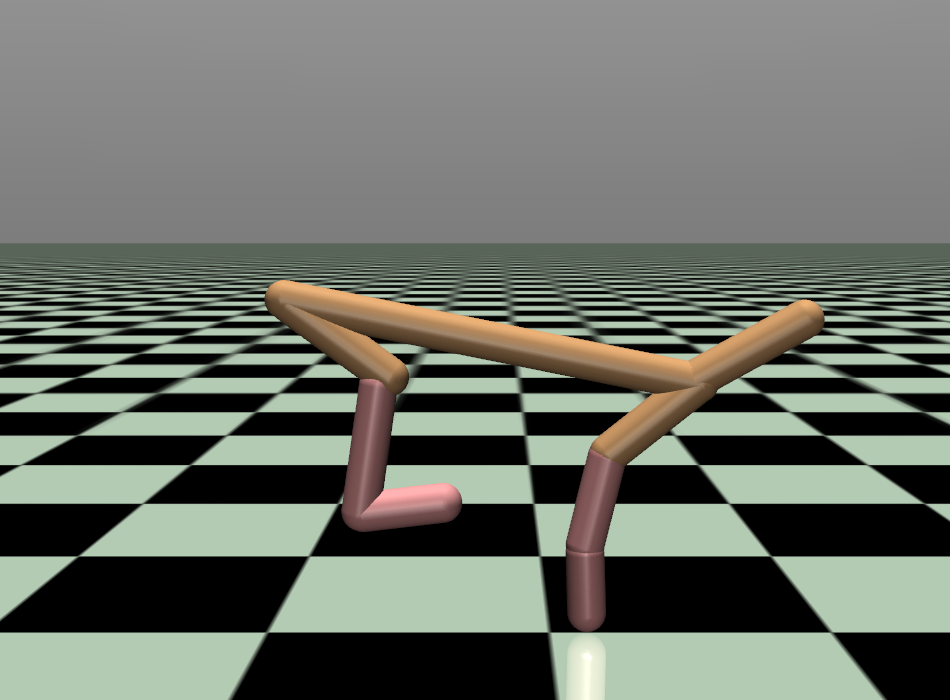}
		}
		\subfigure[time step 5]{
			\label{time step 5}
			\includegraphics[width=0.10\textheight]{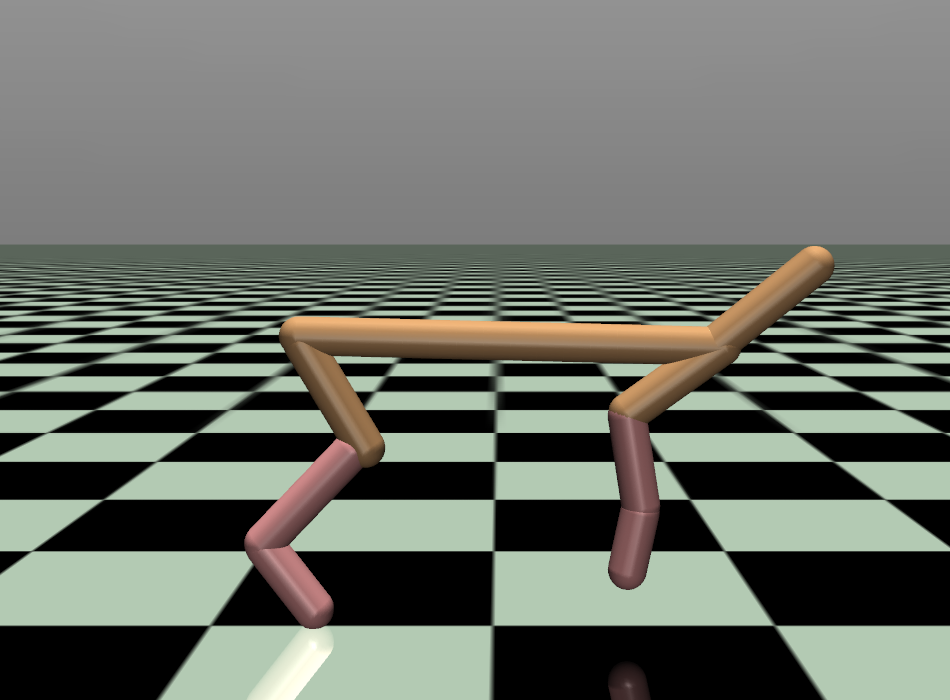}
		}
		\subfigure[time step 6]{
			\label{time step 6}
			\includegraphics[width=0.10\textheight]{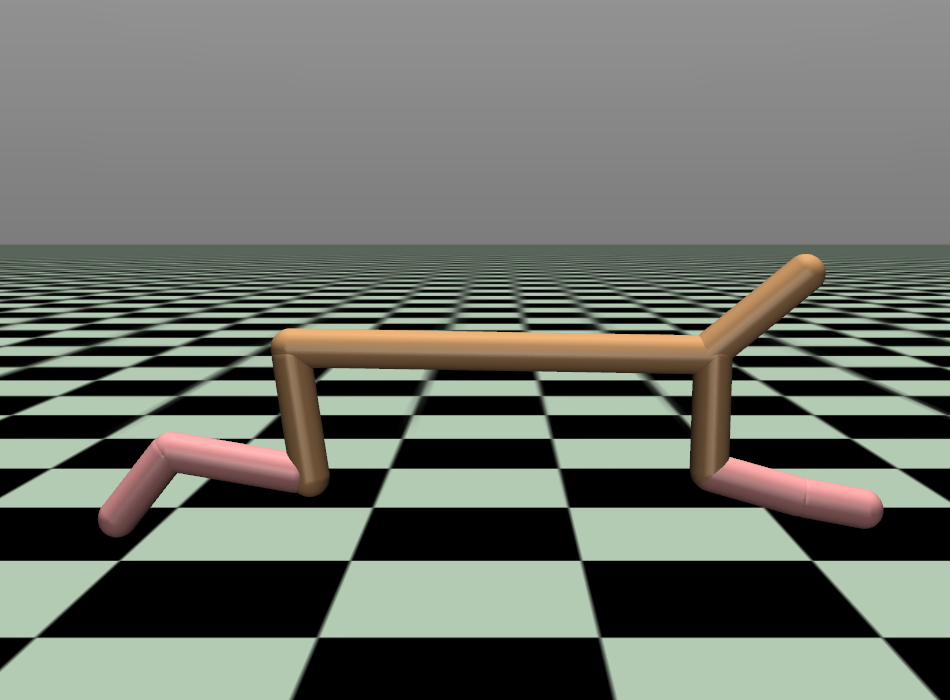}
		}
		\subfigure[time step 7]{
			\label{time step 7}
			\includegraphics[width=0.10\textheight]{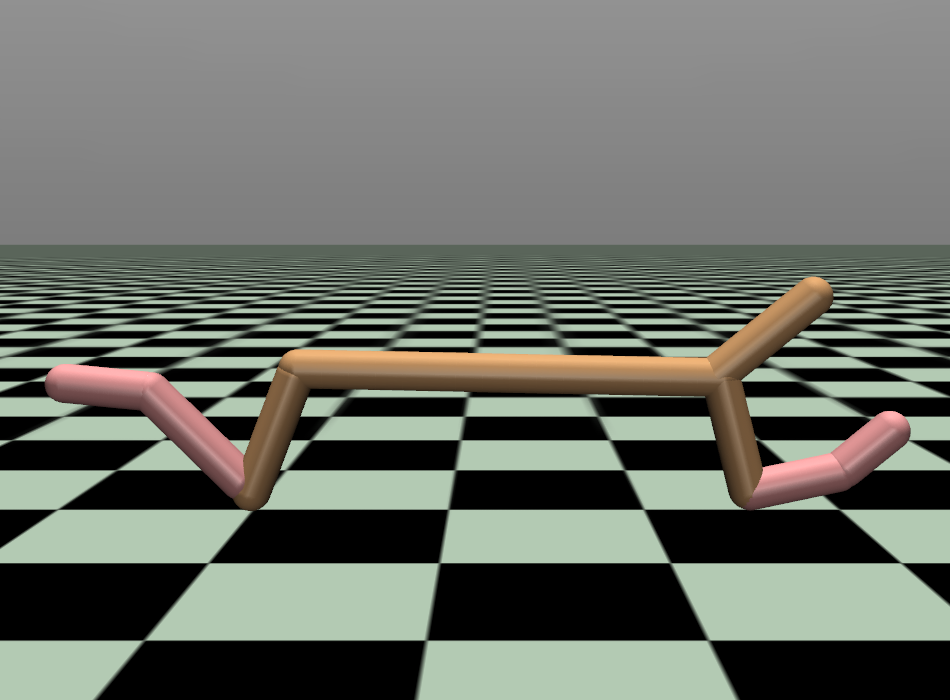}
		}
		
		\quad
		\subfigure{
			\includegraphics[width=0.6\textheight]{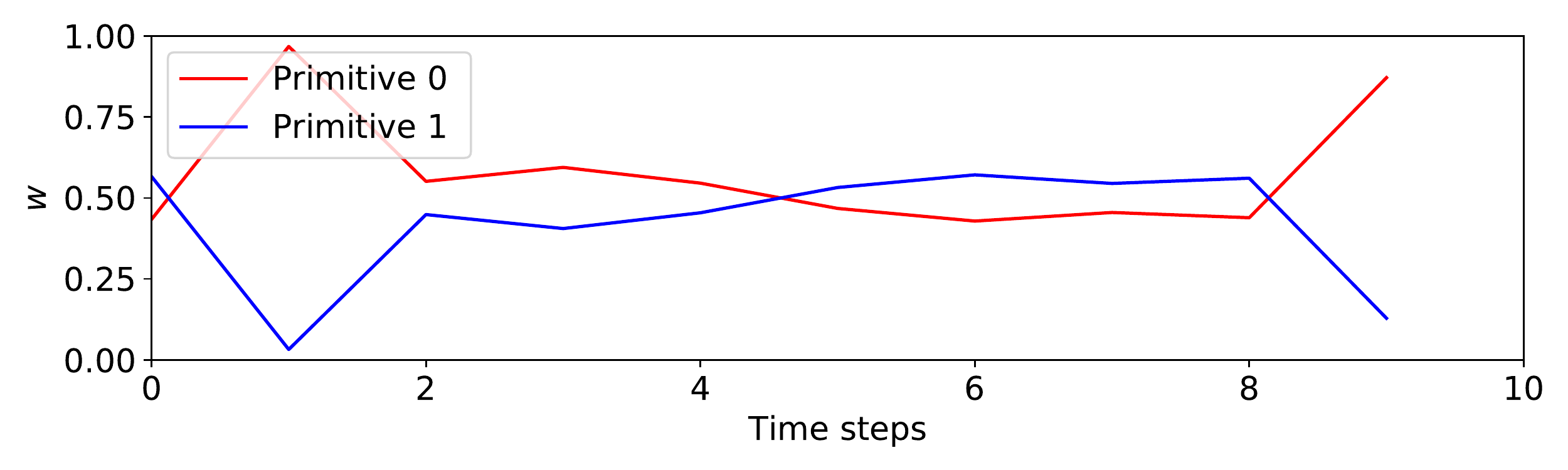}
		}
		\caption{Visualisation of the actions at the selected 5 time steps in one period. 
		The y-axis shows the probabilities of different primitives.
		This result shows that the primitives develop distinct specialisations, with the primitive 0 becomes the most active when the front leg touches the ground, while the primitive 1 becomes the most active when the leg leaves the ground.}
		\label{fig:one_period}
	\end{figure*}
	
	\section{{t}-SNE Visualisation}
	To demonstrate the distinguishable primitives, we plot the t-SNE visualisation for other 5 MuJoCo environments in Fig.~\ref{fig:all_tsne}.
    \label{app:tsne}
	\begin{figure*}[ht!]
		\centering
		\begin{tikzpicture}
		\node[anchor=south west, inner sep = 0] at(0,0){
			\includegraphics[width=0.195\textheight]{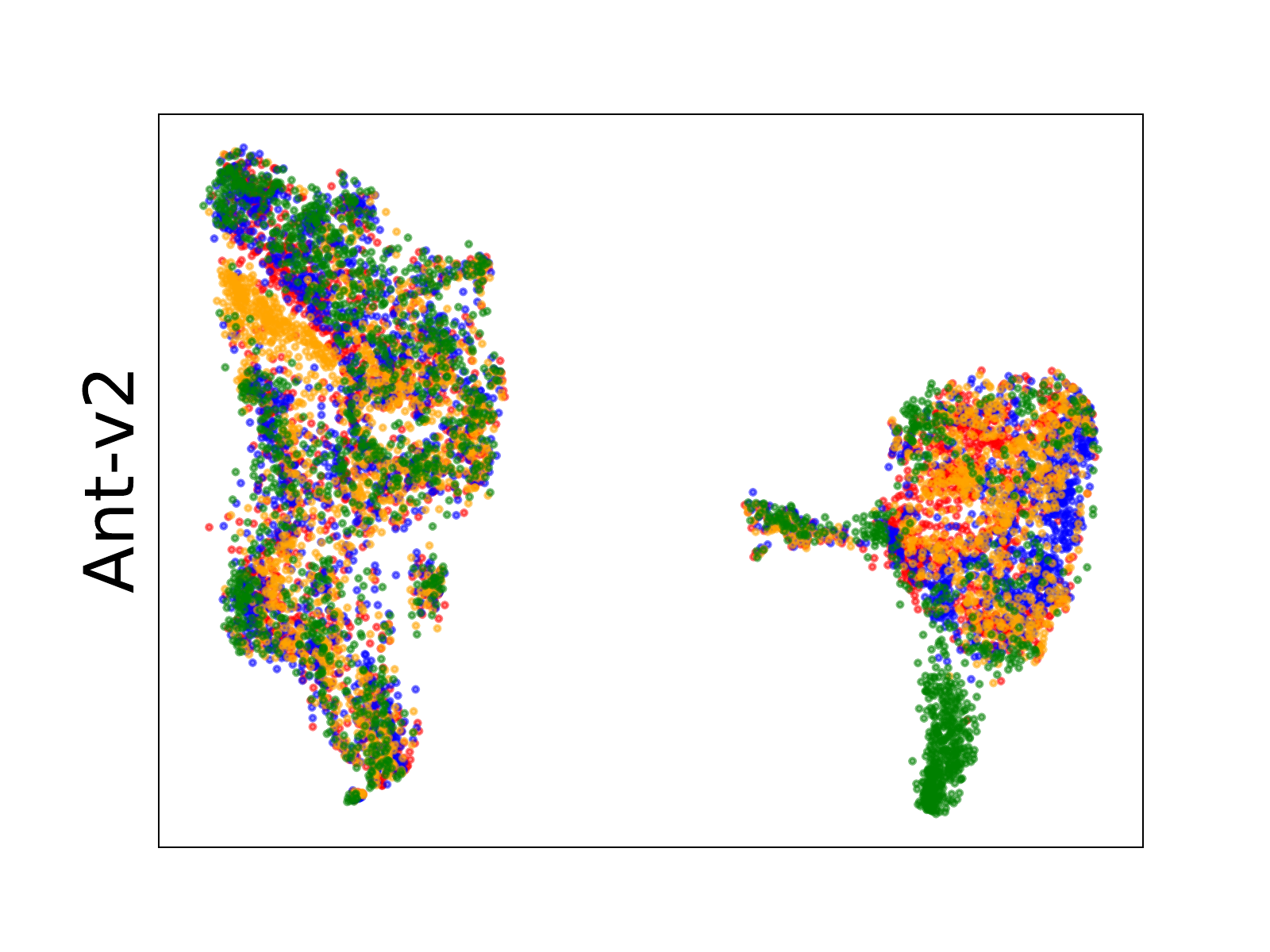}
			\includegraphics[width=0.195\textheight]{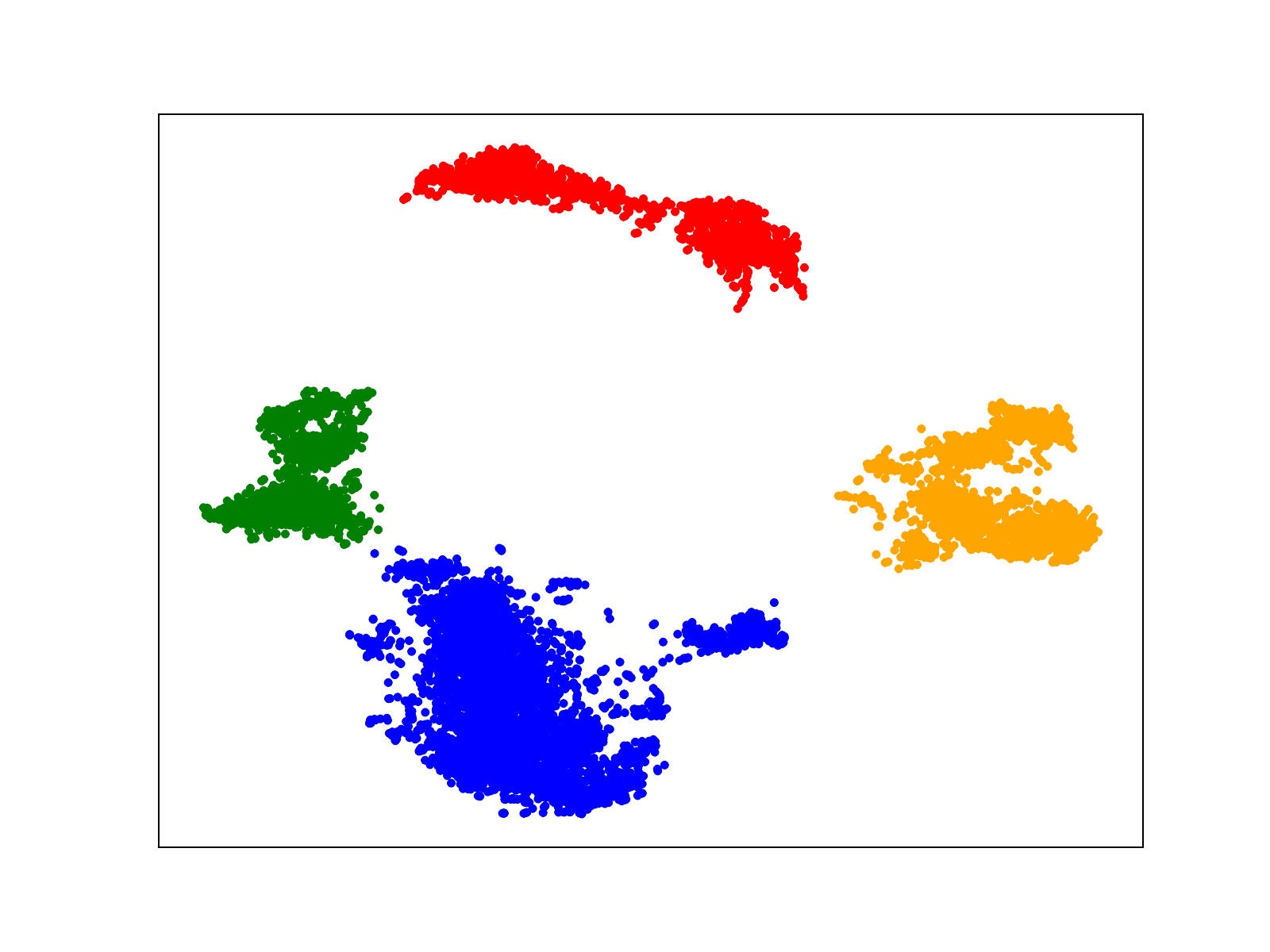}
			\includegraphics[width=0.195\textheight]{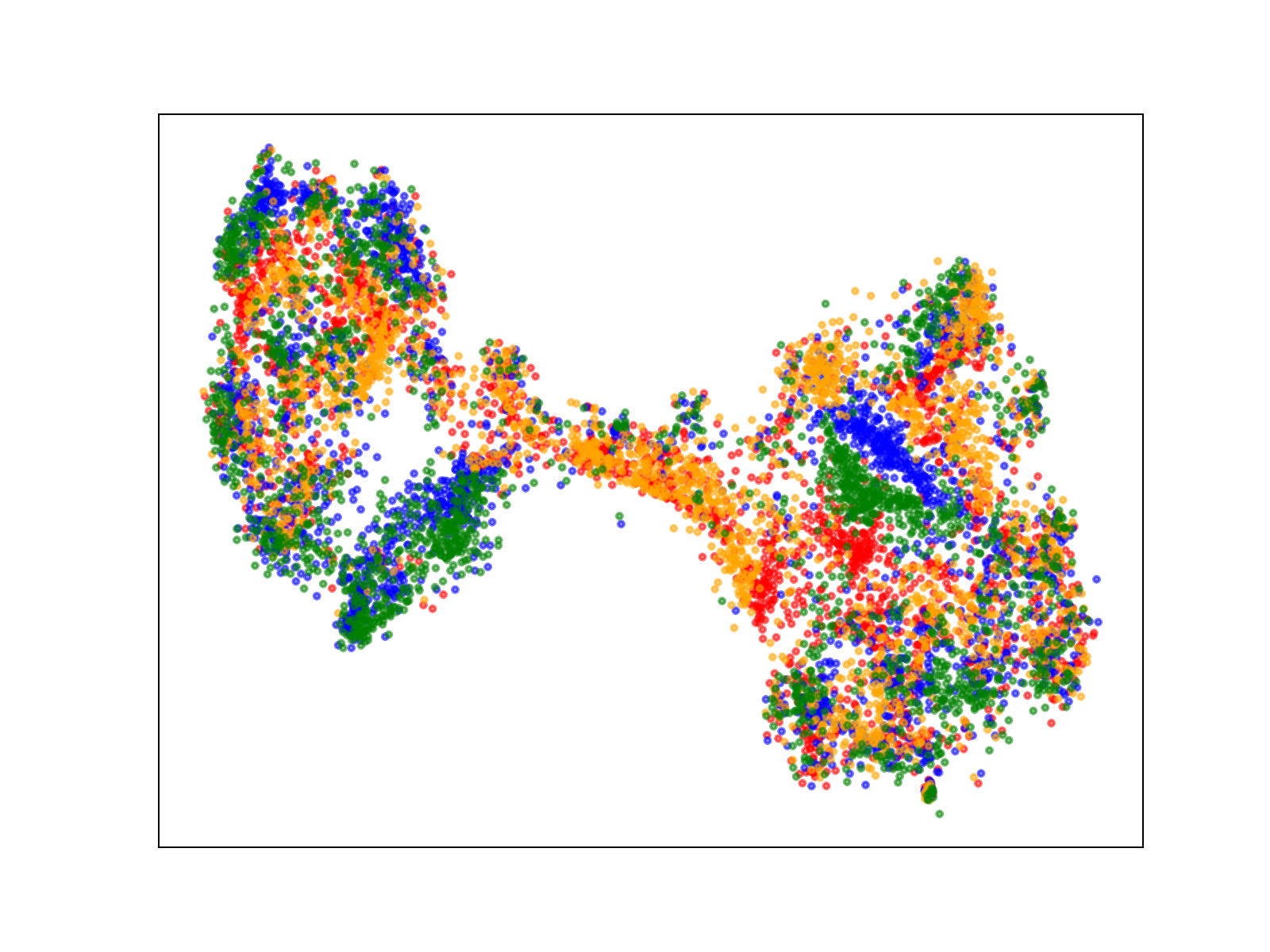}
		};
		\coordinate (startleft) at (7.5,2);
		\coordinate (gatingend) at (4.2,2);
		\draw[very thick, ->] (startleft) to[out=160,in=30] node {} (gatingend);
		\coordinate(startright) at (8.8, 2);
		\coordinate(oursend) at (9.5,2);
		\draw[very thick, ->] (startright) to[out=30,in=160] node {} (oursend);
		\draw[dashed, gray, thick, rounded corners] (7.5, 1.2) rectangle (8.6, 2.1);
		\end{tikzpicture}
		\begin{tikzpicture}
		\node[anchor=south west, inner sep = 0] at(0,0){
			\includegraphics[width=0.195\textheight]{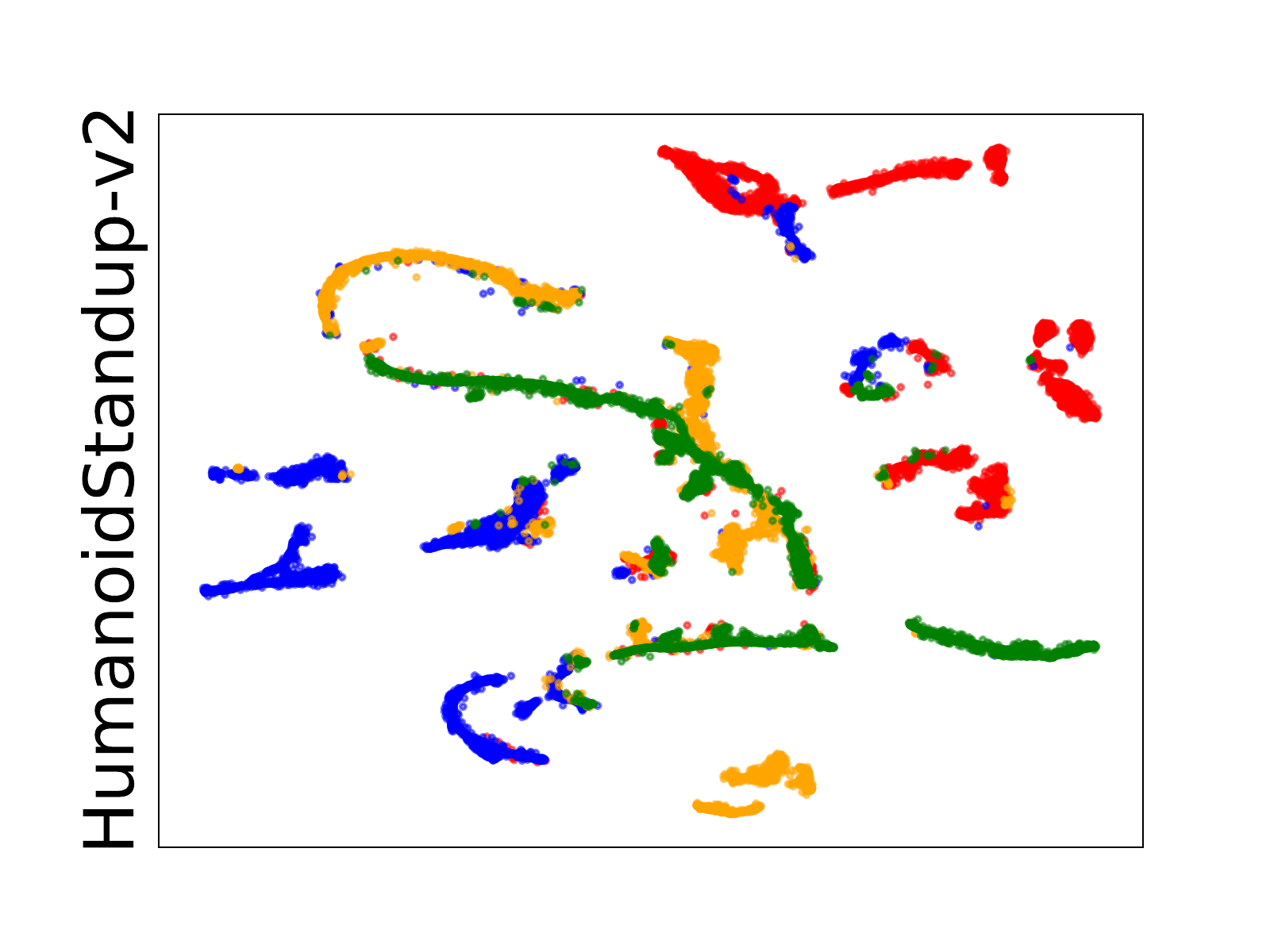}
			\includegraphics[width=0.195\textheight]{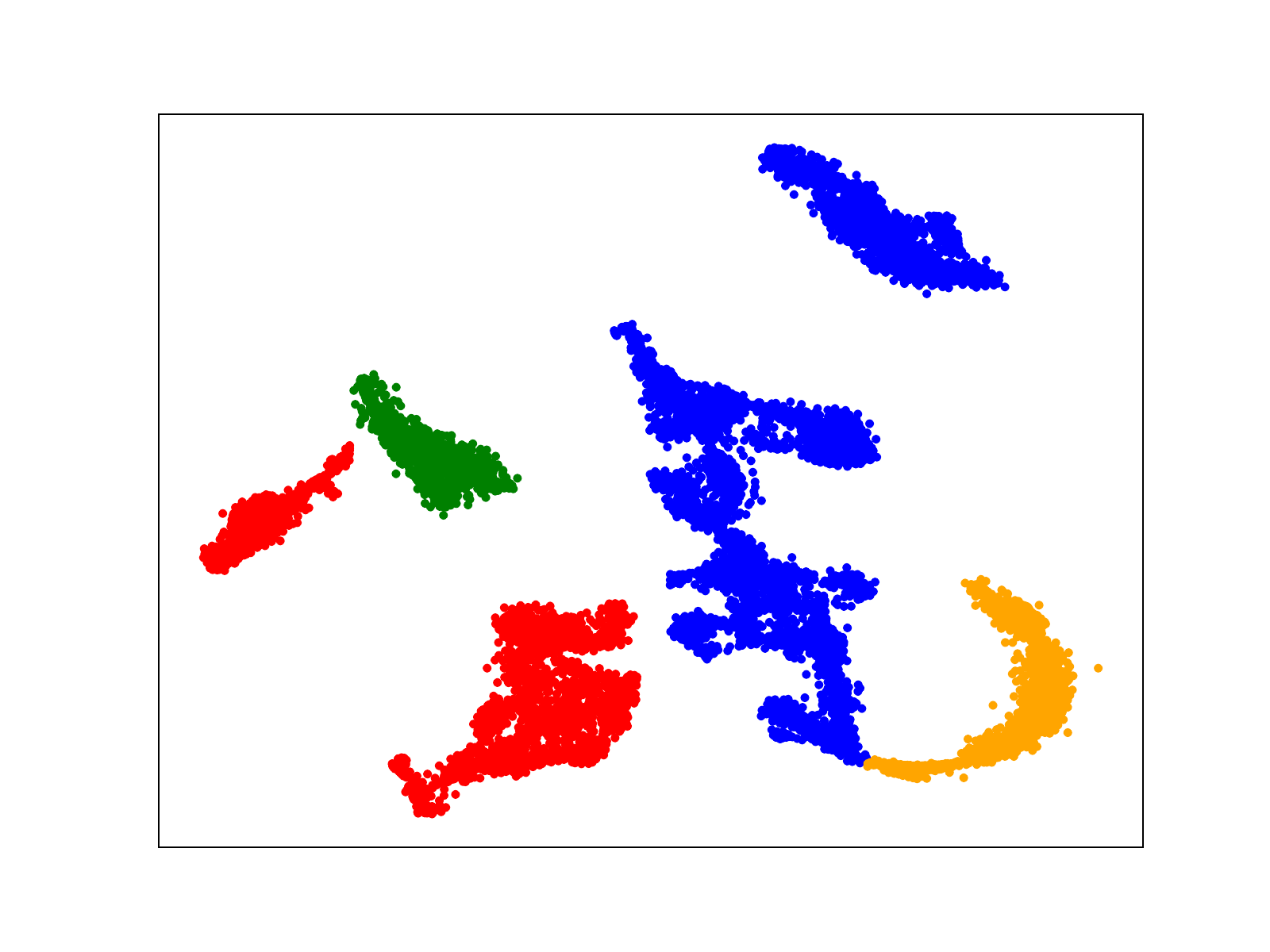}
			\includegraphics[width=0.195\textheight]{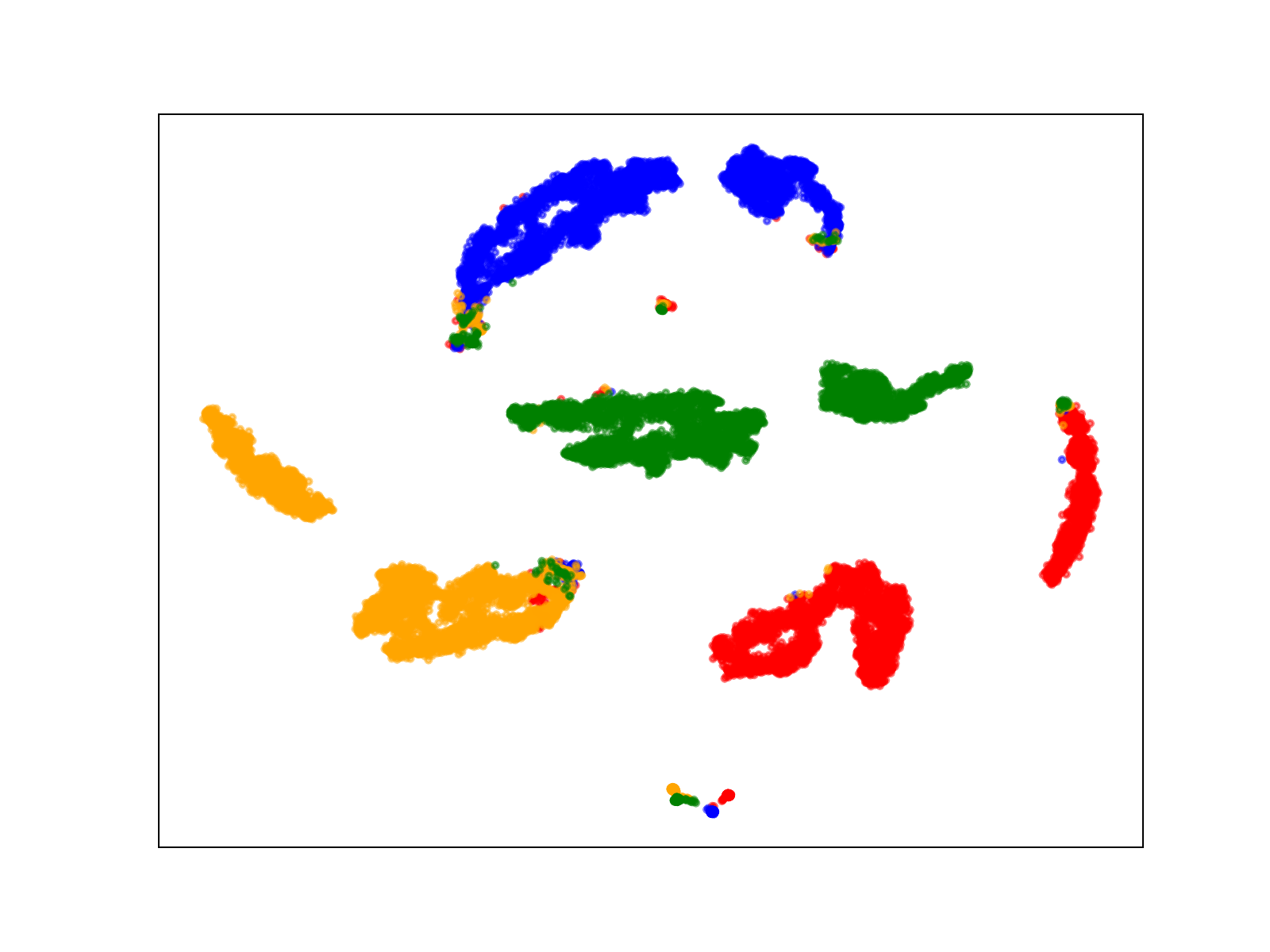}
		};
		\coordinate (startleft) at (6.8,2.3);
		\coordinate (gatingend) at (4.2,2);
		\draw[very thick, ->] (startleft) to[out=150,in=30] node {} (gatingend);
		\coordinate(startright) at (8, 2);
		\coordinate(oursend) at (9.5,2);
		\draw[very thick, ->] (startright) to[out=30,in=150] node {} (oursend);
		\draw[dashed, gray, thick, rounded corners] (6.7, 0.7) rectangle (7.7, 2.1);
		\end{tikzpicture}
		\begin{tikzpicture}
		\node[anchor=south west, inner sep = 0] at(0,0){
			\includegraphics[width=0.195\textheight]{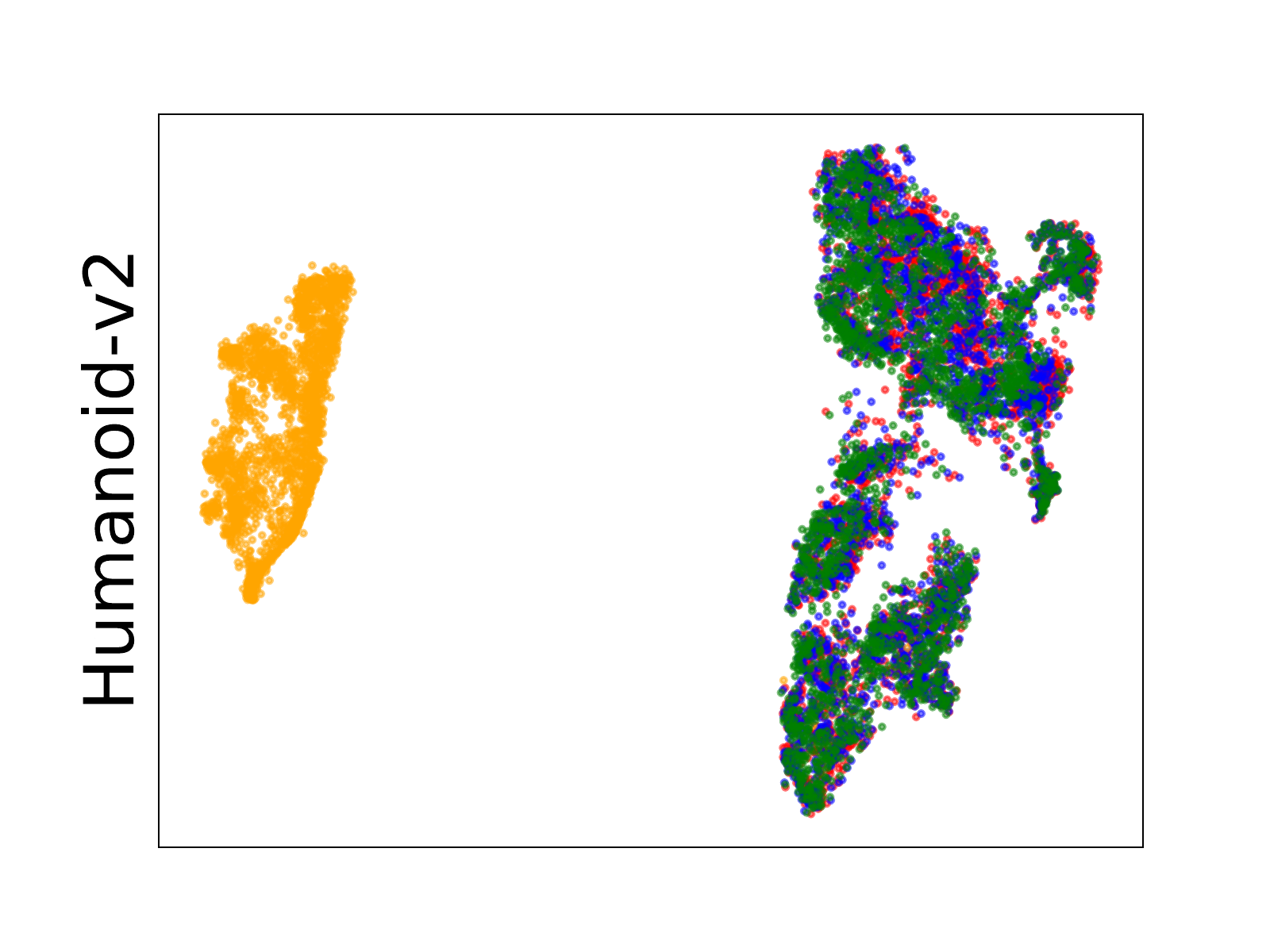}
			\includegraphics[width=0.195\textheight]{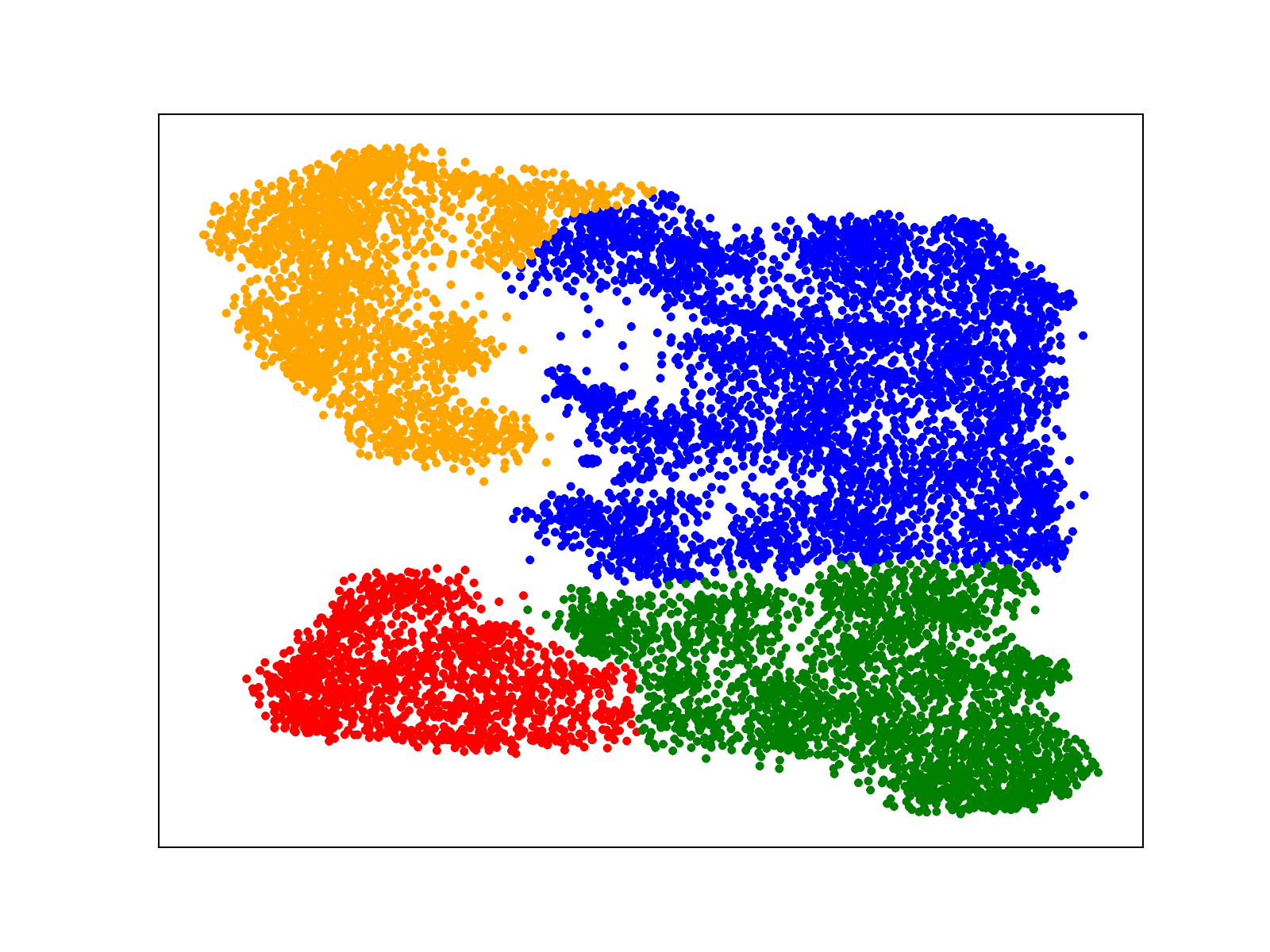}
			\includegraphics[width=0.195\textheight]{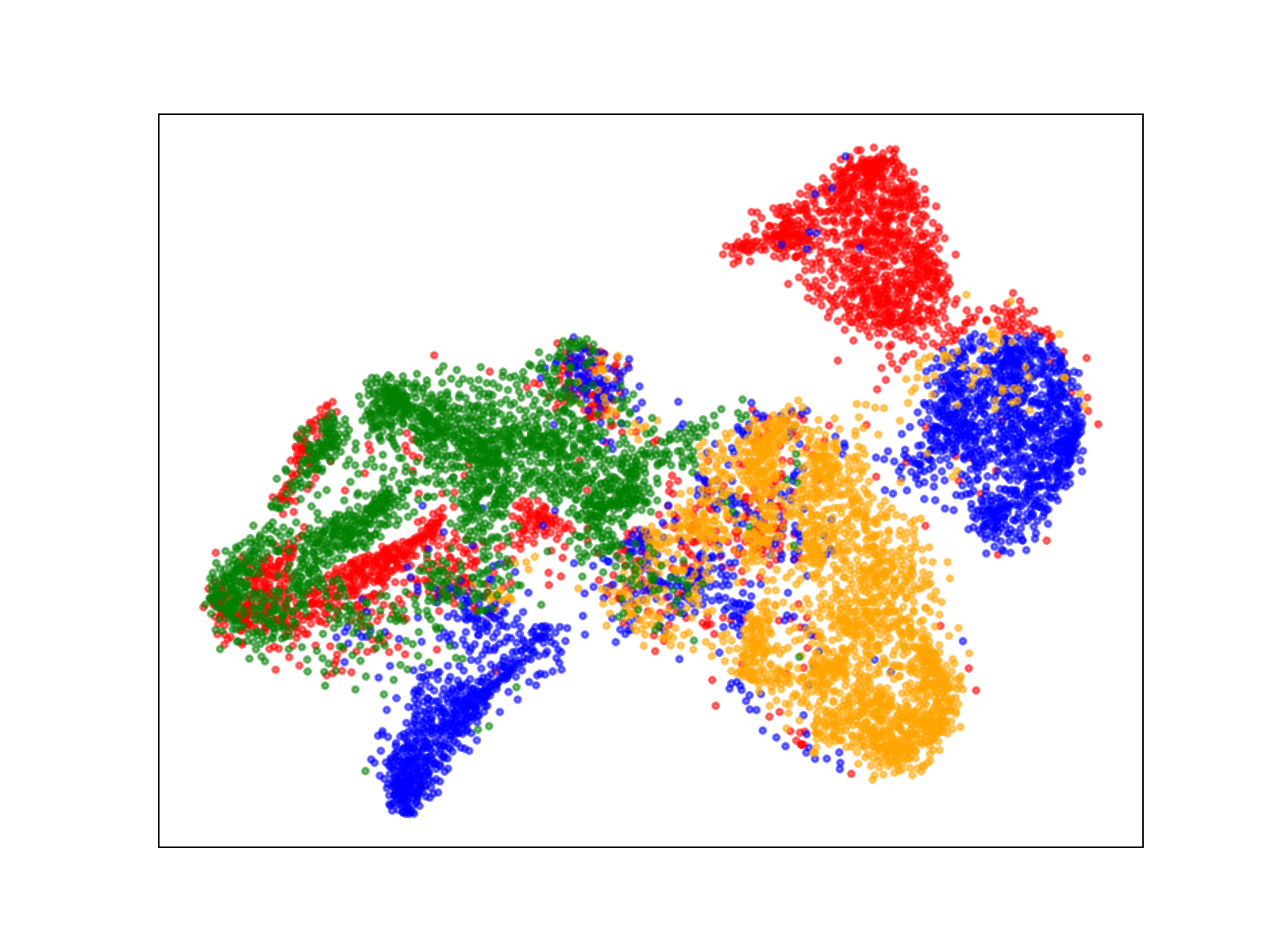}
		};
		\coordinate (startleft) at (5.4,2.5);
		\coordinate (gatingend) at (4.2,2);
		\draw[very thick, ->] (startleft) to[out=-160,in=30] node {} (gatingend);
		\coordinate(startright) at (6, 1.5);
		\coordinate(oursend) at (9.5,2);
		\draw[very thick, ->] (startright) to[out=-45,in=-135] node {} (oursend);
		\draw[dashed, gray, thick, rounded corners] (5.3, 1.7) rectangle (6.5, 2.8);
		\end{tikzpicture}
		\begin{tikzpicture}
		\node[anchor=south west, inner sep = 0] at(0,0){
			\includegraphics[width=0.195\textheight]{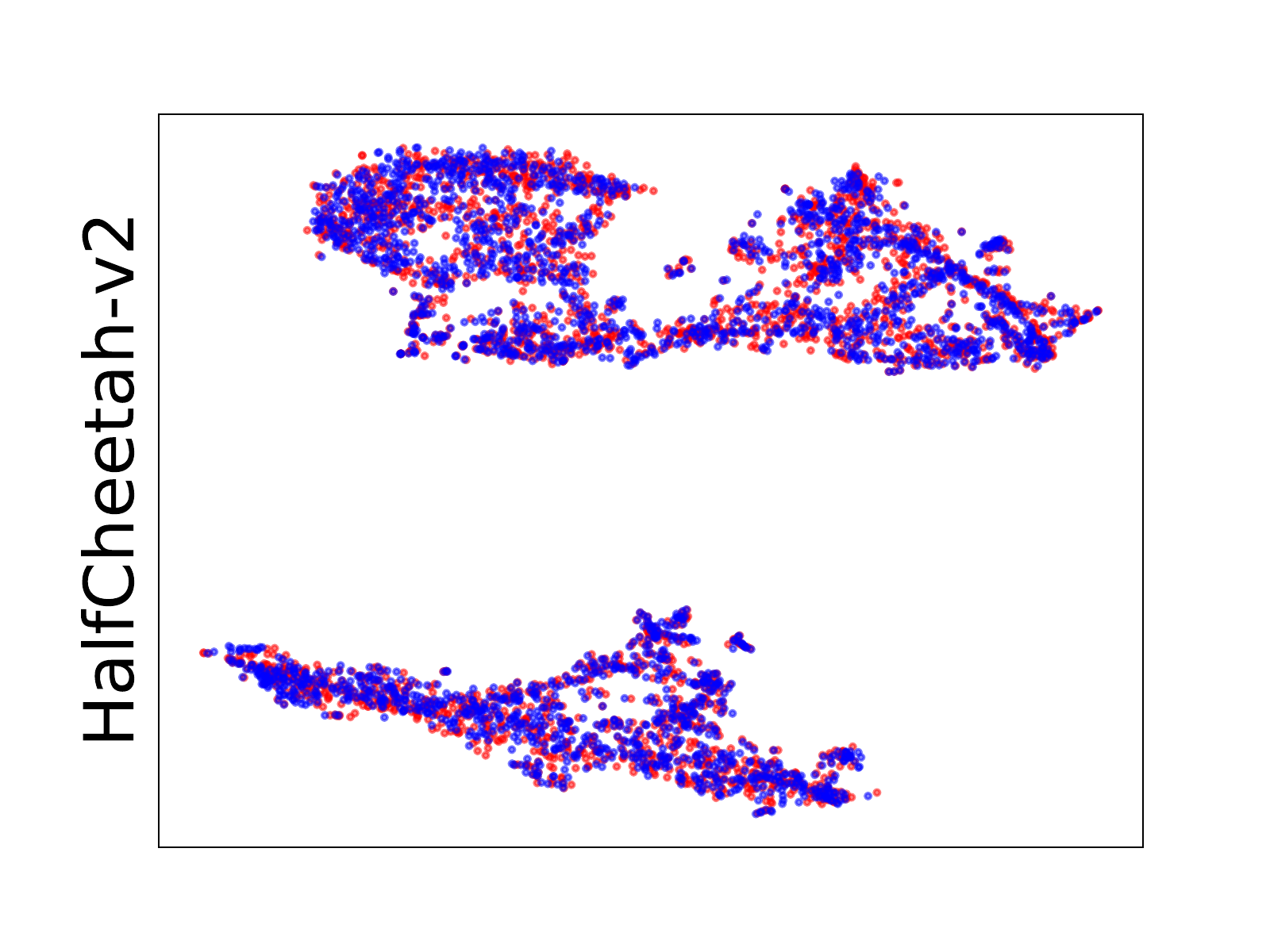}
			\includegraphics[width=0.195\textheight]{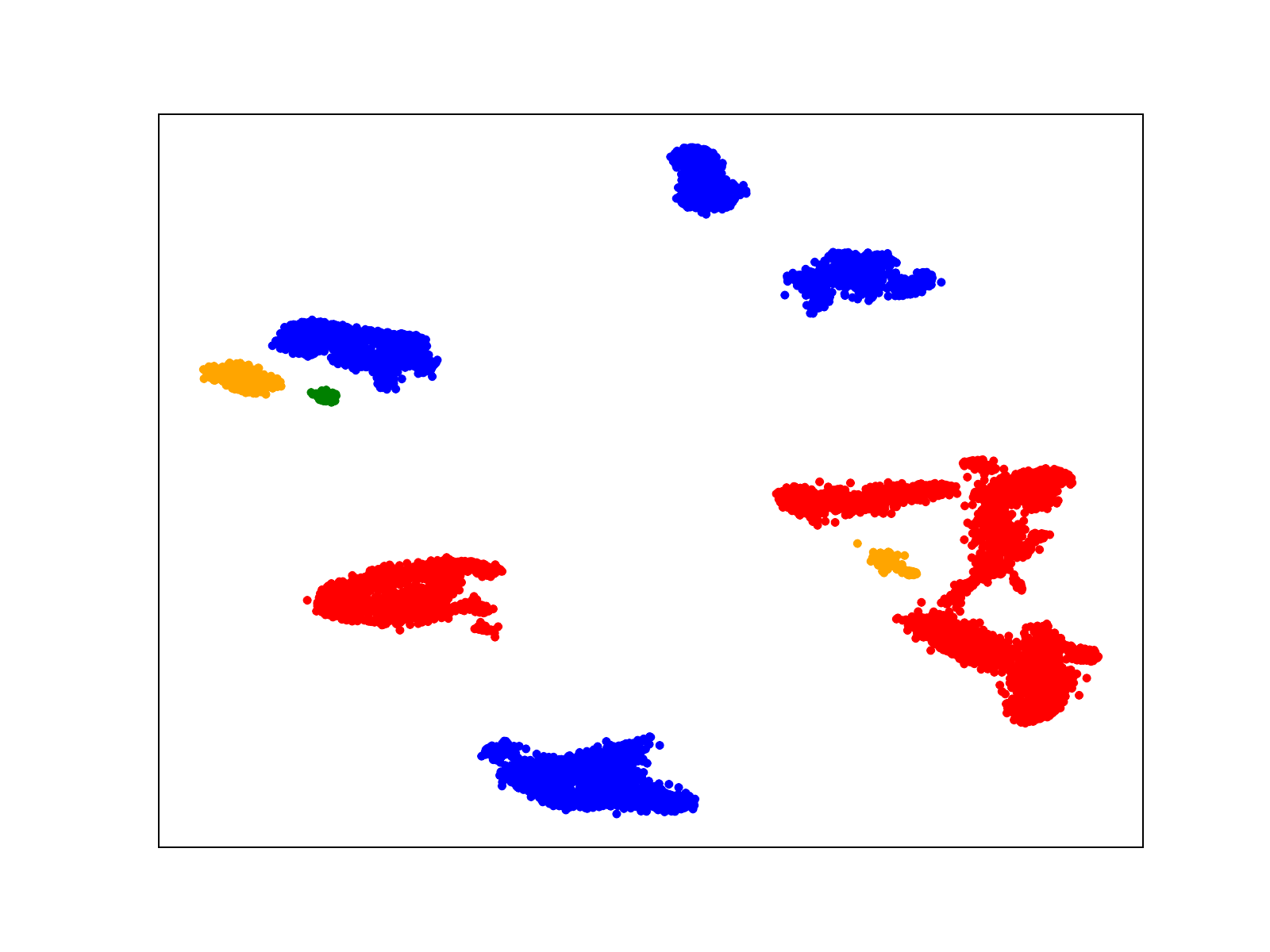}
			\includegraphics[width=0.195\textheight]{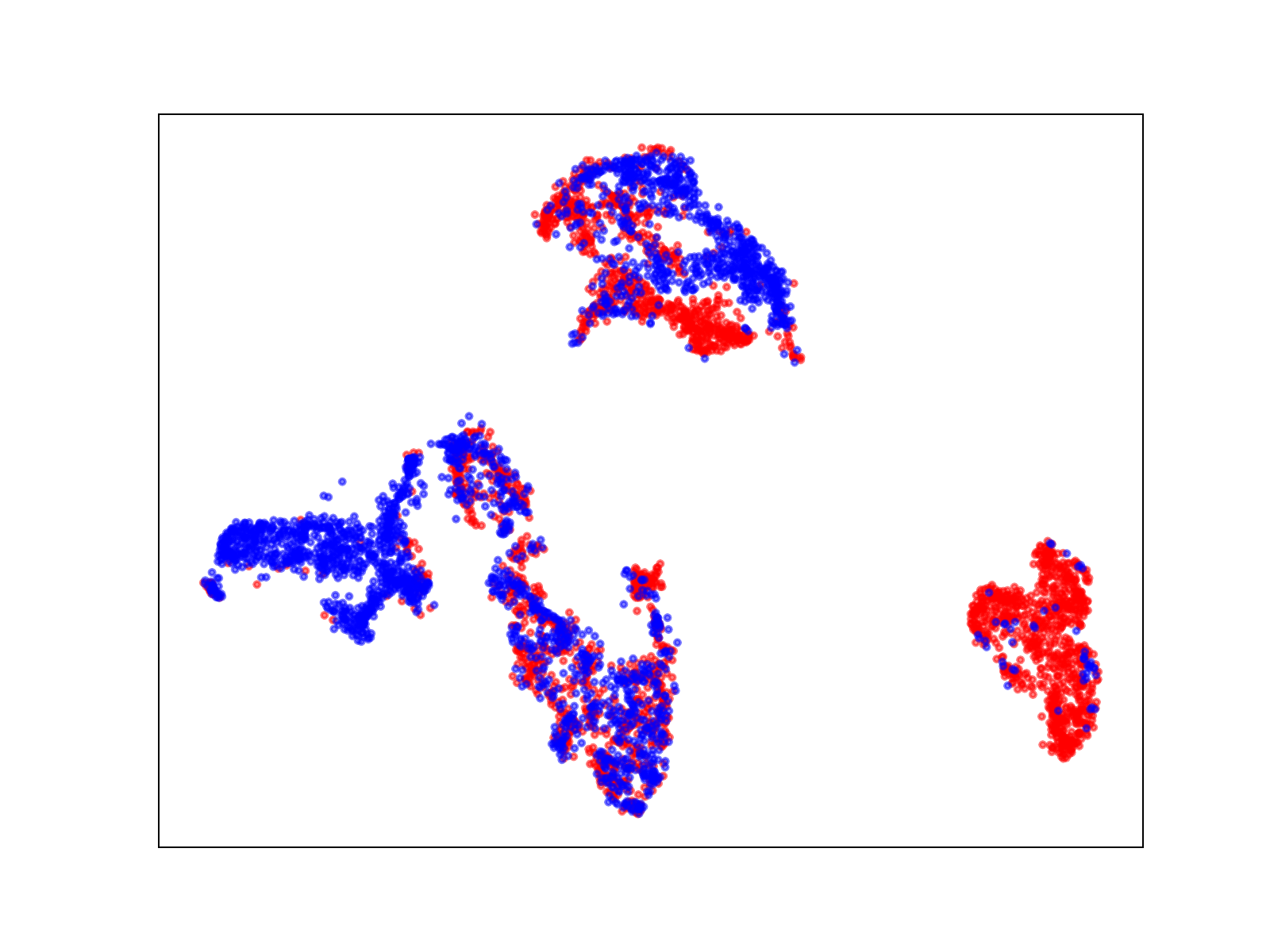}
		};
		\coordinate (startleft) at (7.3,2);
		\coordinate (gatingend) at (4.2,2);
		\draw[very thick, ->] (startleft) to[out=150,in=30] node {} (gatingend);
		\coordinate(startright) at (8, 2);
		\coordinate(oursend) at (9.5,2);
		\draw[very thick, ->] (startright) to[out=30,in=150] node {} (oursend);
		\draw[dashed, gray, thick, rounded corners] (7.3, 0.7) rectangle (8.5, 1.8);
		\end{tikzpicture}
		\begin{tikzpicture}
		\node[anchor=south west, inner sep = 0] at(0,0){
			\subfigure[Gating Operation]{
				\includegraphics[width=0.195\textheight]{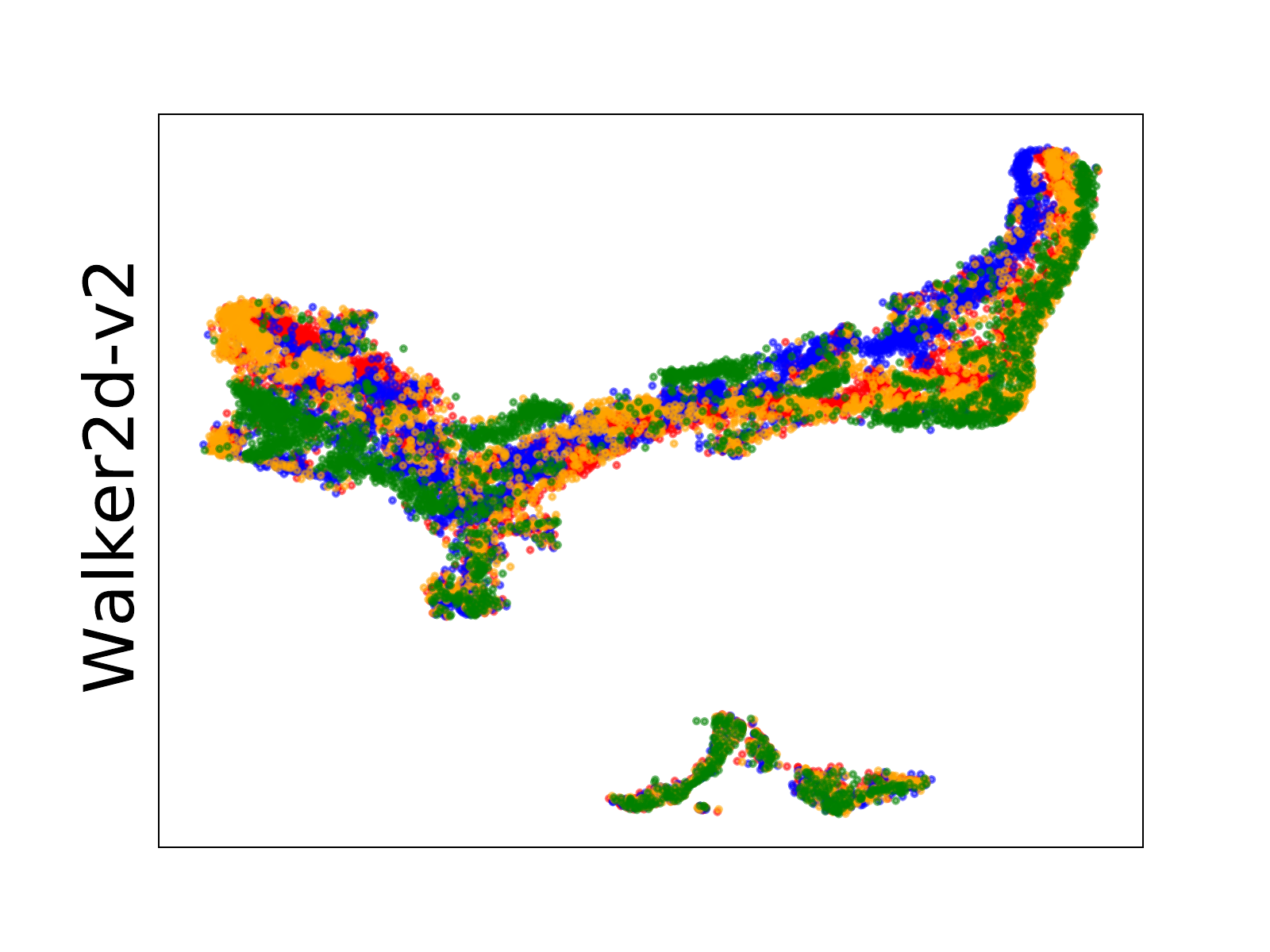}}
			\subfigure[State Observations]{
				\includegraphics[width=0.195\textheight]{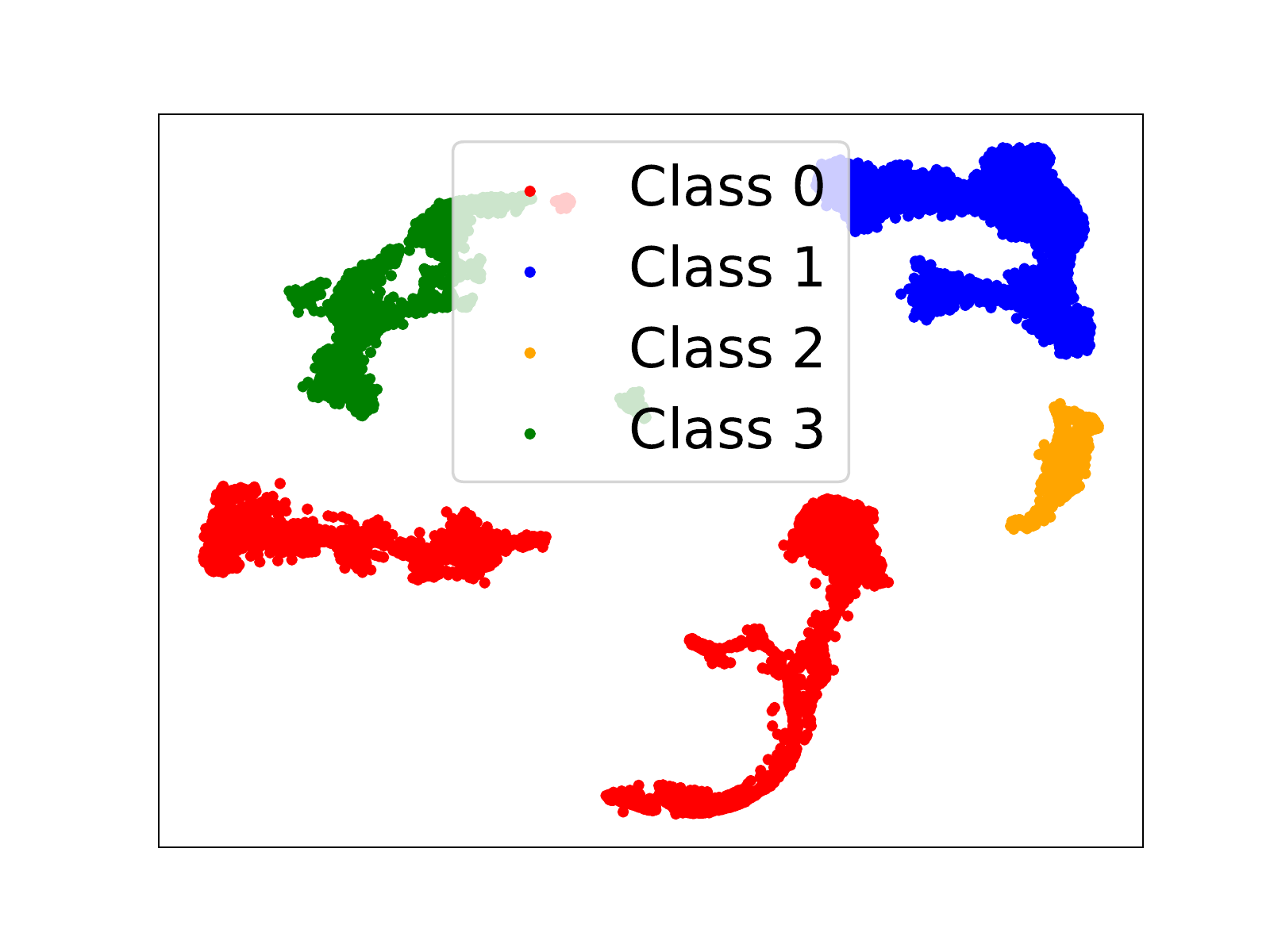}}
			\subfigure[Ours]{
				\includegraphics[width=0.195\textheight]{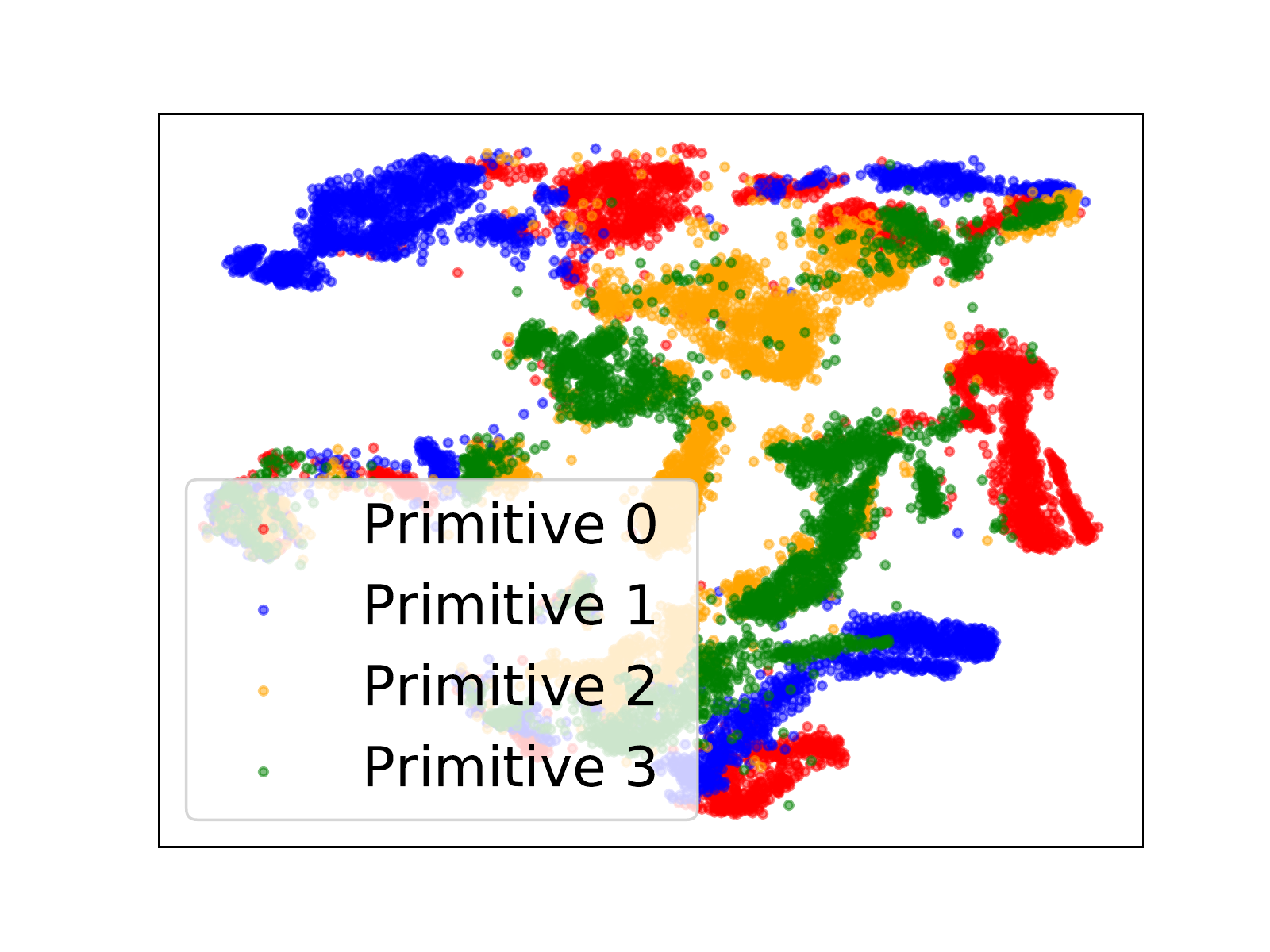}}
		};
		\coordinate (startleft) at (7.3,2.5);
		\coordinate (gatingend) at (4.2,2);
		\draw[very thick, ->] (startleft) to[out=-150,in=30] node {} (gatingend);
		\coordinate(startright) at (8.6, 2.7);
		\coordinate(oursend) at (9.5,2);
		\draw[very thick, ->] (startright) to[out=-30,in=150] node {} (oursend);
		\draw[dashed, gray, thick, rounded corners] (7.4, 2.7) rectangle (8.6, 3.5);
		\end{tikzpicture}
		\caption{We plot the t-SNE visualisation for other 5 MuJoCo environments: \textit{Ant-v2}, \textit{HumanoidStandup-v2}, \textit{Humanoid-v2}, \textit{HalfCheetah-v2} and \textit{Walker2D-v2}. Parameters and other details are the same as the setting mentioned in Sec 4.}
	\label{fig:all_tsne}
	\end{figure*}


\end{document}